\documentclass[11pt]{article}
\usepackage[utf8]{inputenc}

\usepackage{mystyle}

\begin{document}

\title{\Huge A Theoretical Analysis of  Deep Q-Learning}

\author{Jianqing Fan\thanks{Department of Operations Research and Financial Engineering, Princeton University.  Research supported by the NSF grant DMS-1662139 and DMS-1712591, the ONR grant N00014-19-1-2120, and the NIH grant 2R01-GM072611-14.}\quad\qquad Zhaoran Wang\thanks{Department of Industrial Engineering and Management Sciences, Northwestern  University}\quad\quad Yuchen Xie$^\dagger$ \quad\quad Zhuoran Yang$^*$ }


\maketitle


\begin{abstract}
Despite the great empirical success of deep reinforcement learning, its theoretical foundation is less well understood. In this work, we make the first attempt to theoretically understand the deep Q-network (DQN) algorithm \citep{mnih2015human} from both algorithmic and statistical perspectives. In specific, we focus on a  slight simplification of DQN that fully captures its key features. Under mild assumptions, we establish the algorithmic and statistical rates of convergence for the action-value functions of the iterative policy sequence  obtained by DQN. In particular, the statistical error characterizes the bias and variance that arise from approximating the action-value function using deep neural network, while the algorithmic error converges to zero at a geometric rate. As a byproduct, our analysis provides justifications for the techniques of experience replay and target network, which are crucial to the empirical success of DQN. Furthermore, as a simple extension of  DQN, we   propose the Minimax-DQN algorithm for zero-sum Markov game with two players.  Borrowing the analysis of DQN, we also quantify the difference between  the   policies   obtained by Minimax-DQN  and  the Nash equilibrium of the Markov game     in terms of both
 the algorithmic and statistical rates of convergence.
\end{abstract}

\section{Introduction}

Reinforcement learning (RL) attacks the multi-stage decision-making problems by interacting with the environment and learning from the experiences. With the breakthrough in deep learning, deep reinforcement learning (DRL) demonstrates tremendous success in solving highly challenging problems, such as the game of Go \citep{silver2016mastering,silver2017mastering}, computer games \citep{vinyals2019grandmaster},  robotics \citep{kober2012reinforcement},  dialogue systems \citep{chen2017survey}. In DRL, the value or policy functions are often represented as deep neural networks and the related deep learning techniques can be readily applied. For example, deep Q-network (DQN) \citep{mnih2015human},  asynchronous advantage actor-critic (A3C) \citep{mnih2016asynchronous},  trust region policy optimization (TRPO) \citep{schulman2015trust},  proximal policy optimization (PPO) \citep{schulman2017proximal}  
build upon classical RL methods  \citep{watkins1992q, sutton2000policy,konda2000actor}
and  have become benchmark  algorithms for artificial intelligence.

Despite its great empirical success, there exists a substantial gap between the theory and practice of DRL. In particular, most existing theoretical work on reinforcement learning focuses on the tabular case where the state and action spaces are finite, or the case where the value function is linear. Under these restrictive settings, the algorithmic and statistical perspectives of reinforcement learning are well-understood via the tools developed for convex optimization and linear regression. However, in presence of nonlinear function approximators such as deep neural network, the theoretical analysis of reinforcement learning becomes intractable as it involves solving a highly nonconvex statistical optimization problem.

To bridge such a gap in DRL, we make the first attempt to theoretically understand DQN, which can be cast as an extension of the classical Q-learning algorithm \citep{watkins1992q} that uses deep neural network to approximate the  action-value function. Although the algorithmic and statistical properties of the classical Q-learning algorithm are well-studied,   theoretical analysis of DQN is highly challenging due to its differences in the following two aspects.

  First, in online gradient-based temporal-difference reinforcement learning algorithms, approximating the action-value function often leads to instability. \cite{baird1995residual} proves that this is the case even with linear function approximation. The key technique to achieve stability in DQN is experience replay \citep{lin1992self, mnih2015human}. In specific, a replay memory is used to store the trajectory of the Markov decision process (MDP). At each iteration of DQN, a mini-batch of states, actions, rewards, and next states are sampled from the replay memory as observations to train the Q-network, which approximates the action-value function. The intuition behind experience replay is to achieve stability by breaking the temporal dependency among the observations used in  training  the deep neural network.


  Second, in addition to the aforementioned Q-network, DQN uses another neural network named the target network to obtain an unbiased estimator of the mean-squared Bellman error used in training the Q-network. The target network is synchronized with the Q-network after each period of iterations, which leads to a coupling between the two networks. Moreover, even if we fix the target network and focus on updating the Q-network, the subproblem of training a neural network still remains less well-understood in theory.

In this paper, we focus on a slight simplification of DQN, which is amenable to theoretical analysis while fully capturing the above two aspects. In specific, we simplify the technique of experience replay with an independence assumption, and focus on deep neural networks with rectified linear units (ReLU) \citep{nair2010rectified} and large batch size. Under this setting, DQN is reduced to the neural fitted Q-iteration (FQI) algorithm \citep{riedmiller2005neural} and the technique of target network can be cast as the value iteration. More importantly, by adapting the approximation results for ReLU networks to the analysis of Bellman operator, we establish the algorithmic and statistical rates of convergence for the iterative policy sequence obtained by DQN. As shown in the main results in \S\ref{sec:algo}, the statistical error characterizes the bias and variance that arise from approximating the action-value function using neural network, while the algorithmic error geometrically decays to zero as the number of iteration goes to infinity.

Furthermore, we extend   DQN to two-player zero-sum Markov games \citep{shapley1953stochastic}. The proposed algorithm, named Minimax-DQN, can be viewed as a combination of the  Minimax-Q learning algorithm for tabular zero-sum Markov games \citep{littman1994markov} and deep neural networks for function approximation. Compared with DQN, the main difference lies in the approaches to compute the target values. In DQN, the target is computed via maximization over the action space. In contrast, the target obtained  computed by solving the Nash equilibrium of  a zero-sum matrix game in Minimax-DQN, which can be efficiently attained via linear programming. Despite such a difference,   both these two methods  can be viewed as approximately applying   the Bellman  operator to the Q-network. Thus, borrowing the analysis of DQN, we also establish theoretical results for Minimax-DQN. Specifically, we quantify the suboptimality of  policy returned by the algorithm by the difference between the action-value functions  associated with   this policy and   with  the Nash equilibrium policy of the Markov game. For this notion of suboptimality, we establish the both algorithmic and statistical rates of convergence, which implies that the action-value function converges to the optimal counterpart up to an unimprovable statistical error in geometric rate.

Our contribution is three-fold. First, we establish  the algorithmic and statistical errors of the neural FQI algorithm, which can be viewed as a slight simplification of DQN. Under mild assumptions, our results show that the proposed algorithm obtains a sequence of Q-networks that geometrically converges to the optimal action-value function up to an intrinsic statistical error induced by the approximation bias of ReLU network and finite sample size. Second, as a byproduct, our analysis justifies the techniques of experience replay and target network used in DQN, where the latter can be viewed as a single step of  the value iteration. Third, we propose the Minimax-DQN algorithm that extends DQN to two-player zero-sum Markov games. Borrowing the analysis for DQN, we establish the algorithmic and statistical convergence rates of the action-value functions associated with the sequence of policies  returned by the Minimax-DQN algorithm.

\subsection{Related Works}

There is a huge body of literature on deep reinforcement learning, where these  algorithms are based on  Q-learning or policy gradient \citep{sutton2000policy}. We refer the reader to \citet{arulkumaran2017brief} for a survey of the recent developments of DRL.
In addition, the DQN algorithm is first proposed in \citet{mnih2015human}, which applies DQN to Artari 2600 games \citep{bellemare2013arcade}. The extensions of DQN include double DQN \citep{van2016deep}, dueling DQN \citep{wang2016dueling}, deep recurrent Q-network \citep{hausknecht2015deep},  asynchronous  DQN \citep{mnih2016asynchronous}, and variants designed for  distributional reinforcement learning  \citep{bellemare2017distributional, dabney2018distributional, dabney2018implicit}.  All of these algorithms  are  corroborated only by numerical experiments, without theoretical guarantees. Moreover, these algorithms not only  inherit the tricks of experience replay  and the target network proposed in the original DQN, but develop even more tricks to   enhance the performance. Furthermore, recent works such as \citet{schaul2016prioritized, andrychowicz2017hindsight, liu2017effects, zhang2017deeper, novati2019remember} study the effect of experience replay  and propose various modifications.

In addition, our work is  closely related to the literature on batch reinforcement learning \citep{lange2012batch}, where the goal is to  estimate  the   value function   given transition data. These problems are usually formulated into least-squares regression, for which  various algorithms are proposed with  finite-sample analysis.  However, most  existing works focus on the settings where the   value function are approximated by linear functions.  See \citet{bradtke1996linear,boyan2002technical,lagoudakis2003least,lazaric2010analysis,farahmand2010error, lazaric2012finite,tagorti2015rate}  and the references therein for results of the least-squares policy iteration (LSPI) and Bellman residue minimization (BRM) algorithms. Beyond linear function approximation, a recent work \citep{farahmand2016regularized} studies the performance of LSPI and BRM when the value function belongs to a reproducing kernel Hilbert space. However, we study  the fitted Q-iteration algorithm, which is a batch RL counterpart of DQN.
The fitted Q-iteration algorithm  is  proposed in \cite{ernst2005tree}, and \cite{riedmiller2005neural} proposes the neural FQI algorithm. Finite-sample bounds for   FQI   have been established in \cite{murphy2005generalization, munos2008finite} for large classes of regressors. However, their  results are not applicable to  ReLU networks  due to the huge capacity of deep neural networks. Furthermore, various extensions of FQI are
  studied  in  \cite{antos2008fitted, farahmand2009regularized,tosatto2017boosted, geist2019theory} to handle continuous actions space,  ensemble learning, and entropy regularization. The empirical performances of various batch RL methods have been examined in 
   \cite{levine2017shallow, agarwal2019striving, fujimoto2019benchmarking}.

Moreover, Q-learning, and reinforcement  learning methods in general, have been widely applied to  dynamic treatment regimes (DTR)  \citep{chakraborty2013statistical,laber2014dynamic, tsiatis2019dynamic}, where the goal is to find  sequential decision rules for individual patients that   adapt to time-evolving illnesses.  There is a huge body of literature on this line of research. See, e.g., \cite{murphy2003optimal, zhao2009reinforcement, qian2011performance, zhao2011reinforcement, zhang2012robust, zhao2012estimating, goldberg2012q, nahum2012q, goldberg2013adaptive, schulte2014q,  song2015penalized, zhao2015new, linn2017interactive, zhou2017residual, shi2018high, zhu2019proper} and the references therein. Our work provides a theoretical underpinning for the application of DQN to DTR \citep{liu2019learning} and motivates the principled usage of DRL methods in healthcare applications \citep{yu2019reinforcement}.

Furthermore, our work is also related to works that apply reinforcement learning to zero-sum Markov games. The Minimax-Q learning is proposed by \cite{littman1994markov}, which is an online algorithm that is an extension Q-learning.  Subsequently,  for Markov games, various online algorithms are also proposed with theoretical guarantees. These works consider either the tabular case or linear function approximation. See, e.g.,   \cite{bowling2001rational,conitzer2007awesome,  prasad2015two, wei2017online, perolat2018actor,srinivasan2018actor, wei2017online}   and the references therein. In addition, batch reinforcement learning is also applied to zero-sum Markov games by  \cite{lagoudakis2002value,perolat2015approximate, perolat2016softened,perolat2016use,zhang2018finite}, which are closely related to our work. All of these works consider either   linear function approximation  or a general function class with bounded   pseudo-dimension \citep{anthony2009neural}.  However, there results cannot directly imply finite-sample bounds for Minimax-DQN due to the  huge capacity of deep neural networks.

Finally, our work is also related a line of research on the model capacity of ReLU deep neural networks, which leads to understanding the generalization property of deep learning \citep{mohri2012foundations, kawaguchi2017generalization}. Specifically, \cite{bartlett1998sample, neyshabur2015norm, neyshabur2015path,bartlett2017spectrally,golowich2017size, liang2017fisher} propose various norms computed from the networks parameters and establish capacity bounds based upon these norms. In addition,  \cite{maass1994neural, bartlett1999almost, schmidt2017nonparametric, harvey2017nearly, klusowski2016risk, barron2018approximation, suzuki2018adaptivity, bauer2019deep} study the Vapnik-Chervonenkis (VC) dimension of  neural networks and \cite{dziugaite2017computing, neyshabur2017pac} establish the    PAC-Bayes bounds for neural networks.   Among these works, our work is more related to \cite{schmidt2017nonparametric,suzuki2018adaptivity}, which relate the VC dimension of the ReLU  networks to a set of hyperparameters used to define the networks. Based on the VC dimension, they study the statistical error of nonparametric regression using ReLU networks. In sum, theoretical understanding of deep learning is pertinent to the study of DRL algorithms. See \cite{kawaguchi2017generalization, neyshabur2017exploring, fan2019selective} and the references therein for recent developments on theoretical analysis of the generalization property of deep learning.

\subsection{Notation}

 For  a measurable space with domain $\cS$, we denote by $\cB(\cS, V)$ the set of   measurable functions on $\cS$ that are bounded by $V  $ in absolute value.
  Let $\cP(\cS)$ be the set of all probability measures over $\cS$.  For any $\nu \in \cP(\cS)$ and any measurable function $f \colon \cS \rightarrow \RR$, we denote by $\| f\|_{\nu, p}$ the $\ell_p$-norm of $f$ with respect to measure $\nu$ for $ p \geq 1$. In addition, for simplicity, we write $\| f \|_{\nu}$ for $\| f \|_{2, \nu} $. In addition, let $\{ f(n), g(n) \}_{n \geq 1} $ be two positive series. We write $f(n) \lesssim g(n)$ if there exists a constant $C$ such that $f(n) \leq C\cdot  g(n)$ for all $n   $ larger than some $n_0 \in \NN$. In addition, we write $f(n) \asymp g(n)$ if $f(n) \lesssim g(n)$ and $g(n) \lesssim f(n)$.


\section{Background}
In this section, we introduce the background. We first lay out the formulation of the reinforcement learning problem, and then define the family of  ReLU neural networks.

\subsection{Reinforcement Learning} \label{eq:rl_background}
A  discounted Markov decision process is defined by a tuple $(\cS, \cA, P, R, \gamma )$. Here $\cS$ is the set of all states, which can be countable or uncountable,  $\cA$ is the set of all  actions, $P\colon \cS\times \cA \rightarrow \cP(\cS)$ is the Markov transition kernel,  $R \colon \cS \times \cA \rightarrow \cP(\RR)$ is the distribution of the immediate reward, and $\gamma \in (0,1)$ is the discount factor.
In specific, upon taking any action $a\in \cA$ at any state $s\in \cS$, $P(\cdot\given s,a)$ defines the probability distribution of the next state and $R(\cdot \given s, a)$ is the distribution of the immediate reward. Moreover, for regularity, we further assume that $\cS $ is a compact subset of $\RR^r$ which can be infinite,  $\cA = \{ a_1, a_2, \ldots, a_M\}$ has finite cardinality $M$, and the rewards are uniformly bounded by $R_{\max}$, i.e., $R(\cdot~| s, a)$ has a range on $[-R_{\max}, R_{\max}]$ for any $s\in \cS$ and $a \in \cA$.

A policy $\pi \colon \cS \rightarrow \cP(\cA)$ for the MDP   maps any state  $s\in \cS $ to a probability distribution  $\pi(\cdot \given s)$ over $\cA$.
For a given policy $\pi$, starting from the initial state $S_0 = s$, the actions, rewards, and states evolve according to the law as follows:
$$
 \Big \{ (A_t, R_t, S_{t+1}):  A_t \sim \pi(\cdot \given S_t) , R_{t}\sim R(\cdot \given S_t, A_t), S_{t+1}\sim P(\cdot \given S_t, A_t)  \Big \}, t=0, 1, \cdots, 
$$
and the corresponding value function $V^{\pi}\colon \cS \rightarrow \RR$ is defined as the cumulative discounted reward obtained by taking the actions according to $\pi$ when starting from a fixed state, that is,
\# \label{eq:value_fun}
V^\pi(s) = \EE  \bigg [\sum_{t=0}^\infty \gamma^t \cdot R_t\bigggiven S_0 = s \bigg ]. 
\#
The policy $\pi$ can be controlled by decision makers, yet the functions $P$ and $R$ are given by the nature or the system that are unknown to decision makers.

By the law of iterative expectation,  for any policy $\pi$,  
\# \label{fan1}
  V^\pi(s) = \EE \big[ Q^\pi(s, A) \biggiven A \sim \pi(\cdot \given s) \big ], \qquad \forall s\in \cS,
\#
where  $Q^\pi(s, a)$, called an action value function, is given by
\# \label{eq:Q_fun}
   Q^\pi(s, a)   = \EE\bigg[\sum_{t=0}^\infty \gamma^t \cdot  R_t \bigggiven S_0 = s, A_0 = a \bigg] = r(s, a) + \gamma \cdot \EE \bigl [ V^\pi(S')  \biggiven S' \sim P(\cdot \given s,a) \bigr ] ,
\#
with $r(s,a) = \int r R(\ud r \given s,a)$ is the expected reward at state $s$ given action $a$.
Moreover, we define an operator $P^{\pi}$ by
\#\label{eq:operator_p}
(P^\pi Q) (s, a) =  \EE \bigl [Q(S^\prime, A^\prime) \biggiven S^\prime \sim P(\cdot \given s,a), A^\prime \sim \pi(\cdot \given S^\prime) \bigr ],
\#
and define the
 Bellman operator $T^\pi  $ by
$
(T^\pi Q) (s, a) = r(s, a) + \gamma \cdot (P^\pi Q) (s, a).%
$
Then $Q^{\pi}$ in \eqref{eq:Q_fun} is the unique fixed point of $T^\pi$. 

The goal of reinforcement learning is to find the optimal policy, which achieves the largest cumulative reward via  dynamically learning from the acquired data.   To characterize optimality, 
by \eqref{fan1}, we naturally define the optimal action-value function $Q^*$ as
\#\label{eq:optimal_Q}
	Q^*(s,a) = \sup_{\pi } Q^\pi(s,a), \qquad \forall (s,a) \in \cS\times \cA.
\#
where the supremum is taken over all policies.
In addition, for any given action-value function $Q\colon \cS \times \cA \rightarrow \RR$, we define the  greedy policy $\pi_Q$ as any policy that selects the action with the largest $Q$-value, that is, for any $s \in \cS$, $\pi_Q(\cdot ~ | s)$ satisfies
\#\label{eq:greedy_policy}
\pi_{Q}(a \given s ) =0 ~~\text{if }~~Q(s,  a )   \neq  \max_{a'\in \cA} Q(s,a')   .
	 \#
Based on $Q^*$, we  define the optimal policy $\pi^*$ as any policy that is greedy with respect to $Q^*$.  It can be shown that $Q^* = Q^{\pi^*}$.
Finally, we define the Bellman optimality operator $T$  via
\# \label{eq:optimal_bellman}
	(TQ) (s,a) = r(s, a) + \gamma \cdot \EE \Bigl[ \max_{a^\prime \in \cA} Q(S^\prime, a^\prime) \biggiven S^\prime \sim P(\cdot  \given s,a) \Bigr].
\#
Then we have the Bellman optimality equation $T Q^* = Q^*$.

Furthermore, 
 it can be verified that
the Bellman operator $T$ is $\gamma$-contractive with respect to the supremum norm over $\cS\times \cA$. That is, for any two action-value functions $Q, Q' \colon \cS\times \cA\rightarrow \RR$, it holds that $\| TQ - TQ' \|_{\infty} \leq \gamma \cdot \| Q - Q '  \|_{\infty} $.
Such a contraction property yields the celebrated value iteration algorithm \citep{sutton2011reinforcement}, which constructs a sequence of action-value functions $\{ Q_k \} _{k \geq 0}$ by letting $Q_k = T Q_{k-1}$ for all $k \geq 1$, where the initialization function  $Q_0$ is  arbitrary.
Then it holds that $\| Q_k - Q^* \|_{\infty} \leq \gamma^k \cdot \| Q_0 - Q^*\|_{\infty}$, i.e., $\{ Q_k \}_{ k\geq 0}$ converges to the optimal value function at a linear rate.
 This approach forms the basis of the  neural FQI algorithm, where the Bellman operator is empirically learned from a batch of data dynamically and the action-value functions are approximated by  deep neural networks.

%
%

\subsection{Deep Neural Network}

We study the performance of DQN  with rectified linear unit (ReLU) activation function $\sigma(u) = \max(u, 0)$.  For any positive integer $L $ and $\{ d_j \}_{ i=0}^{L+1} \subseteq \NN $, a ReLU network $f \colon  \RR^{d_0} \rightarrow \RR^{d_{L+1}}$ with $L$ hidden layers and width  $\{ d_j \}_{ i=0}^{L+1}$ is of form
\#\label{eq:relu_net}
f(x) = W_{L+1} \sigma( W_{L} \sigma(W_{L-1}  \ldots \sigma( W_2 \sigma(W_1 x + v_1 ) + v_2)\ldots v_{L-1})+v_L) ,
\#
where $W_{\ell} \in \RR^{d_{\ell} \times d_{\ell-1}}$ and $v_\ell \in \RR^{d_\ell}$ are the weight matrix and the shift (bias) vector in the $\ell$-th layer, respectively.
Here we apply $\sigma$ to to each entry of its argument in \eqref{eq:relu_net}.
In deep learning, the network structure is fixed, and the goal is to learn the network parameters (weights) $\{W_\ell, v_\ell\}_{\ell \in [L+1]}$ with the convention that $v_{L+1}  =0$. For deep neural networks, the number of parameters  greatly exceeds the input dimension $d_0$. To restrict the model class, we focus on the class of ReLU  networks where most parameters are zero.
 \begin{definition} [Sparse ReLU Network] \label{def:relu}
For any  $L, s \in \NN$,   $\{ d_j \}_{ i=0}^{L+1} \subseteq \NN $ ,   and $V>0$, the family of sparse ReLU networks bounded by $V$  with $L$ hidden layers, network width $d$, and  weight sparsity   $s$ is defined as
\#\label{eq:relu_func_class}
\cF  ( L, \{ d_j \}_{ i=0}^{L+1} , s, V  ) = \biggl \{ f \colon \max _{\ell \in [L+1]}    \| \tilde W_{\ell}    \|_{\infty}    \leq 1,  \sum _{\ell =1}^{L+1}    \| \tilde W_{\ell}  \|_0    \leq s, \max_{j \in [d_{L+1}]} \| f_j \|_{\infty} \leq V \biggr  \},
\#
where we denote   $(W_{\ell} , v_{\ell} ) $ by $\tilde W_{\ell}$. Moreover, $f $ in \eqref{eq:relu_func_class} is   expressed as in  \eqref{eq:relu_net}, and $f_j$ is the  $j$-th component of~$f$.
 \end{definition}
Here we focus on functions that are uniformly bounded because the value functions in \eqref{eq:value_fun} and \eqref{eq:Q_fun} are always bounded by $V_{\max} = R_{\max} / (1- \gamma)$.
We also assume that the network weights are uniformly bounded and bounded by one without loss of generality. 
In the sequel, we write $\cF(L,\{ d_j \}_{ j=0}^{L+1}, s, V_{\max} ) $ as $\cF(L,  \{ d_j \}_{j=0}^{L+1} , s)$ to simplify the notation.
In addition, we restrict the networks weights to be sparse, i.e., $s$ is much smaller compared with the total number of parameters. Such an assumption implies that the network has  sparse connections, which are useful for applying  deep learning in memory-constrained situations such as mobile devices \citep{han2015deep,liu2015sparse}.
Empirically, sparse  neural networks are realized via various regularization techniques  such as 
Dropout \citep{srivastava2014dropout}, which randomly sets a fixed portion of the network weights to zero. Moreover, sparse network architectures have recently been advocated by the intriguing lottery ticket hypothesis  \citep{frankle2018lottery},  which states that each dense network has a subnetwork with the sparse connections, when trained in isolation, achieves comparable performance as the original network. Thus, focusing on the class of sparse ReLU networks does not sacrifice the statistical accuracy.

Moreover, we introduce the notion of  H\"older smoothness as follows, which is a generalization of Lipschitz continuity, and is widely used to characterize the regularity of functions.

\begin{definition} [H\"older Smooth Function] \label{def:holder_func_class}
Let $\cD$ be a  compact subset of $\RR^r$, where $r\in \NN$. We define the set of H\"older smooth functions   on $\cD$ as
\$
\cC_r(\cD, \beta, H) = \biggl  \{ f\colon \cD \rightarrow \RR \colon \sum_{\balpha \colon | \balpha | < \beta } \| \partial ^{\balpha } f \|_{\infty} +  \sum_{\balpha \colon \| \balpha\|_1 =  \lfloor \beta \rfloor }  \sup_{x, y \in \cD, x\neq y}  \frac{ |  \partial ^{\balpha}  f(x) - \partial ^{\balpha} ( y) | }{ \| x - y \|_{\infty} ^{\beta   - \lfloor \beta \rfloor }}  \leq H \biggr \},
\$
where $\beta> 0$ and $H>0$ are parameters and $\lfloor \beta \rfloor$ is the largest integer no greater than $\beta$. In addition, here  we use the multi-index notation by letting  $\balpha = ( \alpha_1, \ldots, \alpha_r) ^{\top} \in \NN^{r}$,  and $\partial ^{\balpha } = \partial ^{\alpha_1} \ldots \partial ^{\alpha_r} $.
\end{definition}

Finally, we conclude this section by defining functions that can be written as a composition of multiple  H\"older functions, which captures  complex mappings in real-world applications such as multi-level feature extraction.

\begin{definition} [Composition of H\"older Functions]\label{def:comp:holder}
Let  $q \in \NN $ and $\{ p_j \}_{ j \in [q] } \subseteq \NN $ be integers, and let $ \{ a_j ,b_j \}_{j \in [q] }  \subseteq \RR $ such that $a_j < b_j $  $ j  \in [q]$. Moreover,  let   $g_{j} \colon [ a_j, b_j]^{p_j} \rightarrow [ a_{j+1}, b_{j+1} ]^{p_{j+1} } $ be a function,  $\forall j  \in [q]$. Let $( g_{jk} )_{k\in [ p_{j+1}]}$ be the components of $g_j$,  and we assume that each $g_{jk}$ is H\"older smooth, and  depends on at most $t_j$ of its input variables, where $t_j$ could be much smaller than $p_j$, i.e.,    $g_{jk} \in \cC_{t_j} ( [ a_j, b_j]^{t_j}, \beta_j, H_j)$. Finally, we denote by
$\cG( \{ p_j, t_j , \beta_j ,  H_j \}_{j \in[q]} )$ the family of functions that can be written as compositions of  $\{ g_j \}_{j \in [q]}$, with the convention that $p_{q+1} = 1$.  That is,  for any $f \in \cG( \{ p_j, t_j , \beta_j ,   H_j \}_{j \in[q]} )$, we can write
\#\label{eq:composition}
 f = g_{q} \circ g_{q-1} \circ \ldots \circ g_2 \circ g_1,
\#
with $g_{jk } \in \cC_{t_j} ( [ a_j, b_j]^{t_j}, \beta_j, H_j)$ for each $k\in [p_{j+1}]$ and $j \in [q]$.
\end{definition}

Here $f$ in \eqref{eq:composition} is a composition of $q$ vector-valued mappings $\{g_ j \}_{j\in [q]}$ where each $g_j$ has $p_{j+1}$ components and its $k$-th component, $g_{jk}$, $\forall k\in [p_{j+1}]$, is a H\"older smooth function defined on $[a_j, b_j]^{p_j}$.
Moreover, it is well-known the statistical rate for  estimating a H\"older smooth function depends on the input dimension \citep{tsybakov2008introduction}.  
Here we assume 
that $g_{jk}$ only depends on $t_j$ of its inputs, where $t_j \in [p_j]$ can be much smaller than $p_j$,  
which enables us to obtain a more refined analysis that adapts to the effective smoothness of $f$. 
In particular, Definition \ref{def:comp:holder} covers the family of H\"older smooth functions  and  the additive model \citep{friedman1981projection} on  $[0,1]^r$ as two special cases, where the former suffers from the curse of dimensionality whereas the latter does not.


\section{Understanding Deep Q-Network} \label{sec:algo}
In the  DQN algorithm,  a deep neural network $Q_{\theta} \colon \cS \times \cA \rightarrow \RR$ is used to approximate $Q^*$, where $\theta$ is the parameter.
For completeness, we state     DQN as  Algorithm \ref{algo:dqn} in   \S\ref{sec:state_dqn}.  As shown in the experiments in \citep{mnih2015human}, two tricks are pivotal for the empirical success of DQN.

First, DQN use the trick of experience replay \citep{lin1992self}. Specifically, at each time $t$, we store the  transition $(S_t, A_t, R_t , S_{t+1}) $ into the  replay memory $\cM$, and then sample a minibatch of independent samples from $\cM$ to train the neural network via stochastic gradient descent. Since the trajectory of MDP has strong temporal correlation, the goal of experience replay is to
obtain uncorrelated  samples, which yields accurate gradient estimation for the stochastic optimization problem.

Another trick is to use a target network $Q_{\theta^\star}  $ with parameter $\theta^\star$ (current estimate of parameter).  With  independent samples $\{(s_i,a_i,r_i,s'_i)\}_{i \in [n]}$  from the replay memory (we use $s'_i$ instead of $s_{i+1}$ for the next state right after $s_i$ and $a_i$ to avoid notation crash with next independent sample $s_{i+1}$ in the state space), to update the parameter $\theta$ of the Q-network, we  compute the target
$$
 Y_i  = r_i +\gamma \cdot  \max_{a\in \cA} Q_{\theta^\star} (s_i', a )
$$
(compare with Bellman optimality operator \eqref{eq:optimal_bellman}), and update $\theta$ by the gradient of
\#\label{eq:target_loss}
L(\theta) = \frac{1}{n} \sum_{i=1}^n [ Y_i - Q _{\theta} (s_i, a_i ) ] ^2.
\#
 Whereas parameter $\theta^{\star}$ is updated once every $T_{\text{target}}$ steps by letting  $\theta^\star = \theta$. That is, the target network is hold fixed for  $T_{\text{target}}$ steps and then updated it by the current weights of the Q-network. 

To demystify DQN, it is crucial to understand the role played by these two tricks. For experience replay, in practice, the replay memory size is usually very large. For example, the replay memory size   is $10^6$ in   \cite{mnih2015human}. Moreover, DQN use the $\epsilon$-greedy policy, which enables exploration over $\cS \times \cA$. Thus, when  the replay memory is large, experience replay is close to sampling independent transitions from an explorative policy.  This  reduces the variance of the $\nabla L(\theta)$, which is used to update $\theta$.
Thus,  experience replay stabilizes the training of DQN, which benefits the algorithm in terms of computation.

To understand the statistical property of DQN, we   replace the experience replay by sampling independent transitions from   a given distribution $\sigma \in \cP( \cS\times \cA) $.    That is, instead of sampling  from the replay memory, we sample i.i.d. observations $\{ (S_i, A_i ) \}_{i\in [n]}$ from $\sigma $. Moreover, for any $i \in [n]$, let $R_i $ and $S_i'$ be the immediate reward  and the next state when taking action $A_i$ at state $S_i$. Under this setting, we have
$ 
\EE ( Y_i \given S_i, A_i ) = (TQ_{\theta^\star} )( S_i, A_i ) ,
$
where $T$ is the Bellman optimality operator in \eqref{eq:optimal_bellman} and $Q_{\theta^\star}$ is the target network.  

Furthermore, to further understand the necessity of the target network,  let us first neglect the target network and set $\theta^\star = \theta$. Using bias-variance decomposition,  the
 the expected value of $L(\theta) $ in \eqref{eq:target_loss} is
 \#\label{eq:bias_variance}
\EE \bigl[L(\theta)\bigr]  = \| Q_{\theta}  -  T Q _{\theta }    \|_{\sigma}^2 + \EE \Bigl \{  \bigl [ Y_1 - (T Q _{\theta }) (S_1, A_1) \bigr ]^2 \Bigr \}.
\#
Here the first term in \eqref{eq:bias_variance} is known as the mean-squared Bellman error (MSBE), and the second term is the variance of $Y_1$. Whereas $L(\theta)$ can be viewed as the empirical version of the MSBE, which has   bias $\EE   \{  [ Y_1 - (T Q _{\theta }) (S_1, A_1)  ]^2   \}$ that also depends on $\theta$. Thus, without the target network, minimizing $L(\theta)$ can be drastically different from minimizing the MSBE.

To resolve this problem, we use a target network in \eqref{eq:target_loss}, which has expectation
\$
\EE \bigl[L(\theta)\bigr]= \| Q_{\theta}  -  T Q _{\theta ^*}    \|_{\sigma}^2 + \EE \Bigl \{  \bigl [ Y_1 - (T Q _{\theta^* }) (S_1, A_1) \bigr ]^2 \Bigr \},
\$
where the variance of $Y_1$ does not depend on $\theta$. Thus, minimizing $L(\theta)$ is close to solving
\#\label{eq:ideal_case}
\minimize _{ \theta \in \Theta}  \| Q_{\theta}  -  T Q _{\theta^\star}    \|_{\sigma}^2,
\#
where $\Theta$ is the parameter space.
Note that in DQN we hold $ \theta^\star $ still and update $\theta$ for $T_{\text{target}}$ steps. When $T_{\text{target}}$ is sufficiently large and we neglect the fact that the objective in \eqref{eq:ideal_case} is nonconvex,  we would update $\theta$ by the minimizer of \eqref{eq:ideal_case}
for fixed $\theta^\star$.

Therefore, in the ideal case, DQN aims to solve the minimization problem \eqref{eq:ideal_case} with $\theta^\star$ fixed, and then update $\theta^\star$ by the minimizer $\theta$.
Interestingly, this view of DQN  offers a statistical  interpretation of the target network.
In specific,  if  $\{ Q_{\theta} \colon\theta \in  \Theta\} $ is sufficiently large such that it contains $TQ _{\theta^\star}$, then \eqref{eq:ideal_case} has solution $Q_{\theta } = T Q_{\theta^\star}$, which can be viewed as one-step of  value iteration    \citep{sutton2011reinforcement} for neural networks. In addition,   in the sample setting,
 $Q_{\theta^\star}$     is used to construct $\{ Y_i \}_{i \in [n]}$, which serve as the response in the regression problem defined in \eqref{eq:target_loss}, with $(TQ_{\theta^\star}) $ being the regression function.

 Furthermore,  turning the above discussion   into a realizable algorithm, we obtain the  neural fitted Q-iteration (FQI) algorithm, which generates a sequence of value functions.   Specifically, let $\cF$ be a class of function defined on $\cS\times \cA$. In the  $k$-th   iteration of FQI,     let $\tilde Q_k$ be current    estimate of $Q^*$.
Similar to \eqref{eq:target_loss} and \eqref{eq:ideal_case}, we  define $Y_i = R_i + \gamma \cdot \max _{a\in \cA} \tilde Q_k (S_i', a)$, and update $\tilde Q_k$ by
 \#\label{eq:iteration2}
\tilde Q_{k+1} =  \argmin_{f \in \cF} \frac{1}{n} \sum_{i=1}^n \bigl  [ Y_i - f(S_i, A_i) \bigr ]^2.
\#
This gives the fitted-Q iteration algorithm, which is stated in Algorithm \ref{algo:fit_Q}.

The step of minimization problem in  \eqref{eq:iteration2} 
essentially finds $\tilde Q_{k+1}$ in $\cF$ such that 
 $\tilde Q_{k+1}  \approx T \tilde Q_{k}$.  Let us denote $\tilde Q_{k+1}  = \hat{T} _k \tilde Q_{k}$ where  $\hat{T}_k$ is an approximation of the  Bellman optimality operator $T$  learned from the training data  in the $k$-th iteration.
With the above notation, we can now understand our  Algorithm \ref{algo:fit_Q} as follows.  Starting from the initial estimator $\tilde Q_{0}$, collect the data $\{(S_i,A_i,R_i,S'_i)\}_{i \in [n]}$ and learn the map $\hat{T}_1$ via \eqref{eq:iteration2} and get $\tilde Q_{1} = \hat{T}_1 \tilde Q_{0}$.  Then, get a new batch of sample and learn the map $\hat{T}_2$ and get $\tilde Q_{2} = \hat{T}_2 \tilde Q_{1}$, and so on.  Our final estimator of the action value is $\tilde Q_{K} =   \hat{T}_K \cdots \hat{T}_1 \tilde Q_{0}$, which resembles the updates of the value iteration algorithm at the population level.

When $\cF$ is the family of neural networks, Algorithm \ref{algo:fit_Q} is known as the neural FQI algorithm, which is proposed in  \citet{riedmiller2005neural}.  Thus, we can view neural FQI  as a modification  of DQN, where we replace experience replay with sampling from a fixed distribution $\sigma$,  so as to understand the   statistical property.  As a byproduct,  such a modification naturally  justifies the trick of   target network in DQN.
 In addition, note that the optimization problem in \eqref{eq:iteration2} appears in each iteration of FQI, which is nonconvex when neural networks are used. However, since we focus  solely  on  the statistical aspect, we make the assumption that the global optima of \eqref{eq:iteration2} can be reached, which is also contained $\cF$.
 Interestingly, a recent line of research on deep learning  \citep{du2018gradient2, du2018gradient, zou2018stochastic, chizat2018note, allen2018learning, allen2018convergence,  jacot2018neural, cao2019generalization, arora2019fine,  ma2019comparative,mei2019mean, yehudai2019power} has established
 global convergence of gradient-based  algorithms for empirical risk minimization  when the neural networks are overparametrized.
 We provide more discussions on the computation aspect in \S\ref{sec:computation}.
Furthermore, we make the i.i.d. assumption   in Algorithm \ref{algo:fit_Q} to simplify the analysis.
 \cite{antos2008learning} study the performance of fitted value iteration with    fixed data  used in the regression sub-problems repeatedly, where the data is  sampled from a single trajectory based on a fixed   policy such that the induced Markov chain satisfies certain conditions on the mixing time.
 Using similar analysis as  in \citet{antos2008learning},  our algorithm can also be extended to handled fixed data that is collected beforehand.

\begin{algorithm} [h]
\caption{Fitted Q-Iteration Algorithm} 
\label{algo:fit_Q} 
\begin{algorithmic} 
\STATE{{\textbf{Input:}} MDP $(\cS, \cA, P, R, \gamma)$, function class $\cF$, sampling distribution $\sigma$, number of  iterations $K$, number of samples $n$, the initial estimator $\tilde Q_0$.}
\FOR{$k = 0, 1, 2, \ldots, K-1$}
\STATE{Sample i.i.d. observations $\{(S_i, A_i, R_i, S_i')\}_{i\in[n]} $ with $(S_i, A_i) $ drawn from distribution $ \sigma$.}
\STATE{Compute $Y_i = R_i + \gamma \cdot \max _{a\in \cA} \tilde Q_k (S_i', a)$.}
\STATE{Update the action-value function: \$\tilde Q_{k+1} \leftarrow  \argmin_{f \in \cF} \frac{1}{n}\sum_{i=1}^n \bigl [ Y_i - f(S_i, A_i) \bigr ]^2.\$}
\ENDFOR
\STATE{Define policy $\pi_K$ as the greedy policy with respect to $\tilde Q_{K}$.}
\STATE{{\textbf{Output:}} An estimator $\tilde Q_{K }$ of $Q^*$ and policy $\pi_K$.}
\end{algorithmic}
\end{algorithm}

 \section{Theoretical Results} \label{sec:theory}

 We establish statistical guarantees for DQN with     ReLU networks. Specifically, let $Q^{\pi_K}$ be the action-value function corresponding to $\pi_K$, which is returned by Algorithm \ref{algo:fit_Q}. In the following, we obtain an upper bound for  $\| Q^{\pi_K} - Q^* \|_{1, \mu} $, where $\mu \in \cP(\cS\times \cA)$ is allowed to be different from $\nu$.      In addition, we  assume that the state space $\cS$ is a compact subset in $\RR^r$ and the action space $\cA$ is finite. Without loss of generality, we let  $\cS = [0,1]^{r}$ hereafter, where $r$ is a fixed integer. To begin with,  we first specify the function class $\cF$   in Algorithm \ref{algo:fit_Q}.

\begin{definition}[Function Classes] \label{def:func_class}
Following Definition \ref{def:relu}, let $\cF(L, \{d_j \}_{j=0}^{L+1} , s  ) $ be the family of sparse  ReLU networks defined on $\cS$   with $d_0 = r$ and  $d_{L+1} =1$.  Then we define $\cF_0$ by
\#\label{eq:define_cF}
\cF_0 = \bigl \{ f \colon \cS \times \cA \rightarrow \RR   \colon   f(\cdot, a) \in \cF(L, \{ d_j\}_{i=0}^{L+1} , s ) ~\text{for any}~a\in \cA \bigr \}.
\#
 In addition, let $\cG( \{ p_j, t_j , \beta_j ,  H_j \}_{j \in[q]} )  $ be set of   composition of  H\"older smooth functions defined on $\cS\subseteq \RR^r$. Similar to $\cF_0$, we define a function class $\cG_0$ as
\#\label{eq:define_cG}
\cG _0= \bigl \{ f\colon \cS \times \cA \rightarrow \RR  \colon f(\cdot, a) \in \cG( \{ p_j, t_j , \beta_j ,  H_j \}_{j \in[q]} ) ~\text{for any}~a\in \cA \bigr \}.
\#
\end{definition}

By this definition,  for any function $f \in \cF_0$ and any action $a\in \cA$,  $f(\cdot, a)$ is a   ReLU network defined on $\cS$, which is standard for Q-networks. Moreover, $\cG_0$ contains a broad  family of smooth functions on $\cS\times \cA$. In the following, we make a  mild assumption on $\cF_0$ and $\cG_0$.

\begin{assumption}  \label{assume:closedness}
 We assume that for any $f \in \cF_0$, we have $Tf \in \cG_0$, where $T$ is the Bellman optimality operator defined in \eqref{eq:optimal_bellman}. That is, for any $f \in \cF$ and any $a \in \cA$, $(Tf)(s,a) $  can be written as compositions of H\"older smooth functions as a function of $s \in \cS$.
\end{assumption}

This assumption 
specifies that the target function $\cT \tilde Q_k$  in each FQI step stays in function class $\cG_0$. 
When $\cG_0$ can be approximated by functions in $\cF_0$  accurately, this assumption essentially implies that $Q^*$ is close to $\cF_0$ and that $\cF_0$ is  approximately closed  under  Bellman operator $T$. Such an  completeness assumption is commonly made in the  literature on batch reinforcement learning under various forms  and is conjectured to be indispensable in \cite{chen2019information}.

 We remark that this Assumption \eqref{assume:closedness} holds when the MDP satisfies some smoothness conditions.  For any state-action pair $(s,a) \in \cS \times \cA$, let $P(\cdot \given s,a)$ be the density of the next state.  By the definition of the Bellman optimality operator in \eqref{eq:optimal_bellman}, we have
 \#\label{eq:bellman_oper_density}
 (T f) (s,a)  = r(s,a) + \gamma \cdot \int_{\cS} \Bigl [ \max_{a'\in \cA} f(s', a') \Bigr] \cdot P(s'\given s,a) \ud s'.
 \#
 For any  $s' \in \cS$ and  $a \in \cA$,  we define functions $g_1, g_2  $ by letting $g_1 (s) = r(s,a) $ and $g_2(s) =   P(s' \given  s  , a) $. Suppose both $g_1$ and $g_2$ are   H\"older smooth  functions  on $\cS = [0,1]^r$ with parameters $\beta $ and $H $. Since $\| f\|_{\infty} \leq V_{\max}$, by  changing the order of integration and differentiation with respect to $s$ in \eqref{eq:bellman_oper_density}, we obtain that function $s\rightarrow (Tf)(s,a)$ belongs to the H\"older class $\cC_r(\cS, \beta , H')$ with $H'= H ( 1+ V_{\max}   )$.
 Furthermore, in the more general case, suppose for any fixed $a\in \cA$, we can write  $P(s' \given  s  , a) $ as $  h_1[  h_2(s, a), h_3(s') ]  $, where  $h_2 \colon \cS \rightarrow \RR^{r_1}$, and  $h_3 \colon \cS \rightarrow \RR^{r_2}$ can be viewed as feature mappings, and   $h_1 \colon \RR^{r_1+r_2} \rightarrow \RR$ is a bivariate function. We define   function $h_4 \colon \RR^{r_1} \rightarrow \RR$    by
  \$
  h_4 (u ) = \int_{\cS} \Bigl [ \max_{a'\in \cA} f(s', a') \Bigr] h_1 \bigl [ u, h_3(s')  \bigr ]  \ud s'.
  \$ Then by \eqref{eq:bellman_oper_density} we have
  $ (Tf)(s,a) = g_1(s) + h_4\circ h_2(s, a)$.  Then Assumption \ref{assume:closedness} holds  if $h_4$ is H\"older smooth and both $g_1$ and $h_2$ can be represented as compositions of H\"older functions.
  Thus, Assumption \ref{assume:closedness} holds if both the reward function and the transition density of the MDP are sufficiently smooth.

Moreover, even when the transition density   is not smooth, we could also expect Assumption \ref{assume:closedness} to hold. Consider  the extreme case where the MDP has   deterministic transitions, that is, the next state $s' $ is a function of $s$ and $a$, which is denoted by $s' = h(s,a)$. In this case, for any ReLU network $f$, we have
$
(Tf) (s,a) = r(s,a) + \gamma \cdot \max _{a'\in \cA} f[ h(s,a) , a'].
$
Since  \$ \Bigl | \max_{a'\in \cA}  f(s_1, a') - \max_{a'\in \cA}  f(s_2, a') \Bigr |  \leq \max _{a' \in \cA} \bigl | f(s_1, a') - f(s_2, a') \bigr |\$ for any $s_1, s_2 \in \cS$,  and network $f(\cdot, a)$ is Lipschitz continuous for any fixed $a \in \cA$, function $m_1(s) =    \max_{a'} f(s, a')$ is   Lipschitz  on $\cS$. Thus,
for any fixed $a\in \cA$, if both $g_1(s) = r(s,a)$ and $m_2(s) = h(s,a) $ are  compositions    of  H\"older functions, so is    $(Tf)(s,a) = g_1(s) + m_1 \circ m_2(s)$.  Therefore, even if the MDP has deterministic dynamics, when both the reward function $r(s,a)$ and the transition function $h(s,a)$ are sufficiently nice, Assumption \ref{assume:closedness} still holds true.

In the following, we define the  concentration coefficients, which  measures the similarity  between two probability distributions under the   MDP.



 \begin{assumption} [Concentration Coefficients]
 \label{assume:concentrability}
  Let $\nu_1, \nu_2 \in \cP(\cS \times \cA)$ be two probability measures that are  absolutely continuous with respect to the Lebesgue measure on $\cS\times \cA$.  Let   $\{ \pi_t \}_{t\geq 1} $ be a sequence of policies. Suppose the initial state-action pair  $(S_0, A_0)$ of the MDP has distribution $ \nu_1$, and we take action $A_t$ according to policy $\pi_{ t}$.    For any integer $m$, we denote by $ P^{\pi_m} P^{\pi_{m-1} } \cdots P^{\pi_1}\nu_1$ the   distribution of $\{(S_t, A_t)\}_{t=0}^m$.
Then we  define the $m$-th  concentration coefficient as
 \#\label{eq:concentration_coef}
 \kappa(m; \nu_1, \nu_2) = \sup_{\pi_1, \ldots, \pi_m } \biggl [\EE _{\nu_2} \biggl | \frac{    \ud ( P^{\pi_m} P^{\pi_{m-1} } \cdots P^{\pi_1}\nu_1)   } { \ud \nu_2} \biggr | ^2     \biggr ] ^{1/2},
 \#
 where the supremum is taken over all possible policies.

 Furthermore, let $\sigma$ be the sampling distribution in Algorithm \ref{algo:fit_Q} and let $\mu$  be a fixed distribution on $\cS \times \cA$.  We assume that there exists a constant $\phi_{\mu, \sigma}< \infty$ such that
  \#\label{eq:assume:concentrability}
  ( 1- \gamma)^{2}\cdot \sum_{m \geq 1} \gamma^{m-1} \cdot m \cdot \kappa ( m; \mu, \sigma) \leq \phi_{\mu, \sigma},
  \#
where  $(1-\gamma)^{2}$ in \eqref{eq:assume:concentrability}  is a normalization term, since $\sum_{m\geq 1} \gamma^{m-1} \cdot m = ( 1- \gamma)^{-2}$.
 \end{assumption}
By definition, concentration coefficients in \eqref{eq:concentration_coef} quantifies  the similarity between $\nu_2$ and  the distribution of the future states of the MDP when starting from $\nu_1$.  Moreover, \eqref{eq:assume:concentrability} is a standard assumption in the literature. See, e.g.,  \cite{munos2008finite, lazaric2010analysis, scherrer2015approximate, farahmand2010error, farahmand2016regularized}.  This assumption holds for   large class of systems MDPs and specifically for MDPs whose top-Lyapunov exponent   is finite. 
Moreover, this assumption essentially requires that the sampling distribution $\sigma$ has sufficient coverage over $\cS\times \cA$ and is shown to be necessary for the success of batch RL methods \citep{chen2019information}. 
See    \cite{munos2008finite,antos2007value,chen2019information} for more detailed discussions on this assumption.

Now we are ready to present the main theorem.
 %
%
  \begin{theorem}  \label{thm:main}
Under Assumptions  \ref{assume:closedness} and  \ref{assume:concentrability}, let   $\cF_0$  be  defined in \eqref{eq:define_cF} based on the family of sparse ReLU networks $\cF( L^*,\{ d_j^*\}_{j =0}^{L^*+1} , s^*)$ and  let $\cG_0 $ be given in \eqref{eq:define_cG}
with $\{ H_j\}_{j\in[q]}$ being absolute constants. 
 Moreover, for any $j \in [q-1]$, we define $\beta_j^* = \beta_j \cdot \prod_{\ell = j+1} \min ( \beta_{\ell}, 1)$; let $\beta_q^* = 1$.  In addition, let $
\alpha^* =  \max_{ j \in [q] } t_j / (2 \beta_j^* + t_j ) .
$
 For the parameters of $\cG_0$,  we assume that  the sample size $n$ is sufficiently large such that there exists a constant $\xi>0$ satisfying 
\#\label{eq:CG_param}
\max \bigg\{ \sum_{j=1}^q    (t_j + \beta_j + 1) ^{3 + t_j}, ~~\sum_{j \in [q] } \log ( t_j + \beta_j)    , ~~\max_{j \in [q] } p_j \bigg\} \lesssim  (\log n )^{\xi}.
\#
For the hyperparameters $L^*$, $ \{ d_j^*\}_{j =0}^{L^*+1} $, and $s^*$ of the ReLU network, we set    $d_0^* = 0$ and  $d_{L^*+1}^* = 1$.  Moreover, we set
\#\label{eq:dnn_hyperparam}
   L^* \lesssim (\log n)^{\xi^*} , ~~r \leq  \min _{j \in [L^*]} d_j^*  \leq \max_{j \in [L^*]} d_j^* \lesssim  n^{\xi^*} , ~~\text{and} ~~s^* \asymp n^{\alpha^*} \cdot  (\log n)^{\xi^*}
\#
for some constant $\xi^* > 1 + 2\xi$.  For any $K \in \NN$, let   $Q^{\pi_K}$ be the action-value function corresponding to policy $\pi_K$, which is returned by Algorithm \ref{algo:fit_Q} based on function class $\cF_0$. Then there exists a constant      $C  > 0 $  such that
      \#\label{eq:stat_rate_final}
   \|  Q^* - Q^{ \pi_{K}  }\|_{1, \mu}  \leq  C  \cdot \frac{ \phi_{\mu, \sigma} \cdot \gamma} {(1 -\gamma)^2 } \cdot  | \cA| \cdot  ( \log n )^{1 + 2\xi^* } \cdot n^{(\alpha^* - 1)/2 }+ \frac{4   \gamma^{K+1}   }{(1- \gamma) ^2 }  \cdot R_{\max}.
\#
\end{theorem}
This theorem implies that the statistical rate of convergence is the sum of a statistical error and an  algorithmic error. The algorithmic error  converges to zero in linear rate   as the algorithm proceeds, whereas the statistical error reflects the fundamental difficulty of the problem.
Thus, when the number of iterations satisfy
\$K \geq C'   \cdot   \bigl [ \log |A| + ( 1 -\alpha^*) \cdot \log n \bigr ] / \log (1/ \gamma) \$
iterations, where $C'$ is a sufficiently large constant, the algorithmic error is dominated by the statistical error. In this case, if we view both $\gamma$ and $\phi_{\mu, \sigma}$ as constants and ignore the polylogarithmic  term, Algorithm \ref{algo:fit_Q}  achieves error rate
\#\label{eq:rate}
   |\cA| \cdot    n^{ ( \alpha^* -1)/2  } = | \cA | \cdot \max_{j \in [q]} n^{-  \beta_j^* / ( 2\beta_j^* + t_j)} ,
  \#  which scales linearly with the capacity of the action space, and decays to zero when the  $n$ goes to infinity.  Furthermore, the rates $\{ n^{-  \beta_j^* / ( 2\beta_j^* + t_j)}\}_{j \in [q]} $ in \eqref{eq:rate} recovers the  statistical rate of nonparametric regression in $\ell_2$-norm,   whereas  our statistical rate   $n^{(\alpha^* -1)/2}$ in  \eqref{eq:rate} is the   fastest among these nonparametric rates, which illustrates the benefit of compositional structure of $\cG_0$.

Furthermore, as a concrete example, we assume that both the reward function and the Markov transition kernel are  H\"older smooth with smoothness parameter $\beta$. As stated below Assumption \ref{assume:closedness}, for any $f \in \cF_0$, we have $(Tf)(\cdot, a) \in \cC_r( \cS, \beta, H')$. Then Theorem \ref{thm:main} implies that  Algorithm~\ref{algo:fit_Q} achieves error rate $| \cA| \cdot n^{-   \beta / ( 2\beta +r )}$ when $K$ is sufficiently large. Since $|\cA|$ is finite, this rate achieves the minimax-optimal statistical rate of convergence within the class of H\"older smooth functions defined on $[0,1]^d$ \citep{stone1982optimal} and thus cannot be further improved.
As another example, when $(Tf)(\cdot, a) \in \cC_r( \cS, \beta, H')$ can be represented as an additive model over $[0,1]^r$ where each component has smoothness parameter $\beta$,  \eqref{eq:rate} reduces to    $| \cA| \cdot n^{-   \beta / ( 2\beta +1 )}$, which does not depends on the input dimension $r$ explicitly. Thus, by having a composite structure in $\cG_0$, Theorem \ref{thm:main} yields more refined statistical rates that adapt to the intrinsic difficulty of solving each iteration of Algorithm~\ref{algo:fit_Q}.

 In the sequel, we conclude this section by sketching   the proof of Theorem \ref{thm:main}; the detailed proof is deferred to \S\ref{proof:thm:main}.

\begin{proof}[Proof Sketch of Theorem  \ref{thm:main}]

Recall that $\pi_k$ is the greedy policy with respect to $\tilde Q_k$ and $Q^{\pi_K}$ is the
action-value function associated with $\pi_K$, whose definition is given in \eqref{eq:Q_fun}.
Since $\{ \tilde Q_k \}_{k \in [K]}$ is constructed by a iterative algorithm, it is crucial to relate $\|   Q^* - Q^{\pi_K}  \|_{1, \mu}$,   the quantity of interest, to the errors incurred in the previous steps, namely $\{ \tilde Q_k - T \tilde Q_{k-1}  \}_{k \in [K]}$. Thus, in the first step of the  proof, we establish  Theorem \ref{thm:err_prop}, also
known as the error propagation \citep{munos2008finite,lazaric2010analysis, scherrer2015approximate, farahmand2010error, farahmand2016regularized} in the   batch reinforcement learning literature, which provides an upper bound on
$\|   Q^* - Q^{\pi_K}  \|_{1, \mu}$ using $\{\|\tilde Q_k - T \tilde Q_{k-1}  \|_{\sigma} \}_{k \in [K]}$. In particular, Theorem \ref{thm:err_prop} asserts that
 \#\label{eq:first_step_proof001}
\|  Q^* - Q^{\pi_K }\|_{1, \mu}  \leq  &\frac{2\phi_{\mu, \sigma}\cdot \gamma} {(1 -\gamma)^2 } \cdot \max_{k \in [K]} \|\tilde Q_k - T \tilde Q_{k-1}  \|_{\sigma}
  + \frac{4   \gamma^{K+1}   }{(1- \gamma) ^2 }  \cdot R_{\max},
\#
where $\phi_{\mu, \sigma}$, given in \eqref{eq:assume:concentrability},  is a constant that only depends on distributions $\mu$ and $\sigma$.

%
The upper bound in \eqref{eq:first_step_proof001} shows that
the total error of Algorithm \ref{algo:fit_Q} can be viewed as a sum of a  statistical error and an algorithmic error, where
$ \max_{k \in [K]} \|\tilde Q_k - T \tilde Q_{k-1}  \|_{\sigma} $ is essentially the statistical error
and the second term on the right-hand side of \eqref{eq:first_step_proof001} corresponds to the algorithmic error. Here,  the statistical error diminishes as the sample size $n$ in each iteration grows, whereas the algorithmic error decays to zero geometrically as  the number of iterations $K$ increases.
 This implies that  the fundamental difficulty  of DQN is captured by the error incurred in a single  step.   Moreover, the  proof of this theorem depends on  bounding $\tilde Q_\ell - T \tilde Q_{\ell-1}$  using  $\tilde Q_{k}  - T \tilde Q_{k-1}$  for any $ k < \ell$, which characterizes how the one-step error  $\tilde Q_{k}  - T \tilde Q_{k-1}$ propagates as the algorithm proceeds.  See \S \ref{proof:thm:err_prop} for a detailed proof.

It remains to bound $ \|\tilde Q_k - T \tilde Q_{k-1}  \|_{\sigma} $
  for any $k \in [K]$. We achieve such a goal using
     tools from nonparametric regression. Specifically, as we will show in Theorem \ref{thm:each_term_error},
     under Assumption
   \ref{assume:closedness},
for any $k \in [K]$  we have   
  \#\label{eq:one_iter_bound00}
& \| \tilde Q_{k+1} -T\tilde  Q_k \|_{\sigma}^2  \leq   4 \cdot [  \text{dist}_{\infty}  (\cF_0, \cG_0) ]^2      + C \cdot  V_{\max }^2  / n \cdot \log    N_{\delta}
 + C  \cdot V_{\max} \cdot \delta
 \#
 for any $\delta > 0$,
where $C>0$ is an absolute constant, 
\#\label{eq:space_bias}
\text{dist} _{\infty} (\cF_0, \cG_0) =    \sup_{f' \in \cG_0   }   \inf_{f \in \cF_0 }    \| f -   f'   \|_{\infty}
\# is the $\ell_{\infty} $-error of approximating functions in $\cG_0$ using functions in $\cF_0$, and  $N_\delta$ is the minimum number of cardinality of the balls required to cover $\cF_0$ with respect to $\ell_{\infty}$-norm.

Furthermore,  \eqref{eq:one_iter_bound00}   characterizes  the bias and variance that arise in estimating the action-value functions using deep ReLU networks.
Specifically,
$
  [  \text{dist}_{\infty}  (\cF_0, \cG_0) ]^2
$
corresponds to the bias incurred by approximating functions in $\cG_0$ using ReLU networks. We note that such a bias is measured in the $\ell_{\infty}$-norm. In addition,
$
 V_{\max }^2  / n \cdot \log    N_{\delta}  + V_{\max} \cdot \delta
$
 controls the variance of the estimator.

In the sequel, we fix $\delta = 1/n$ in \eqref{eq:one_iter_bound00}, which implies that
\#\label{eq:second_step_proof00}
& \| \tilde Q_{k+1} -T\tilde  Q_k \|_{\sigma}^2   \leq    4\cdot  \text{dist}_{\infty}^2 (\cF_0, \cG_0) + C' \cdot V_{\max}^2  / n \cdot \log  N_{\delta},
\#
where $C'>0$ is an absolute constant.

In the subsequent proof, we establish upper bounds for $\text{dist}(\cF_0, \cG_0)$ defined in \eqref{eq:space_bias} and $\log N_{\delta}$, respectively. Recall that the family of composite H\"older smooth functions $\cG_0$ is defined in \eqref{eq:define_cG}.

%
%
 By the definition of $\cG_0$ in \eqref{eq:define_cG}, for any $f \in \cG_0$ and any $a \in \cA$, $f(\cdot, a) \in \cG(\{(p_j, t_j, \beta_j, H_j)\}_{j \in [q]})$ is a composition of H\"older smooth functions,
that is,
$
f(\cdot ,a) = g_q \circ \cdots \circ g_1.
$  Recall that, as defined in Definition \ref{def:comp:holder}, $g_{jk}$ is the $k$-th entry of the vector-valued function $g_j$. Here $g_{jk} \in \cC_{t_j} ( [ a_j, b_j]^{t_j} , \beta_j, H_j)$ for each $k\in [p_{j+1}]$ and $j \in [q]$.
To  construct a ReLU network that  is  $f(\cdot, a)$,   we
  first show that $f(\cdot, a)$ can be reformulated as a composition of  H\"older functions defined on a hypercube.
Specifically, let
$h_1 = g_1 / (2 H_1) + 1/2$,  $h_q (u) = g_q( 2 H_{q-1} u - H_{q-1} ) $, and
\$
h_j(u) = g_j ( 2H_{j-1} u - H_{j-1} ) / ( 2H_j) + 1/2 \$
for all  $ j \in \{ 2, \ldots, q-1\} $.  Then we immediately have
\#\label{eq:composition_form00}
f(\cdot, a)  = g_q \circ \cdots \circ g_1 = h_q \circ \cdots \circ h_1.
\#
Furthermore,  by the definition of H\"older smooth functions in Definition \ref{def:holder_func_class}, for any $ j\in [ q ] $ and  $k \in [p_{j+1}]$, it is not hard to verify  that     $
 h_{jk} \in \cC _{t_j} \bigl (  [0, 1]^{t_j}, W \bigr) ,
  $
  where  we define $W> 0$ by
  \#\label{eq:define_const_Q00}
W=\max  \Bigl \{ \max_{ 1 \leq j \leq q-1 } (2 H_{j-1} )^{\beta_j} , H_q ( 2 H_{q-1} )^{\beta_q} \Bigr \},
\#

   Now we employ Lemma \ref{eq:rate_approx_single_layer}, obtained    from  \cite{schmidt2017nonparametric},  to construct a ReLU network that approximates each $h_{jk}$, which,  combined with \eqref{eq:composition_form00}, yields a ReLU network that is close to  $f(\cdot ,a)$ in the $\ell_{\infty}$-norm.

To   apply Lemma \ref{eq:rate_approx_single_layer} we set $m = \eta \cdot \lceil \log_2 n \rceil$ for a sufficiently large constant $\eta > 1$, and set $N$ to be a sufficiently large integer that depends on $n$, which will be specified later. In addition, we set
$
L_j = 8  + (m+5) \cdot (1 + \lceil \log _2 (t_j + \beta_j )  \rceil ).
$
 Then, by Lemma \ref{eq:rate_approx_single_layer},  there exists  a ReLU network $\tilde  h_{jk}$ such that
$
\| \tilde h_{jk} - h_{jk}  \|_{\infty} \lesssim   
 N^{- \beta_j / t_j}.
$ 
 Furthermore, we have $\tilde h_{jk} \in \cF(L_j, \{t_j, \tilde d_j  , \ldots, \tilde d_j  , 1\} , \tilde s_j) , $ with 
 \#\label{eq:sj_bound}
\tilde d_j = 6 (t_j + \lceil \beta_j \rceil)   N, \qquad  \tilde  s_j \leq  141    \cdot (t_j + \beta_j + 1) ^{3 + t_j} \cdot N\cdot  ( m+6)  . 
\#
Now we define $  \tilde f \colon \cS \rightarrow \RR$ by $\tilde f = \tilde h_{q} \circ \cdots \circ \tilde h_1$ and set
 \#\label{eq:choice_N00}
  N = \bigl \lceil \max _{1\leq j \leq q} C \cdot n^{t_j / ( 2 \beta_j^* + t_j )}  \bigr\rceil .
  \#
  For this choice of $N$, we show that $\tilde f$ belongs to function class
  $ \cF(L^*, \{d_j^*\}_{j=1}^{L^*+1}, s^*  )$.
  Moreover, we define
   $\lambda_j = \prod_{\ell = j+1}^q  ( \beta_{\ell } \wedge 1) $  for any $j \in [q-1]$, and set $\lambda_q = 1$.
   Then we have $\beta_j \cdot \lambda_j = \beta_j^* $ for all $j \in [q]$.
   Furthermore, we show that $\tilde f $ is a good approximation of $f(\cdot, a)$. Specifically, we have
 \$
 \|   f (\cdot , a) - \tilde f \|_{\infty}  \lesssim  \sum_{j=1}^q  \| \tilde h_j - h_j \|_{\infty}^{\lambda_j}  .
 \$ Combining this with    \eqref{eq:choice_N00} and the fact that $\| \tilde h_{jk} - h_{jk}  \|_{\infty} \lesssim   
 N^{- \beta_j / t_j}$, we obtain that
 \#\label{eq:bias_trash}
 \bigl[\text{dist}(\cF_0, \cG_0)\bigr]^2 \lesssim n^{\alpha^* - 1}.
 \#

 Moreover, using classical results on the covering number of neural networks \citep{anthony2009neural}, we further show that
 \#\label{eq:vc_trash}
  \log  N_{\delta}  & \lesssim | \cA |  \cdot s^* \cdot L^* \max_{j \in [L^*]} \log (d_j^*) \lesssim  | \cA |  \cdot  n^{\alpha^*} \cdot (\log n)^ {1 + 2\xi^* }, 
  \#
 where $\delta  = 1/n$.
  Therefore,  combining \eqref{eq:one_iter_bound00},  \eqref{eq:second_step_proof00}, \eqref{eq:bias_trash}, and \eqref{eq:vc_trash},  we conclude the proof.
 \end{proof}


\section{Extension to Two-Player Zero-Sum Markov Games} \label{sec:zerosum}

In this section, we propose the Minimax-DQN algorithm, which combines  DQN and the Minimax-Q learning for two-player zero-sum Markov games.  We first present the background of zero-sum Markov games and introduce the the algorithm in \S\ref{sec:minimax_dqn}. Borrowing the analysis for DQN in the previous section, we provide theoretical guarantees for the proposed algorithm in \S\ref{sec:minimax_theory}.

\subsection{Minimax-DQN Algorithm}\label{sec:minimax_dqn}

As one of the simplistic  extension  of  MDP to the  multi-agent setting, two-player zero-sum Markov game  is denoted  by $(\cS, \cA, \cB,  P, R,  \gamma )$, where
  $\cS$ is  state space,  $\cA$ and $\cB$ are the action spaces of the first and second player, respectively.  In addition,  $P\colon \cS\times \cA \times \cB \rightarrow \cP(\cS)$ is the Markov transition kernel,  and $R   \colon \cS \times \cA \times \cB \rightarrow  \cP(\RR)$ is the distribution of    immediate reward  received by the first  player. At any time $t$,  the two players simultaneously take actions $A_t\in \cA$ and $B_t\in \cB$ at state $S_t \in \cS$, then the first player receives reward $R_t \sim R(S_t,A_t,B_t)$ and the second player obtains $ - R_t$.  The goal of each agent is to maximize its own cumulative discounted return.

  Furthermore, let  $\pi \colon  \cS\to\cP(\cA)$ and $\nu \colon  \cS\to\cP(\cB)$ be  policies  of the first and second players, respectively. Then,  we  similarly define the
  action-value function $Q^{\pi,\nu}:\cS\times\cA\times\cB\to\RR$ as
  \#\label{eq:new_Q_fun}
  Q^ {\pi,\nu }(s,a,b) =  \EE\bigg[\sum_{t=0}^\infty \gamma^t \cdot  R_t \bigggiven (S_0, A_0 , B_0 ) = (s, a, b) , A_t \sim \pi(\cdot~|S_t) ,  B_t \sim \nu(\cdot ~| S_t) \bigg],
  \#
  and define the state-value function $  V^{\pi, \nu} :\cS\rightarrow \RR$ as
  \#\label{eq:new_V_fun}
  V^ {\pi,\nu }(s) = \EE \bigl  [ Q ^{\pi, \nu} (s,A,B) \biggiven  A \sim \pi(\cdot \given s),  B \sim \nu(\cdot \given s) \bigr ].
  \#
  Note that these two value functions are defined by the rewards of the first player. Thus,  at any state-action tuple  $(s,a,b)$,   the  two players aim  to solve
  $ \max _{\pi } \min _{\nu  } Q^{\pi,\nu }(s, a,b) $   and $  \min _{\nu }\max _{\pi}Q ^{\pi,\nu }(s,a,b) ,
  $
  respectively. By the von Neumann's minimax theorem \citep{von1947theory,patek1997stochastic}, there exists a minimax function of the game, $Q^*\colon  \cS \times \cA \times \cB \rightarrow \RR$, such that
  \#\label{eq:new_Qs}
  Q ^*(s,a,b) = \max _{\pi } \min _{\nu } Q^{\pi,\nu }(s,a,b)   =    \min _{\nu }\max _{\pi }Q ^{\pi,\nu }(s, a,b).
  \#

 Moreover, for joint  policy $(\pi, \nu)$ of two players, we  define the Bellman  operators
  ${T}^{\pi,\nu}$ and $T$    by
  \#
  ({T}^{\pi,\nu} Q) (s, a,b) & = r(s, a,b) + \gamma \cdot (P^{\pi,\nu} Q) (s, a,b),\label{eq:bellman_oper21}\\
  ({T}Q) (s, a,b) & =     r (s, a,b) + \gamma \cdot(P^*Q) (s, a,b), \label{eq:bellman_oper22}
\#
where $r(s,a,b) = \int r   R(\ud r\given s,a,b)$, and we define operators
  $P^{\pi,\nu}$ and $ P^* $ by
 \$
 (P^{\pi,\nu} Q) (s, a,b) =&~  \EE_{s' \sim P(\cdot \given s,a,b), a' \sim \pi(\cdot \given s'),b' \sim \nu(\cdot \given s') } \bigl [Q(s', a',b') \bigr ],\notag\\
 (P^*Q) (s, a,b) =&~  \EE_{s' \sim P(\cdot \given s,a,b)} \Bigl\{\max_{\pi'\in\cP(\cA)}\min_{\nu'\in\cP(\cB)}\EE_{a' \sim \pi',b' \sim \nu' }\bigl [Q(s', a',b') \bigr ]\Bigl\}.
 \$
 Note that  $P^*$ is defined by  solving a zero-sum matrix game based on $Q(s', \cdot, \cdot) \in \RR^{|\cA | \times | \cB| }$, which could be achieved via linear programming. It can be shown that both $T^{\pi, \nu}$ and $T$ are $\gamma$-contractive, with $Q^{\pi, \nu}$  defined  in  \eqref{eq:new_Q_fun} and $Q^*$ defined in  \eqref{eq:new_Qs} being the unique fixed points,  respectively. Furthermore, similar to   \eqref{eq:greedy_policy}, in zero-sum Markov games,  for  any action-value function $Q$, the   equilibrium joint policy   with respect to $Q$ is defined as
 \#
 \bigl [ \pi_{Q}(\cdot\given s), \nu_Q (\cdot \given s) \bigr ] =\argmax_{\pi'\in\cP(\cA)}\argmin_{\nu'\in\cP(\cB)} \EE_{a\sim \pi', b\sim \nu'} \big[ Q (s, a,b)\big], \qquad \forall s\in \cS. \label{eq:equi_policy}
 \#
 That is, $ \pi_{Q}(\cdot\given s) $ and $\nu_Q (\cdot \given s) $ solves the zero-sum matrix game based on $Q(s, \cdot ,\cdot )$ for all $s\in \cS$.
By this definition, we obtain that the equilibrium joint policy with respect to
 the minimax function $Q^*$ defined in \eqref{eq:new_Qs}  achieves the Nash equilibrium of the Markov game.

  Therefore, to learn the Nash equilibrium, it suffices to estimate $Q^*$, which is the unique  fixed point of the Bellman operator $T$. Similar to the standard Q-learning for MDP,  \cite{littman1994markov} proposes the Minimax-Q learning algorithm, which constructs a sequence of action-value functions that converges to $Q^*$. Specifically, in each iteration, based on  a transition $(s,a, b,s')$,   Minimax-Q learning updates the current estimator of $Q^*$, denoted by   $Q $, via
  \$
  Q(s, a, b) \leftarrow (1-\alpha ) \cdot  Q(s,  a, b)  + \alpha \cdot \Big \{   r(s,a,b) + \gamma \cdot \max_{\pi'\in\cP(\cA)}\min_{\nu'\in\cP(\cB)}\EE_{a' \sim \pi',b' \sim \nu' }\bigl [Q(s', a',b') \bigr ] \Bigr \},
  \$
   where $\alpha \in (0,1)$ is the stepsize.

 Motivated by this algorithm, we propose the  Minimax-DQN algorithm which  extend DQN to two-player zero-sum Markov games. Specifically, we parametrize the action-value function using a deep neural network $Q_{\theta} \colon \cS \times \cA \times \cB \rightarrow \RR$ and store the transition $(S_t, A_t, B_t, R_t, S_{t+1})$ into the replay memory $\cM$ at each time-step.
Parameter $\theta$ of the Q-network is updated as follows.
 Let $Q_{\theta^*}$ be the target network.
 With $n$ independent samples $\{(s_i,a_i,b_i, r_i,s'_i)\}_{i \in [n]  } $ from $\cM$,  for all $i \in [n]$, we compute the target
 \#\label{eq:minmax_tgt}
Y_i  = r_i + \gamma \cdot \max_{\pi'\in\cP(\cA)}   \min_{\nu'\in\cP(\cB)}\EE_{a  \sim \pi',b  \sim \nu' }   \bigl [Q_{\theta^*}(s_{i}' ,  a ,b ) \bigr ],
\#
which can be attained via linear programming.
 Then we update $\theta$ in the direction of $\nabla_{\theta} L(\theta)$, where  $L(\theta) = n^{-1} \sum_{i\in [n]} [ Y_i - Q_{\theta} (s_i, a_i, b_i) ]^2$. Finally, the target network $Q_{\theta^*}$ is updated every $T_{\textrm{target}}$ steps by letting $\theta^* = \theta$. For brevity,  we defer the details of Minimax-DQN to Algorithm \ref{algo:m_dqn} in  \S \ref{sec:state_dqn}.

 To understand the theoretical aspects of this algorithm, we similarly utilize the framework  of batch reinforcement learning for statistical analysis. With the insights gained in \S \ref{sec:algo}, we consider a  modification of Minimax-DQN  based on neural fitted Q-iteration, whose details are stated in Algorithm \ref{algo:fit_Q2}. As in the MDP setting,  we replace sampling from the replay memory by sampling i.i.d. state-action tuples from a fixed distribution $\sigma \in \cP(\cS\times \cA \times \cB)$, and
estimate $Q^*$ in \eqref{eq:new_Qs}  by solving a sequence of  least-squares regression problems specified by \eqref{eq:least_squares_game}. Intuitively, this algorithm approximates the value iteration algorithm for zero-sum Markov games \citep{littman1994markov} by constructing a sequence of value functions $\{ \tilde Q_{k}  \}_{k\geq 0}$ such that $\tilde Q_{k+1} \approx T \tilde Q_k$ for all $k$, where $T$ defined in \eqref{eq:bellman_oper22} is the Bellman operator. 

 \begin{algorithm} 
 	\caption{Fitted Q-Iteration Algorithm for Zero-Sum Markov Games (Minimax-FQI)} 
 	\label{algo:fit_Q2} 
 	\begin{algorithmic} 
 		\STATE{{\textbf{Input:}} Two-player zero-sum Markov game $(\cS, \cA, \cB, P, R, \gamma)$, function class $\cF$,  distribution $\sigma\in \cP(\cS\times \cA \times \cB)$, number of  iterations $K$, number of samples $n$, the initial estimator $\tilde Q_0\in \cF$.}
 		\FOR{$k = 0, 1, 2, \ldots, K-1$}
 		\STATE{Sample $n$  i.i.d. observations $\{(S_i, A_i, B_i )\}_{i \in [n] } $ from $\sigma$,  obtain $R_i \sim R(\cdot ~| S_i, A_i, B_i)$ and $S_i' \sim P(\cdot ~| S_i, A_i, B_i)$.}
 		\STATE{Compute $Y_i = R_i +   \gamma \cdot \max_{\pi'\in\cP(\cA)}   \min_{\nu'\in\cP(\cB)}\EE_{a  \sim \pi',b  \sim \nu' }   \bigl [\tilde Q_k (s_{i}' ,  a ,b ) \bigr ].   $}
 		\STATE{Update the action-value function: \# \label{eq:least_squares_game} 
		\tilde Q_{k+1} \leftarrow  \argmin_{f \in \cF} \frac{1}{n}\sum_{i=1}^n \bigl [ Y_i - f(S_i, A_i, B_i) \bigr ]^2.\#}
 		\ENDFOR
 		\STATE{Let $(\pi_K, \nu_K)$ be  the equilibrium joint policy with respect to $\tilde Q_{K}$, which is  defined in \eqref{eq:equi_policy}.}
 		\STATE{{\textbf{Output:}} An estimator $\tilde Q_{K }$ of $Q^*$ and joint policy $(\pi_K, \nu_K)$.}
 	\end{algorithmic}
 \end{algorithm}

 \subsection{Theoretical Results for Minimax-FQI} \label{sec:minimax_theory}

Following the theoretical results established in \S\ref{sec:theory},  in this subsection, we provide statistical guarantees for the Minimax-FQI algorithm with $\cF$ being a family of deep neural networks with ReLU activation.
 Hereafter, without loss of generality, we assume $\cS = [0,1]^{r}$  with $r$ being a fixed integer, and the action spaces $\cA$ and $\cB$ are both finite.
 To evaluate the performance of the algorithm, we first introduce the best-response policy as follows.

 \begin{definition} \label{def:bestresponse} For any policy $\pi \colon \cS \rightarrow \cP(\cA)$ of player one, the best-response policy against $\pi$, denoted by $\nu^*_\pi$, is defined as the optimal policy of second player when the first player follows $\pi$. In other words, for all $s\in \cS$, we have
 	$
 	\nu^*_\pi ( \cdot \given s) = \argmin _{\nu} V^{\pi, \nu} (s) ,
 	$
 	where $V^{\pi, \nu}$ is defined in \eqref{eq:new_V_fun}.
 \end{definition}

Note that when the first player adopt a fixed policy $\pi$, from the perspective of the second player, the Markov game becomes a MDP. Thus, $\nu_\pi^*$ is the optimal policy of the MDP induced by $\pi$. Moreover, it can be shown that, for any policy $\pi$,
$
Q^*(s,a,b) \geq Q^{\pi, \nu^*_\pi} (s,a,b)
$
holds for every state-action tuple $(s,a,b)$.
Thus, by considering the adversarial case where
the opponent always plays the best-response policy,  the difference between $Q^{\pi. \nu_\pi^*}$ and $Q^*$ servers as a  characterization of the suboptimality of $\pi$.
Hence, to quantify the performance of Algorithm \ref{algo:fit_Q2}, we consider the closeness between $Q^*$ and $Q^{\pi_K, \nu_{\pi_K}^*}$, which will be denoted by $Q_K^*$ hereafter for simplicity.
Specifically, in the following we establish an upper bound for $\|   Q^* - Q_K^* \|_{1, \mu} $ for some distribution $  \mu \in \cP(\cS\times \cA\times \cB)$.

We first specify the function class $\cF$ in Algorithm \ref{algo:fit_Q2}  as follows.

  \begin{assumption}[Function Classes] \label{assume:closedness2}
  	Following Definition \ref{def:func_class}, let $\cF(L, \{d_j \}_{j=0}^{L+1} , s  ) $ and $\cG( \{ p_j, t_j , \beta_j ,  H_j \}_{j \in[q]} )  $ be the family of sparse  ReLU networks and the set of composition of  H\"older smooth functions defined on $\cS$, respectively. Similar to \eqref{eq:define_cF}, we define $\cF_1$ by
  	\#
  	\cF_1 & = \bigl \{ f \colon \cS \times \cA \times \cB \rightarrow \RR   \colon   f(\cdot, a, b) \in \cF(L, \{ d_j\}_{i=0}^{L+1} , s ) ~\text{for any}~(a,b) \in \cA\times \cB \bigr \}.\label{eq:define_cF2}
  	\#
  	 For the Bellman operator $T$ defined in \eqref{eq:bellman_oper22}, we
    assume that for any $f \in \cF_1$ and any state-action tuple $(s,a,b)$, we have     $(Tf)( \cdot ,a, b) \in   \cG( \{ p_j, t_j , \beta_j ,  H_j \}_{j \in[q]} )$.
  \end{assumption}
  We remark that this Assumption is in the same flavor as Assumption \ref{assume:closedness}. As discussed in \S\ref{sec:theory},  this assumption holds if both the reward function and the transition density of the Markov game are sufficiently smooth.

   In the following, we define the concentration coefficients for Markov games.



  \begin{assumption}[Concentration Coefficient for   Zero-Sum Markov Games]\label{assume:concentrability2}
  	Let   $\{ \tau_t \colon \cS \rightarrow \cP(\cA\times \cB)  \} $ be a sequence of joint policies for the two players in the zero-sum Markov game.
  	Let $\nu_1, \nu_2 \in \cP(\cS\times\cA\times\cB)$ be two absolutely continuous probability measures.
  	Suppose  the initial state-action pair  $(S_0, A_0,B_0)$ has distribution $ \nu_1$, the future states are sampled according to the Markov transition kernel, and
  	the   action $(A_t, B_t)$ is   sampled from   policy $\tau_t$.    For any integer $m$, we denote by $ P^{\tau_m} P^{\tau_{m-1}}  \cdots P^{\tau_1} \nu_1$ the   distribution of $\{(S_t, A_t,B_t)\}_{t=0}^m$.
  	Then, the $m$-th  concentration coefficient is defined as
  	\#\label{eq:concentration2}
  	\kappa (m; \nu_1, \nu_2) =  \sup_{\tau_{1}, \ldots, \tau_{m} } \biggl [\EE _{\nu_2} \biggl | \frac{    \ud (  P^{\tau_m} P^{\tau_{m-1}}  \cdots P^{\tau_1} \nu_1 )   } { \ud \nu_2} \biggr | ^2     \biggr ] ^{1/2},
  	\#
  	where  the supremum is taken over all possible joint policy sequences $\{ \tau_t  \}_{t\in[m] } $.
  	
  	Furthermore,  for  some  $\mu \in \cP(\cS\times \cA \times \cB)$, we    assume that there exists a finite constant $\phi_{\mu, \sigma} $ such that
  	$
  	( 1- \gamma)^{2}\cdot \sum_{m \geq 1} \gamma^{m-1} \cdot m \cdot \kappa ( m; \mu, \sigma) \leq \phi_{\mu, \sigma},
  	$
  	where    $\sigma$ is the sampling distribution in Algorithm~\ref{algo:fit_Q2} and $ \kappa ( m; \mu, \sigma) $ is the $m$-th concentration coefficient defined in \eqref{eq:concentration2}.
  \end{assumption}

   We remark that the definition of the $m$-th concentration coefficient is the same as in \eqref{eq:concentration_coef}  if we replace   the action space $\cA$ of  the MDP by  $\cA \times \cB$ of the Markov game. Thus, Assumptions \ref{assume:concentrability} and~\ref{assume:concentrability2} are of the same nature, which are standard in the literature.

  Now we are ready to present the main theorem.

  \begin{theorem}   \label{thm:game}
  	Under Assumptions  \ref{assume:closedness2} and  \ref{assume:concentrability2}, consider the Minimax-FQI algorithm with the function class $\cF$ being  $\cF_1$    defined in \eqref{eq:define_cF2} based on the family of sparse ReLU networks $\cF( L^*,\{ d_j^*\}_{j =0}^{L^*+1} , s^*)$.  We make the same assumptions on $\cF( L^*,\{ d_j^*\}_{j =0}^{L^*+1} , s^*)$ and   $ \cG( \{ p_j, t_j , \beta_j ,  H_j \}_{j \in[q]} )$ as in \eqref{eq:CG_param} and
  	\eqref{eq:dnn_hyperparam}.
  	Then for
   any $K \in \NN$, let $(\pi_K, \nu_K)$ be the policy returned by  the algorithm and let $Q_K^*$ be the action-value function corresponding to $(\pi_K, \nu_{\pi_K}^*)$.
    Then there exists a constant   $C > 0  $  such that
  	\#\label{eq:stat_rate_final2}
  	\|  Q^* - Q_K^*  \|_{1, \mu}  \leq  C  \cdot \frac{ \phi_{\mu, \sigma} \cdot \gamma} {(1 -\gamma)^2 } \cdot  | \cA| \cdot | \cB| \cdot  ( \log n )^{1 + 2 \xi^* } \cdot n^{(\alpha^* - 1)/2}+ \frac{4   \gamma^{K+1}   }{(1- \gamma) ^2 }  \cdot R_{\max},
  	\#
  	where $\xi^* $ appears in \eqref{eq:dnn_hyperparam},  $
  	\alpha^* =  \max_{ j \in [q] } t_j / (2 \beta_j^* + t_j )
  	$ and $\phi_{\mu, \sigma}$ is specified in Assumption  \ref{assume:concentrability2}.
  \end{theorem}

Similar to Theorem \ref{thm:main}, the bound
in \eqref{eq:stat_rate_final2} shows that closeness between $(\pi_K, \nu_K)$ returned by Algorithm \ref{algo:fit_Q2} and the Nash equilibrium policy $(\pi_{Q^*}, \nu_{Q^*} )$, measured by 	$ \|  Q^* - Q_K^*  \|_{1, \mu}$,
is bounded by the sum of
 statistical error and an  algorithmic error.
 Specifically, the statistical error balances the bias and variance of estimating the  value functions using the family of deep ReLU neural networks, which exhibits the
  fundamental difficulty of the problem.
  Whereas the algorithmic error decay to zero geometrically as $K$ increases. Thus, when $K$ is sufficiently large, both $\gamma $ and $\phi_{\mu, \sigma}$ are constants, and the polylogarithmic term is ignored,
  Algorithm \ref{algo:fit_Q2}  achieves error rate
  \#\label{eq:rate2}
  |\cA| \cdot |\cB | \cdot    n^{  \alpha^* -1  } = | \cA | \cdot |\cB | \cdot   \max_{j \in [q]} n^{-  \beta_j^* / ( 2\beta_j^* + t_j)} ,
  \#  which scales linearly with the capacity of joint action space.
  Besides,
  if $| \cB| = 1$,
  the minimax-FQI algorithm reduces to Algorithm \ref{algo:fit_Q}. In this case, \eqref{eq:rate2} also recovers the error rate of Algorithm~\ref{algo:fit_Q}. Furthermore, the statistical rate $n^{(\alpha^* - 1)/2}$ achieves the optimal $\ell_2$-norm error of regression for nonparametric regression with a compositional structure, which
indicates that the statistical error in \eqref{eq:stat_rate_final2}  can not be further improved.

   \begin{proof}
   	See \S\ref{proof:thm:game} for a detailed proof.
   	\end{proof}


\section{Proof of the Main Theorem} \label{proof:thm:main}
 In this section, we present a detailed proof  of Theorem \ref{thm:main}.
\begin{proof}
The proof requires two key ingredients. First in Theorem \ref{thm:err_prop} we quantify how the error of  action-value function approximation propagates through each iteration of Algorithm \ref{algo:fit_Q}. Then in Theorem \ref{thm:each_term_error} we analyze such one-step approximation error for ReLU networks.

 \begin{theorem}[Error Propagation] \label{thm:err_prop}
 Recall that $\{\tilde Q_k\}_{ 0\leq k \leq K} $ are the iterates of Algorithm \ref{algo:fit_Q}. Let $\pi_K$ be the one-step greedy policy with respect to $\tilde Q_K$, and let $Q^{\pi_K}$ be the action-value function corresponding to $\pi_K$.
 Under Assumption \ref{assume:concentrability}, we have
 \#\label{eq:err_prop_final}
   \|  Q^* - Q^{\pi_K }\|_{1, \mu}  \leq  \frac{2\phi_{\mu, \sigma}\gamma} {(1 -\gamma)^2 } \cdot \varepsilon_{\max} + \frac{4   \gamma^{K+1}   }{(1- \gamma) ^2 }  \cdot R_{\max},
   \#
   where we define the maximum one-step approximation error  as $\varepsilon_{\max} = \max_{ k \in [K] } \| T \tilde Q_{k-1} - \tilde Q_{k} \|_{\sigma}$. Here $\phi_{\mu, \sigma}$ is a constant that only depends on the probability distributions $\mu$ and $\sigma$.
 \end{theorem}
\begin{proof}
See \S \ref{proof:thm:err_prop} for a detailed proof.
\end{proof}

We remark that similar error propagation result is established for the state-value function    in
\cite{munos2008finite} for studying the fitted value iteration algorithm, which is further extended by
 \cite{lazaric2010analysis, scherrer2015approximate, farahmand2010error, farahmand2016regularized} for other batch reinforcement learning methods.

In the sequel, we establish an upper bound for the one-step approximation error $\| T \tilde Q_{k-1} - \tilde Q_{k} \|_{\sigma}$ for each $k \in [K]$.

\begin{theorem}  [One-step Approximation Error] \label{thm:each_term_error}
 Let  $\cF\subseteq \cB( \cS \times \cA, V_{\max} ) $ be a class of measurable functions on $\cS \times \cA$ that are bounded by  $ V_{\max} =  R_{\max} / (1- \gamma)$, and let $\sigma$ be a probability distribution on $\cS\times \cA$.
 Also, let   $\{(S_i, A_i)\}_{i\in [n]}$ be $n$ i.i.d. random variables in $\cS \times \cA$ following $\sigma$. For each $i\in [n]$,   let $R_i$ and $S_i'
 $ be the reward and the next state corresponding to $(S_i, A_i)$.   In addition,  for any fixed $Q\in \cF$, we define $Y_i = R_i + \gamma \cdot \max_{a \in \cA} Q(S_i', a)$.
Based on $\{( X_i, A_i, Y_i)\}_{i\in [n] } $, we define $\hat Q$~as the solution to the least-squares problem
\#\label{eq:define_fit_Q}
 \minimize_{f\in \cF } \frac{1}{n} \sum_{i=1}^n \bigl[ f(S_i, A_i) - Y_i \bigr]^2.
\#
Meanwhile, for any $\delta > 0$, let $\cN(\delta, \cF, \| \cdot \|_{\infty})$ be the minimal $\delta$-covering set of $\cF$ with respect to $\ell_{\infty}$-norm, and we denote by $N_\delta$ its cardinality.
Then for any $\epsilon \in (0, 1]$ and any $\delta > 0$,  we have  \#\label{eq:one_iter_bound}
\| \hat Q - T Q \|_{\sigma}^2 \leq  ( 1 + \epsilon)^2 \cdot \omega(\cF) + C \cdot  V_{\max }^2  / ( n \cdot \epsilon) \cdot \log    N_{\delta}  + C' \cdot V_{\max} \cdot \delta,
\#
where $C$ and $C'$ are two positive absolute constants and $\omega(\cF)$ is defined as
\#\label{eq:bellman_bias}
\omega(\cF) = \sup_{g \in \cF    }   \inf_{f \in \cF } \| f - T g   \|_{\sigma}^2 .
\#
\end{theorem}

\begin{proof}
See \S \ref{proof:each_term_error} for a detailed proof.
\end{proof}

This theorem characterizes  the bias and variance that arise in estimating the action-value functions using deep ReLU networks.
Specifically,
$
  \omega(\cF)
$ in \eqref{eq:bellman_bias} 
corresponds to the bias incurred by approximating  the target function $Tf$ using ReLU neural networks.
It can be viewed as a measure of completeness of $\cF$ with respect to the Bellman operator $T$. 
 In addition,
$
 V_{\max }^2  / n \cdot \log    N_{\delta}  + V_{\max} \cdot \delta
$
 controls the variance of the estimator, where the covering number $N_{\delta}$ is used to obtain  a uniform bound over $\cF_0$.

To obtain an upper bound for  $\| T \tilde Q_{k-1} - \tilde Q_{k} \|_{\sigma}$ as required in Theorem \ref{thm:err_prop},
we set $Q = \tilde Q_{k-1} $ in  Theorem \ref{thm:each_term_error}. Then according to Algorithm 1, $\hat Q $ defined in \eqref{eq:define_fit_Q} becomes $\tilde Q_{k}$. We set the function class $\cF $ in Theorem \ref{thm:each_term_error} to be the family of ReLU Q-networks $\cF_0$ defined in \eqref{eq:define_cF}.  By setting
$\epsilon = 1$ and  $\delta = 1/n$ in Theorem \ref{thm:each_term_error}, we obtain
\#\label{eq:apply_thm_each_term}
\| \tilde Q_{k+1} -T\tilde  Q_k \|_{\sigma}^2 \leq 4\cdot  \omega (\cF_0)  + C \cdot V_{\max}^2  / n \cdot \log N_0,
\#
where $C$ is a positive absolute constant and
\#\label{eq:n0}
N_0 = \bigl| \cN (1/n, \cF_0,  \| \cdot \|_{\infty} ) \bigr|
\#
 is the $1/n$-covering number of $\cF_0$.
In the subsequent proof, we establish upper bounds for $\omega (\cF_0)$ defined in \eqref{eq:bellman_bias} and $\log N_0$, respectively. Recall that the family of composite H\"older smooth functions $\cG_0$ is defined in \eqref{eq:define_cG}. By Assumption \ref{assume:closedness}, we have $Tg \in \cG_0$ for any $g \in \cF_0$. Hence, we~have
\#\label{eq:bound_approx_F0}
 \omega ( \cF_0 ) = \sup_{f' \in \cG_0   }   \inf_{f \in \cF_0 } \| f -   f'   \|_{\sigma}^2  \leq \sup_{f' \in \cG_0   }   \inf_{f \in \cF_0 }    \| f -   f'   \|_{\infty}^2,
\#
where the right-hand side is the $\ell_{\infty}$-error of approximating the functions in $\cG_0$ using the family of ReLU networks   $\cF_0$. 

%
%
 By the definition of $\cG_0$ in \eqref{eq:define_cG}, for any $f \in \cG_0$ and any $a \in \cA$, $f(\cdot, a) \in \cG(\{(p_j, t_j, \beta_j, H_j)\}_{j \in [q]})$ is a composition of H\"older smooth functions,
that is,
$
f(\cdot ,a) = g_q \circ \cdots \circ g_1
$. Recall that, as defined in Definition \ref{def:comp:holder}, $g_{jk}$ is the $k$-th entry of the vector-valued function $g_j$. Here $g_{jk} \in \cC_{t_j} ( [ a_j, b_j]^{t_j} , \beta_j, H_j)$ for each $k\in [p_{j+1}]$ and $j \in [q]$.
In the sequel, we construct a ReLU network to approximate $f(\cdot, a)$ and establish an upper bound of the approximation error on the right-hand side of \eqref{eq:bound_approx_F0}.
We first show that $f(\cdot, a)$ can be reformulated as a composition of  H\"older functions defined on a hypercube.
We define
$h_1 = g_1 / (2 H_1) + 1/2$,
\$
h_j(u) = g_j ( 2H_{j-1} u - H_{j-1} ) / ( 2H_j) + 1/2,~~ \text{for all}~~ j \in \{ 2, \ldots, q-1\} , \$ and $h_q (u) = g_q( 2 H_{q-1} u - H_{q-1} ) $. Then we immediately have
\#\label{eq:composition_form}
f(\cdot, a)  = g_q \circ \cdots \circ g_1 = h_q \circ \cdots \circ h_1.
\#
Furthermore,  by the definition of H\"older smooth functions in Definition \ref{def:holder_func_class}, for any $ k\in [ p_2]$, we have that $h_{1k}$ takes value in $[0,1]$ and $h_{1k }  \in \cC_{t_1}( [0,1]^{t_1} , \beta_1, 1)$. Similarly, for any $j  \in \{ 2, \ldots, q-1\}$ and $k \in [p_{j+1}]$,  $ h_{jk}$ also takes value in $[0,1]$ and
 \#\label{eq:index_holder}
 h_{jk} \in \cC _{t_j} \bigl (  [0, 1]^{t_j}, \beta_j, (2 H_{j-1} )^{\beta_j} \bigr ) .
 \#
  Finally, recall that we use the convention that $p_{q+1} = 1$, that is, $h_q $ is a scalar-valued function that satisfies
  \$
  h_{q } \in \cC_{t_q} \bigl( [0,1]^{t_q} , \beta_q, H_q (2 H_{q-1} )^{\beta_q} \bigr).
  \$
  
  In the following, we show that the composition function in \eqref{eq:composition_form} can be approximated by an element in  $\cF(L^*, \{d_j^*\}_{j=1}^{L+1}, s^*  )$ when the network hyperparameters are properly chosen. 
  Our proof consists of three steps. In the first step, we construct a ReLU network $\tilde h_{jk}$ that approximates each $h_{jk}$ in \eqref{eq:index_holder}. Then, in the second step,  we approximate $f(\cdot, a)$ by the composition of $\{ \tilde h_j \}_{j\in [q]}$ and quantify the architecture of this network.  Finally, in the last step, we prove that this network can be embedded into class $\cF(L^*, \{d_j^*\}_{j=1}^{L+1}, s^*  )$ and  characterize the final approximation error. 
  
  \vspace{5pt} 
  {\noindent \bf Step (i).} Now we employ the following lemma, obtained from  \cite{schmidt2017nonparametric},  to construct a ReLU network that approximates each $h_{jk}$, which combined with \eqref{eq:composition_form} yields a ReLU network that is close to  $f(\cdot ,a)$. Recall that, as defined in Definition \ref{def:holder_func_class}, we denote by $\cC_r ( \cD , \beta, H)$ the family of H\"older smooth functions with parameters $\beta$ and $H$ on $\cD \subseteq \RR^r$.

\begin{lemma}[Theorem 5 in \cite{schmidt2017nonparametric}]  \label{eq:rate_approx_single_layer}
 For any integers $m \geq 1$ and $N \geq \max \{  (\beta+1)^r, (H+1) e^r \} $, let $L = 8 + (m+5) \cdot ( 1 + \lceil \log_2 ( r + \beta) \rceil )$,
	$d_0 = r$, $d_j = 6 (r + \lceil \beta  \rceil ) N $ for each $j\in [L]$, and $d_{L+1} = 1$. For any $g \in \cC_r( [0,1]^r, \beta, H)$,  there exists a ReLU network $ f \in \cF(L, \{d_j\}_{j=0}^{L+1}, s, V_{\max})$ as defined in Definition \ref{def:relu} such that
	\$
	\| f - g \|_{\infty} \leq (2 H+1 ) \cdot 6 ^{r} \cdot N \cdot  ( 1+ r ^2 + \beta ^2 ) \cdot  2^{-m}  + H \cdot 3^{\beta } \cdot N^{-\beta/r},
	\$
	where the parameter $s$ satisfies $ s \leq 141 \cdot (r+ \beta + 1)^{3+r}  \cdot  N \cdot (m+6 )$.   \end{lemma}
\begin{proof}
See Appendix B in  \cite{schmidt2017nonparametric} for a detailed proof. The idea is to first approximate the  H\"older smooth function by polynomials via local Taylor expansion. Then, neural networks are constructed explicitly to approximate
each monomial terms in these local polynomials. 
\end{proof}

  We  apply Lemma \ref{eq:rate_approx_single_layer} to  $h_{jk} \colon [0,1]^{t_j} \rightarrow [0,1]$ for any $j \in [q]$ and $k \in [p_{j+1}]$.  We set $m = \eta \cdot \lceil \log_2 n \rceil$ for a sufficiently large constant $\eta > 1$, and set $N$ to be a sufficiently large integer depending on $n$, which will be specified later.
  In addition, we set
\#\label{eq:lj}
L_j = 8  + (m+5) \cdot (1 + \lceil \log _2 (t_j + \beta_j )  \rceil )
\#
 and define
\#\label{eq:define_const_Q}
W=\max  \Bigl \{ \max_{ 1 \leq j \leq q-1 } (2 H_{j-1} )^{\beta_j} , H_q ( 2 H_{q-1} )^{\beta_q} , 1 \Bigr \}.
\# 
We will later verify that $N \geq \max  \{  (\beta+1)^{t_j} , (W+1) e^{t_j}  \} $ for all $j \in [q]$.  Then by Lemma \ref{eq:rate_approx_single_layer},  there exists  a ReLU network $\hat  h_{jk}$ such that
\#\label{eq:some_bound}
\| \hat h_{jk} - h_{jk}  \|_{\infty} \leq (2 W + 1) \cdot 6^{t_j}  \cdot N \cdot  2^{-m} + W \cdot 3^{\beta_j} \cdot
 N^{- \beta_j / t_j}.
 \#
 Furthermore, we have $\hat h_{jk} \in \cF(L_j, \{t_j, \tilde d_j  , \ldots, \tilde d_j  , 1\} , \tilde s_j) $ with
 \#\label{eq:sj_bound}
\tilde d_j = 6 (t_j + \lceil \beta_j \rceil)   N, \qquad  \tilde  s_j \leq  141    \cdot (t_j + \beta_j + 1) ^{3 + t_j} \cdot N\cdot  ( m+6)  .
 \#
 Meanwhile, since $h_{j+1} = ( h_{(j+1) k} )_{k \in [p_{j+2}]}$ takes input from $[0,1]^{t_{j+1}}$, we need to further transform $\hat h_{jk}$ so that it takes value in $[0,1]$. In particular, we define $\sigma (u) = 1 - (1- u)_{+} = \min\{\max\{u, 0\}, 1\}$ for any $u \in \RR$. Note that $\sigma$ can be represented by a two-layer ReLU network with four nonzero weights. Then we define $\tilde h_{jk} = \sigma\circ  \hat h _{jk} $ and $\tilde h_j= ( \tilde h_{jk} )_{k \in [p_{j+1}]}$.  Note that by the definition of $\tilde h_{jk}$, we have $\tilde h_{jk} \in \cF(L_j +2, \{t_j, \tilde d_j , \ldots, \tilde d_j, 1\} , \tilde s_j+4), $ which yields
 \#\label{eq:tilde_hj_network}
 \tilde h_j \in \cF \bigl( L_j +2 , \{t_j, \tilde d_j\cdot p_{j+1}, \ldots, \tilde d_j \cdot p_{j+1}, p_{j+1}\} ,  (\tilde s_j+4) \cdot p_{j+1} \bigr).
 \#
Moreover, since both  $\tilde h_{jk} $ and $h_{jk}$ take value in $[0,1]$, by \eqref{eq:some_bound} we have
 \#\label{eq:one_term_approx_err}
 \| \tilde h_{jk} - h_{jk} \|_{\infty} &= \| \sigma \circ  \hat  h_{jk} - \sigma \circ h_{jk} \|_{\infty} \leq \| \hat h_{jk} - h_{jk} \|_{\infty}\notag\\
 & \leq (2 W + 1) \cdot 6^{t_j} \cdot  N \cdot  n^{-\eta}   + W \cdot 3^{\beta_j} \cdot N^{ - \beta_j / t_j},
 \#
 where the constant $W$ is defined in \eqref{eq:define_const_Q}. Since we can set the constant $\eta$ in \eqref{eq:one_term_approx_err}  to be sufficiently large,  the second term on the right-hand side of \eqref{eq:one_term_approx_err} is the leading term asymptotically, that is,
 \#\label{eq:trash}
  \| \tilde h_{jk} - h_{jk} \|_{\infty}  \lesssim       N^{ - \beta_j / t_j } .
 \#
Thus, in the first step, we have shown that there exists  $\tilde h_{jk} \in \cF(L_j +2, \{t_j, \tilde d_j , \ldots, \tilde d_j, 1\} , \tilde s_j+4)  $ satisfying \eqref{eq:trash}.

\vspace{5pt}
{\noindent \bf Step (ii).} In the second step, we stack $\tilde h_j$ defined in \eqref{eq:tilde_hj_network} to   approximate $f(\cdot, a)$ in \eqref{eq:composition_form}.
 Specifically, 
  we define $  \tilde f \colon \cS \rightarrow \RR$ as $\tilde f = \tilde h_{q} \circ \cdots \circ \tilde h_1$, which falls in the function class
\#\label{eq:inter_composite_class}
\cF   (\tilde L, \{r, \tilde d,\ldots, \tilde d, 1\}, \tilde s  ),
\#
where we define $\tilde L = \sum_{j=1}^ q (L_j +2) $, $\tilde d = \max _{j \in [q]} \tilde d_j \cdot  p_{j+1}   $, and $\tilde s = \sum_{j=1}^q (\tilde s_j +4) \cdot p_{j+1}$. Recall that $L_j$ is defined in \eqref{eq:lj}. Then when $n $ is sufficiently large, we have
\#\label{eq:tilde_L_bound}
\tilde L  & \leq \sum_{j=1} ^q \bigl \{  8  +\eta \cdot  ( \log _2 n + 6) \cdot \bigl [ 1 + \lceil \log _2 (t_j + \beta_j ) \rceil   \bigr ]  \big\}  \notag \\
&\lesssim \sum_{j=1}^q \eta \cdot  \log_2 (t_j + \beta_j) \cdot \log _2 n \lesssim (\log n)^{1+ \xi} ,
\#
where $\xi >0 $ is an absolute constant. Here the last inequality follows from \eqref{eq:CG_param}.
 Moreover, for $\tilde d$ defined  in \eqref{eq:inter_composite_class}, by \eqref{eq:CG_param} we  have  \#\label{eq:tilde_d_bound}
  N \cdot    \max_{j \in [q]} \{ p_{j+1} \cdot (t_j + \beta_j)  \}    \lesssim \tilde d \leq  6 \cdot  N \cdot \bigl (   \max_{j \in [q]} p_j \bigr ) \cdot \big[  \max _{j \in [q] }( t_j  + \beta_j ) \bigr ]    \lesssim N \cdot (\log n )^{2 \xi}    .
 \#
 In addition, combining \eqref{eq:sj_bound},  \eqref{eq:CG_param}, and the fact that $t_j \leq p_j$, we obtain
 \#\label{eq:tilde_s_bound}
 \tilde s &  \lesssim N \cdot \log n  \cdot  \biggl [ \sum_{j=1}^q   p_{j+1} \cdot (t_j + \beta_j + 1) ^{3 + t_j} \bigg]   \notag \\
 &  \lesssim  N   \cdot \log n \cdot  \bigl ( \max_{j \in [q]} p_j  \bigl )  \cdot  \biggl [ \sum_{j=1}^q    (t_j + \beta_j + 1) ^{3 + t_j} \bigg]  \lesssim N \cdot ( \log n )^{1+ 2\xi}.
  \#

 \vspace{5pt}
 {\bf \noindent Step (iii).} In the last step,   
 we show that the function class in \eqref{eq:inter_composite_class} can be embedded in $\cF(L^*, \{d_j^*\}_{j=1}^{L^*+1}, s^*  )$ and characterize the final approximation bias, where $L^*$, $\{d_j^*\}_{j=1}^{L^*+1}$, and $s^*$ are specified in \eqref{eq:dnn_hyperparam}.
To this end, we set
 \#\label{eq:choice_N}
  N = \bigl \lceil \max _{1\leq j \leq q} C \cdot n^{t_j / ( 2 \beta_j^* + t_j )}  \bigr\rceil ,
  \#
 where the absolute constant  $C > 0$ is sufficiently large.
Note that we define  $\alpha^* = \max_{ j \in [q] } t_j / (2 \beta_j^* + t_j ) $. Then \eqref{eq:choice_N} implies that $N \asymp n^{\alpha^*}$.  
 When $n$ is sufficiently large, it holds that $N \geq \max  \{  (\beta+1)^{t_j} , (W+1) e^{t_j}  \} $ for all $j \in [q]$.  
 When $\xi^*$ in \eqref{eq:dnn_hyperparam} satisfies $\xi^* \geq 1+ 2 \xi$, by \eqref{eq:tilde_L_bound} we have
 \$ 
 \tilde L \leq L^* \lesssim (\log n)^{\xi^*}.
 \$
 In addition, \eqref{eq:tilde_d_bound} and \eqref{eq:dnn_hyperparam} implies that we can set $d_j^* \geq \tilde d $ for all $j \in [L^*]$.
Finally, by  \eqref{eq:tilde_s_bound} and \eqref{eq:choice_N}, we have
$
\tilde s  \lesssim n^{\alpha^*} \cdot  ( \log n )^{ \xi^*}   ,
$
which implies $\tilde s + (L^* - \tilde L) \cdot r \leq s^*$.
For   an $\tilde L $-layer  ReLU network  in \eqref{eq:inter_composite_class}, we can make it an $L^*$-layer ReLU network by inserting $L^* - \tilde L$ identity layers, since the inputs of each layer are nonnegative. Thus,  ReLU networks in \eqref{eq:inter_composite_class} can be embedded in
\$
  \cF \bigl [  L^*,  \{r, r, \ldots, r, \tilde d,\ldots, \tilde d, 1\}, \tilde s + (L^* - \tilde L) d \bigr ],
\$
which is a subset of
 $\cF(L^*, \{d_j^*\}_{j=1}^{L+1}, s^*  )$ by \eqref{eq:dnn_hyperparam}.

 To obtain the approximation error $\| \tilde f - f (\cdot , a) \|_{\infty}$,  we define $G_j = h_j \circ \cdots \circ h_1$ and $\tilde G_j = \tilde h_j  \circ \cdots \circ \tilde h_1$ for any $j  \in [q]$.  By triangle inequality, for any $j > 1$ we have
 \#\label{eq:approx_one_step10}
  \| G_j - \tilde G_j  \|_{\infty } &  \leq    \| h_j \circ \tilde G_{j-1}   - h_j \circ G_{j-1}    \|_{\infty} +   \|  \tilde h_j \circ \tilde  G_{j-1}  -  h_j \circ \tilde  G_{j-1}   \|_{\infty}  \notag \\
 &   \leq W \cdot  \|  G_{j-1} - \tilde G_{j-1}  \|_{\infty} ^{\beta_j \wedge 1 } +     \| h_j - \tilde h_j   \|_{\infty} ,
 \#
 where the second inequality holds since $h_j$ is H\"older smooth.
To simplify the notation,  we define
 $\lambda_j = \prod_{\ell = j+1}^q  ( \beta_{\ell } \wedge 1) $  for any $j \in [q-1]$, and set $\lambda_q = 1$. By applying recursion to \eqref{eq:approx_one_step10}, we obtain
 \#\label{eq:approx_err_ff}
 \|   f (\cdot , a) - \tilde f \|_{\infty} =  \| G_q - \tilde G_q \|_{\infty} \leq  W  \sum_{j=1}^q  \| \tilde h_j - h_j \|_{\infty}^{\lambda_j},
 \#
 where the constant $W$ is defined in \eqref{eq:define_const_Q}. Here in \eqref{eq:approx_err_ff} we use the fact that  $(a+ b)^{\alpha} \leq a^{\alpha} + b^{\alpha}$ for all $ \alpha \in [0,1]$ and $a, b > 0$.

 In the sequel, we combine  \eqref{eq:bound_approx_F0}, \eqref{eq:one_term_approx_err}, \eqref{eq:approx_err_ff}, and \eqref{eq:choice_N} to obtain the final bound on $\omega (\cF_0)$.
    Also note that
 $
 \beta_j^* = \beta_j \cdot \prod_{\ell = j+1}^q  ( \beta_{\ell } \wedge 1) = \beta_j \cdot \lambda_j
 $
 for all $j \in [q-1] $. Thus we have $\beta_j^* = \beta_j \cdot \lambda_j$ for all $j \in [q]$. Combining \eqref{eq:approx_err_ff} and  \eqref{eq:trash}, we have
   \#\label{eq:trash2}
    \|   f (\cdot , a) - \tilde f \|_{\infty}  \lesssim  \sum _{j  =1}^q \bigl ( N^{ - \beta_j / t_j } \bigr )^{ \lambda_j} = \sum _{j =1}^q  N^{- \beta_j^* / t_j} \lesssim \max _{j \in [q] } N^{- \beta_j^* / t_j } .
   \#
   Thus, we combine  \eqref{eq:bound_approx_F0}, \eqref{eq:choice_N}, and \eqref{eq:trash2} to obtain
   \#\label{eq:trash3}
   \omega (\cF_0) \leq  \bigl ( \max _{j \in [q] } N^{- \beta_j^* / t_j } \bigr )^2  \asymp \max_{j \in [q] } n^{-2 \beta_j^* / ( 2\beta_j^*  + t_j ) } = n^{\alpha^* -1}.
   \#

 As the final step of the proof, it remains to control the covering number of $\cF_0$ defined in \eqref{eq:define_cF}. By definition, for any $f \in \cF_0$, we have $f (\cdot , a) \in \cF(L^*, \{d_j^*\}_{j=1}^{L^*+1}, s^*  ) $  for any $a \in \cA$.  For notational simplicity, we denote   by $\cN_{\delta}$ the $\delta$-covering of $ \cF(L^*, \{d_j^*\}_{j=1}^{L^*+1}, s^*  )$, that is, we define
 \$
 \cN_{\delta} =  \cN \bigl[ \delta,  \cF(L^*, \{d_j^*\}_{j=1}^{L^*+1}, s^*  ), \| \cdot \|_{\infty} \bigr].
 \$
 By the definition of covering, for any $f \in \cF_0$ and any $a\in \cA$,   there exists $g_a \in \cN _{\delta} $ such that
 $ \| f (\cdot , a) - g_a \|_{\infty} \leq \delta$.
  Then we define a function $g \colon \cS\times \cA \rightarrow \RR$ by $g(s ,a) = g_a(s )$ for any $(s,a) \in \cS \times \cA$.
  By the definition of $g$, it holds that $\| f - g\|_{\infty} \leq \delta$. Therefore,   the cardinality of $\cN( \delta , \cF_0, \| \cdot \| _{\infty}) $ satisfies
  \#\label{eq:cardinality_bound}
  \bigl |\cN( \delta , \cF_0, \| \cdot \|_{\infty}  )  \bigr | \leq | \cN_{\delta} | ^{|\cA|}.
  \#
Now  we utilize the following lemma in \cite{anthony2009neural} to obtain an upper bound of  the cardinality of $\cN_{\delta}$.

  \begin{lemma} [Covering Number of ReLU Network] \label{lemma:cover_number}
Recall that  the family of ReLU networks $\cF(L, \{d_j\}_{j=0}^{L+1}, s, V_{\max}) $  is given in Definition \ref{def:relu}. Let $D = \prod_{\ell=1}^{L+1} ( d_{\ell} +1 )$. For any $\delta > 0$, we have
\$
\log \Bigl | \cN\bigl [ \delta, \cF(L, \{d_j\}_{j=0}^{L+1}, s, V_{\max} ), \| \cdot \|_{\infty}  \bigr ] \Bigr | \leq  (s +1) \cdot \log \bigl [ 2 \delta ^{-1} \cdot (L+1) \cdot D^2 \bigr ] .
\$
\end{lemma}
\begin{proof}
See Theorem 14.5 in \cite{anthony2009neural} for a detailed proof.
\end{proof}

Recall that we
  denote $\cN( 1/ n  , \cF_0, \| \cdot \| _{\infty} )$ by $N_0$ in \eqref{eq:n0}.
By combining \eqref{eq:cardinality_bound} with Lemma \ref{lemma:cover_number} and setting $\delta = 1/n$, we obtain  that
\$
\log  N_0  \leq |\cA| \cdot \log  | \cN_{\delta} |   \leq | \cA | \cdot   (s^* +1) \cdot \log \bigl [ 2 n  \cdot (L^*+1) \cdot D^2 \bigr ],
 \$
 where $D = \prod_{\ell=1}^{L^*+1} ( d_{\ell}^* +1 )$. By the choice of $L^*$, $s^*$, and $\{d_j^*\}_{j=0}^{L^*+1}$ in  \eqref{eq:dnn_hyperparam}, we conclude that
 \#\label{eq:covering_final}
 \log  N_0   \lesssim | \cA |  \cdot s^* \cdot L^* \max_{j \in [L^*]} \log (d_j^*) \lesssim  n^{\alpha^*} \cdot (\log n)^ {1 + 2\xi^*}  .
 \#
 Finally, combining \eqref{eq:err_prop_final}, \eqref{eq:apply_thm_each_term},  \eqref{eq:trash3}, and \eqref{eq:covering_final}, we conclude the proof of Theorem \ref{thm:main}.
\end{proof}

\section{Conclusion}
We study deep Q-network  from the statistical perspective.  Specifically, by neglecting the computational issues, we consider  the fitted Q-iteration algorithm with ReLU networks, which can be viewed as a modification of DQN that  fully captures its  key features. Under mild assumptions,  we show that DQN creates a sequence of policies whose corresponding value functions converge  to the optimal value function, when both the sample size and the number of iteration go to infinity.   Moreover, we establish a precise characterization of both the statistical and the algorithmic rates of convergence.  As a byproduct,   our results  provide theoretical justification for the trick of using a target network in DQN.  Furthermore, we extend DQN to two-player zero-sum Markov games by proposing the Minimax-DQN algorithm. Utilizing the analysis of DQN, we establish theoretical guarantees for Minimax-DQN. To further extend this work, one future direction is to analyze  reinforcement learning methods targeting at MDP  with continuous action spaces, e.g.,  example, soft Q-learning \citep{haarnoja2017reinforcement} and deep deterministic policy gradient (DDPG) \citep{lillicrap2015continuous}.  Another promising direction is to combine results on optimization for deep learning with our statistical analysis to gain  a  unified understanding of the statistical and computational aspects of DQN.


\newpage 

\appendix{}

\section{Deep Q-Network} \label{sec:state_dqn}
We first present the DQN algorithm for MDP in details, which is proposed by \cite{mnih2015human} and adapted here to discounted MDP.  As shown in Algorithm \ref{algo:dqn} below, DQN features two key tricks  that lead to its empirical success, namely, experience replay and target network.
 \begin{algorithm}  [h]
\caption{Deep Q-Network (DQN)} 
\label{algo:dqn}  
\begin{algorithmic} 
 \STATE{{\textbf{Input:}} MDP $(\cS, \cA, P, R, \gamma)$,    replay memory $\cM$,   number of iterations $T$,  minibatch size  $n$,  exploration probability $\epsilon \in (0,1)$,  a family of deep Q-networks $ Q_{\theta} \colon \cS \times \cA   \rightarrow \RR$, an integer  $T_{\text{target}}$ for updating the target network,  and a sequence  of stepsizes $\{\alpha_t \}_{t \geq 0}$.}
 	\STATE{Initialize the replay memory $\cM$ to be empty.}
 	\STATE{Initialize the Q-network with random weights $ \theta$.}
 	\STATE{Initialize the weights of the target network with $ \theta^\star = \theta$. }
	\STATE{Initialize the initial state $S_0$.}
	\FOR {$ t = 0, 1, \ldots, T$} 
	\STATE{With probability $\epsilon$, choose $A_t $ uniformly at random from $\cA$, and with probability $1 - \epsilon$, choose $A_t $ such that $Q_ \theta (S_t, A_t ) = \max_{a \in \cA} Q_ \theta (S_t, a)$.  } 
	\STATE{Execute $A_t$ and observe reward $R_t$ and the next state $S_{t+1}$.}
	\STATE{Store transition $(S_t, A_t, R_t, S_{t+1})$ in $\cM$.} 
	\STATE{{ Experience replay:} Sample random minibatch of transitions  $\{(s_i, a_i, r_i, s_{i}' )\}_{i\in [n]}$ from $\cM$. }
\STATE{For each $i \in [n]$, compute the target  $Y_i =  r_i + \gamma \cdot \max _{a\in \cA} Q_{\theta^\star} (s_i', a )
$.}
\STATE{Update the Q-network: Perform a gradient descent step $$
\theta \leftarrow \theta - \alpha _t \cdot \frac{1}{n}\sum_{i\in[n]} \bigl[ Y_i - Q_{\theta} (s_i, a_i ) \bigr]\cdot \nabla _{\theta} Q_{\theta} (s_i, a_i). $$}
\STATE{Update the target network: Update $\theta^\star \leftarrow   \theta$ every $T_{\text{target} }$ steps.}
\ENDFOR
\STATE{Define policy $\overline \pi$ as the greedy policy with respect to $Q_{\theta}$.}
\STATE{{\textbf{Output:}} Action-value function    $Q_{\theta} $ and policy $\overline \pi$.}
\end{algorithmic}
\end{algorithm}

Furthermore, in the following, we present the details of the Minimax-DQN algorithm that extends DQN to  two-player zero-sum Markov games introduced in \S\ref{sec:zerosum}.  Similar to DQN, this algorithm also utilizes the experience replay and target networks. The main difference is that here the target $Y_i$   in \eqref{eq:minmax_tgt} is obtained by  solving a zero-sum matrix game. In Algorithm \ref{algo:m_dqn} we present the algorithm for the second player, which can be easily modified for   the first player.  We note that for the second player, similar to \eqref{eq:equi_policy},the equilibrium joint policy is defined as 
\#
\bigl [ \tilde \pi_{Q}(\cdot\given s), \tilde \nu_Q (\cdot \given s) \bigr ] =\argmax_{\nu'\in\cP(\cB)} \argmin_{\pi'\in\cP(\cA)} \EE_{a\sim \pi', b\sim \nu'} \big[ Q (s, a,b)\big], \qquad \forall s\in \cS. \label{eq:equi_policy2}
\#

\begin{algorithm}[h]
	\caption{Minimax Deep Q-Network (Minimax-DQN) for the second  player} 
	\label{algo:m_dqn}  
	\begin{algorithmic} 
		\STATE{{\textbf{Input:}} Zero-Sum Markov game $(\cS, \cA, \cB,  P, R, \gamma)$,    replay memory $\cM$,   number of iterations $T$,  minibatch size  $n$,  exploration probability $\epsilon \in (0,1)$, a family of deep Q-networks $ Q_{\theta} \colon \cS \times \cA \times \cB \rightarrow \RR$, an integer  $T_{\text{target}}$ for updating the target network,  and a sequence  of stepsizes $\{\alpha_t \}_{t \geq 0}$.}
		\STATE{Initialize the replay memory $\cM$ to be empty.}
		\STATE{Initialize the Q-network with random weights  $ \theta$.}
		\STATE{Initialize the weights of the target network  by letting  $ \theta^\star = \theta$. }
		\STATE{Initialize the initial state $S_0$.}
		\FOR {$ t = 0, 1, \ldots, T$} 
		\STATE{With probability $\epsilon$, choose $B_t $ uniformly at random from $\cB$, and with probability $1 - \epsilon$, sample $B_t$ according to the equilibrium policy $   \tilde \nu_{Q_{\theta} }(\cdot \given S_t)$ defined in \eqref{eq:equi_policy2}.  } 
		\STATE{Execute $B_t$ and observe the first player's action $A_t$, reward $R_t$ satisfying $-R_t  \sim R(S_t, A_t, B_t) $,  and the next state $S_{t+1} \sim P(\cdot \given S_t, A_t, B_t)$.}
		\STATE{Store transition $(S_t, A_t, B_t,   R_t, S_{t+1})$ in $\cM$.} 
		\STATE{{ Experience replay:} Sample random minibatch of transitions  $\{(s_i, a_i, b_i,  r_i, s_{i}' )\}_{i\in [n]}$ from $\cM$. }
		\STATE{For each $i \in [n]$, compute the target $$
			Y_i =   r_i + \gamma \cdot\max_{\nu'\in\cP(\cB)}  \min_{\pi'\in\cP(\cA)}    \EE_{a  \sim \pi',b  \sim \nu' }   \bigl [Q_{\theta^*}(s_{i}' ,  a ,b ) \bigr ]. 
			$$}
		\STATE{Update the Q-network: Perform a gradient descent step $$
			\theta \leftarrow \theta - \alpha _t \cdot \frac{1}{n}\sum_{i\in[n]} \bigl[ Y_i - Q_{\theta} (s_i, a_i , b_i ) \bigr]\cdot \nabla _{\theta} Q_{\theta} (s_i, a_i, b_i ). $$}
		\STATE{Update the target network: Update $\theta^\star \leftarrow   \theta$ every $T_{\text{target} }$ steps.}
		\ENDFOR
	 
		\STATE{{\textbf{Output:}} Q-network   $Q_{\theta} $ and equilibrium  joint policy with respect to $Q_{\theta} $.}
	\end{algorithmic}
\end{algorithm}


\section{Computational Aspect of DQN}\label{sec:computation}

Recall that  in  Algorithm \ref{algo:fit_Q}  we assume     the  global optima of the  nonlinear least-squares problem in \eqref{eq:target_loss}   is   obtained in each iteration. We make such an assumption as our focus is on the statistical analysis.  In terms of optimization, it has been  shown recently that, when the neural network is overparametrized,  (stochastic) gradient
descent converges to the global minima of the empirical  function.  Moreover, the generalization error of the obtained neural network can also be established.
The intuition behind these results is that, when the neural network is overparametrized, it behaves similar to the     random feature model \citep{rahimi2008random, rahimi2009weighted}.
See, e.g., \cite{du2018gradient2, du2018gradient, zou2018stochastic, chizat2018note, allen2018learning, allen2018convergence,  jacot2018neural, cao2019generalization, arora2019fine,  ma2019comparative,mei2019mean, yehudai2019power, bietti2019inductive, yang2019fine, yang2019scaling, gao2019convergence, bai2019beyond, huang2020deep} and the references therein. Also see \cite{fan2019selective} for a detailed survey.
In this section, we make an initial attempt in providing a unified statistical and computational analysis of DQN. 

In the sequel, we consider the reinforcement learning problem with the state space $\cS = [0,1]^r$ and a finite action space $ \cA$.
To simplify the notation, we represent  action $a$ using one-hot embedding and thus identify it as an element in $\{0,1\} ^{| \cA|} \subseteq \RR^{|\cA|}$.
In practice,   categorical actions are often embedded into the Euclidean space \citep{arnold2015reinforcement}.
Thus, we can pack the state $s$  and the action $a$ together and obtain a vector $(s,a) $  in $\RR^{d }$, where  we denote $r+ | \cA| $ by $d$. Moreover, without loss of generality, we assume that $\| ( s,a) \|_2 \leq 1$. 

We represent the Q-network by the family of  two-layer neural networks
\#\label{eq:two_layer_nn}
Q \bigl (s ,a ; b, W \bigr ) = \frac{1}{\sqrt{ 2m}}   \sum_{j =1}^{2m}  b_j  \cdot \sigma [  W_{j}^\top (s,a) ], \qquad \forall (s,a) \in \cS\times \cA.
 \#
Here $2 m$ is the number of neurons,  $b_j \in \RR$ and $ W_j \in \RR^{d  }$ for all $j \in [2m]$, and $\sigma(u) = \max \{u, 0\}$ is the ReLU activation function. Here
$ b = (b_1, \ldots, b_{2m})^\top \in \RR^{2m}$ and $W = (W_1, \ldots, W_{2m}) \in \RR^{d \times {2m}} $ are
 the weights of the neural network.
 
 For such class of neural networks, for any $k\geq 1$, in 
$k$-th iteration of the neural FQI algorithm, 
  the optimization problem in \eqref{eq:target_loss}
 becomes
\#\label{eq:opt}
\minimize _{b, W }   \frac{1}{2n}\sum_{i=1}^n \bigl [ Y_i - Q(S_i, A_i; b, W  ) \bigr  ]^2,
\#
where $Y_i  = R_i + \gamma \cdot \max_{a \in \cA } \tilde Q_{k-1}  (S_i', a)$ is the target and $\tilde Q_{k-1}  $ is the Q-network computed in the previous iteration.  Notice that this problem is  a  least-squares regression with overparameterized neural networks. 
For computational efficiency, we propose to  solve   \eqref{eq:opt} via stochastic gradient descent (SGD).  
Specifically, 
in each iteration of SGD, we sample a fresh observation $(S, A, R, S') $  
with $(S,A) $ drawn  from the sampling distribution $ \sigma$, $R \sim R(\cdot \given S,A)$, and $S' \sim P(\cdot \given S,A)$. 
Then an estimator of the gradient is computed based on  $(S, A, R, S') $, which is used to update the network parameters. We run the SGD updates for a total of $n$ iterations and denote the output by $\tilde Q_k$.

Besides,  in each FQI-step, 
  we initialize  the parameters  via the symmetric initialization scheme  \citep{gao2019convergence, bai2019beyond} as follows. 
   For any $ j \in [m]$, we set  $b_j \overset{\text{i.i.d.}}{\sim}  \textrm{Unif}(\{ -1, 1\}) $ and $W_{j } \overset{\text{i.i.d.}}{\sim}  N(0, I_d / d)$, where $I_d
$ is the identity matrix in $\RR^d$.
For any $j \in \{ m+1, \ldots, 2m\}$, we set $b_j = - b_{j-m}$ and $W_j = W_{j-m}$. 
We remark that such 
 initialization implies that  the initial $Q$-network is a zero function, which is used only to simply the theoretical analysis.
Besides, 
 for ease of presentation, during training we  fix  the value of $ b $   at its   initial value  and only optimize over $W$. 
 We initialize $b$ and $W$ at the very beginning of our algorithm and in each FQI subproblem, we update the Q-network starting from the same initialization. 
 Hereafter, we denote the initial value of $W$ and $b$ by $W^{(0)} \in \RR^{ d\times 2m } $  and   $b^{(0)} \in \RR^{2m} $, respectively, 
 and let $Q(\cdot, \cdot; W)$ denote $Q(\cdot, \cdot; b^{(0)} , W)$. 
 In order to have bounded functions, we further  restrict the weight $W$ to a Frobenius norm ball centered at $W^{(0)} $ with radius $B> 0$, i.e., we define
 \#\label{eq:weight_ball_W}
 \cB_B = \bigl \{ W \in \RR^{d\times 2m} \colon \| W -  W^{(0)}  \|_{\fro} \leq B \bigr \} ,
 \#
 where $B $ is a  sufficiently large constant.
Thus, the population version of the $k$-th iteration of the FQI algorithm becomes 
\#\label{eq:new_opt}
\minimize_{W \in \cB_B} L ( W) =  \EE  \bigl \{  \bigl [ Y - Q(S,  A;   W  ) \bigr  ]^2 \big\} ,
\#
where $(S, A) \sim \sigma$ and $Y$ is computed using   $\tilde Q_{k-1}$. We solve this optimization problem via projected SGD, which generates a sequence of weight matrices  $\{ W^{(t)}  \}_{t\geq 0} \subseteq \cB_B$ satisfying
\#
  W ^{ (t)  }    = \Pi _{\cB_{B} } \Bigl [ W^{ (t -1) }    -     \eta \cdot  \bigl [ Y_t - Q\bigl (S_t,  A_t;  W^{ (t - 1 ) }  \bigr ) \bigr  ]\cdot \nabla_W  Q\bigl (S_t, A_t;    W ^{(t - 1) }   \bigr ) \Bigr ], \quad \forall t \geq 1, 
  \label{eq:pgd_algo}
\#
where  $ \Pi _{\cB_{B} }$ is the projection operator onto $\cB_{B}$ with respect to the Frobenius norm, $ \eta > 0$ is the step size, and $(S_t, A_t, Y_t)$ is a random observation. We present the details of fitted Q-iteration method with projected SGD in Algorithm~\ref{algo:fit_Q_overparam}.
  
 \begin{algorithm} [h]
\caption{Fitted Q-Iteration Algorithm with Projected SGD Updates} 
\label{algo:fit_Q_overparam} 
\begin{algorithmic} 
\STATE{{\textbf{Input:}} MDP $(\cS, \cA, P, R, \gamma)$, function class $\cF$, sampling distribution $\sigma$, number of FQI iterations $K$, number of SGD iterations $T$, the initial estimator $\tilde Q_0$.}
\STATE{Initialize the weights $b^{(0)}$ and $W^{(0)}$ of Q-network via the symmetric initialization scheme.}
\FOR{$k = 0, 1, 2, \ldots, K-1$}
\FOR{$t = 1, \ldots, T$}
\STATE{Draw an independent sample $ (S_t , A_t, R_t, S_t')  $ with $(S_t , A_t ) $ drawn from distribution $ \sigma$.} 
\STATE{Compute  $Y_t= R_t + \gamma \cdot \max _{a\in \cA} \tilde Q_k (S_t', a)$.}
\STATE{Perform projected SGD update
\$
\tilde W^{(t)} & \leftarrow  W^{ (t -1) }    -     \eta \cdot  \bigl [ Y_t - Q\bigl (S_t,  A_t;  W^{ (t - 1 ) }  \bigr ) \bigr  ]\cdot \nabla_W  Q\bigl (S_t, A_t;    W ^{(t - 1) }   \bigr )   \\
W^{(t)} & \leftarrow \Pi _{\cB_{B} } \bigl (  \tilde W^{(t)} \bigr ) = \argmin_{W \in \cB_B} \| W - \tilde W^{(t) } \|_{\fro}  .
\$ }
\ENDFOR
\STATE{Update the action-value function $ \tilde Q_{k+1} (\cdot, \cdot)  \leftarrow Q(\cdot, \cdot; W_{k+1} ) $ where $W_{k+1} = T^{-1 } \sum_{t=1}^T W^{(t) } $. }
\ENDFOR
\STATE{Define policy $\pi_K$ as the greedy policy with respect to $\tilde Q_{K}$.}
\STATE{{\textbf{Output:}} An estimator $\tilde Q_{K }$ of $Q^*$ and policy $\pi_K$.}
\end{algorithmic}
\end{algorithm}

 To understand the convergence of the projected SGD updates in \eqref{eq:pgd_algo}, we utilize the fact that the dynamics of training overparametrized  neural networks is  captured by the neural tangent kernel  \citep{jacot2018neural} when the width is sufficiently large. 
 Specifically, 
 since $\sigma(u) = u \cdot \ind\{ u > 0\}$, the gradient of 
 the Q-network in \eqref{eq:two_layer_nn} is given by 
 \#
 \nabla_{W_j } Q(s,a; b, W) &  =   1 / \sqrt{2m} \cdot   b_j  \cdot \ind \{   W_{j}^\top (s,a) > 0   \} \cdot (s,a)   ,    \qquad \forall j \in [2m].  \label{eq:grad_relu2}
 \#
  Recall that we initialize parameters $b$ and $W$ as $b^{(0)}$ and $W^{(0)}$ and that we only update $W$ during training. 
  We define a   function class $\cF_{B, m }^{(t)}$ as 
 \#\label{eq:Fbmt}
 \cF_{B, m} ^{(t)} = \biggl \{   \hat Q (s,a; W) =  \frac{1} { \sqrt{2m}}   \sum_{j =1}^{2m}  b_j^{(0)}   \cdot \ind \bigl  \{   (W_{j}^{(t)} ) ^\top (s,a) > 0   \bigr \} \cdot W_{j}^\top (s,a) \colon  W \in \cB_{B} \biggr \} .
 \#
 By \eqref{eq:grad_relu2}, for each function $\hat Q(\cdot, \cdot; W) \in \cF_{B, m }^{(t)}$, we can write it  as 
$$\hat Q(\cdot, \cdot; W) = Q(\cdot, \cdot; W^{(t)}) + \bigl \la \nabla _{W} Q(\cdot, \cdot; W^{(t)}) , W - W^{(t)} \bigr  \ra , \qquad \forall W \in \cB_B, $$
 which 
is the first-order linearization  of $Q(\cdot, \cdot; W^{(t)})$ at $W^{(t)}$. 
 Furthermore, since $B$ in \eqref{eq:weight_ball_W} is a constant, for each weight matrix $W$ in $ \cB_B$, when $m$ goes to infinity, $\|  W_j - W_j^{(0)} \|_2  $ would  be small for almost all $j \in [2m]$, which implies  that 
 $ \ind \{   W_{j} ^\top (s,a ) > 0\} = \ind\{ (W_j^{(0)} ) ^\top (s,a) > 0 \} $  holds with  high probability for all $j \in [2m]$ and $(s,a) \in \cS \times \cA$. 
 As a result, when $m$  is sufficiently large,  $ \cF_{B, m} ^{(t)}$ defined in \eqref{eq:Fbmt} is close to
 \#\label{eq:Fbm}
 \cF_{B, m}^{(0)}   =  \biggl \{    \hat Q(s,a; W) =  \frac{1} { \sqrt{2m}}   \sum_{j =1}^{2m}  b_j ^{(0)} \cdot \ind \bigl  \{   ( W_{j}^{(0) } ) ^\top (s,a) > 0   \bigr \} \cdot W_{j}^\top (s,a) \colon   W \in \cB_{B} \biggr \} ,
 \#
 where $b^{(0)}$ and $W^{(0)}$ are the initial parameters. 
 Specifically, as proved in Lemma A.2 in \cite{wang2019neural}, when the sampling distribution  $\sigma$ is regular in the sense that Assumption \eqref{assume:sampling} specified below is satisfied, 
 for any $W_1, W_2 \in \cB_{B}$, we have
\$
\EE_{\textrm{init}} \Bigl [ \bigl  \| \la \nabla_{W} Q (\cdot, \cdot ; W_1) -  \nabla_{W} Q (\cdot, \cdot ; W^{(0)}) , W_2 \bigr \ra \bigr \|_{\sigma} ^2   \Bigr ]  = \cO( B^3 \cdot m^{- 1/2}   ) ,
\$
where $\EE_{\textrm{init}} $ denotes that the expectation is taken with respect to the initialization of the network parameters. Thus, when the network width $2m$ is sufficiently large such that $B^3 \cdot m^{- 1/2} = o(1)$, the linearized  function classes $\{ \cF_{B, m}^{(t)}\}_{t\in [T]} $ are all close to $\cF_{B, m}^{(0)}$.

 To simplify the notation, for 
 $b \in \{ -1, 1\}$ and $ W \in \RR^d$, we define feature mapping $\phi (\cdot, \cdot; b, W)\colon \cS \times \cA \rightarrow \RR^d$ as 
 \#\label{eq:tangent_features} 
 \phi(s,a; b, W) = b \cdot \ind \{ W^\top (s,a) > 0 \} \cdot (s,a) , \qquad  \forall (s,a) \in  \cS\times \cA.
 \#
 Besides, 
 for all $j \in [2m]$, we let $\phi_j$ denote 
  $\phi(\cdot, \cdot; b^{(0)}_j, W^{(0)}_j)$.  
 Due to the symmetric initialization scheme, 
 $\{ \phi_j  \}_{j\in [m]}$ are i.i.d. 
 random feature functions and $\phi_j    = -  \phi_{j+m}   $ for all $j \in [m]$. 
 Thus, each $ \hat Q(\cdot ,\cdot ; W) $ in \eqref{eq:Fbm} can be equivalently written as 
 \#\label{eq:simplify_Q_lin}
 \hat  Q(s,a ; W) =  \frac{1} { \sqrt{2m}}   \sum_{j =1}^{2m} \phi_j (s,a) ^\top W_j  = \frac{1} { \sqrt{2m}}   \sum _{j=1}^m \phi_j(s,a) ^\top ( W_j - W_{j+m}). 
 \#
 Let $ W_j ' = ( W_j - W_{j+m} ) /\sqrt{2}$. Since $W \in \cB_B$, we have 
 \#\label{eq:bound_w_norm_nn}
 \sum_{j=1}^m \|W_j '\|_2^2 =   \frac{1}{2 } \sum_{j=1}^m \bigl  \| (W_j  - W_{j}^{(0)}  ) -  ( W_{j+m} - W_{j+m} ^{(0)}  ) \bigr   \|_2^2 \leq \sum_{j=1}^{2m}    \| W_j  - W_{j}^{(0)}  \|^2 \leq B^2,   
 \# 
 where we use the fact that $W_j^{(0)} = W_{j+m}^{(0)}$ for all $j \in [m]$. 
 Thus, combining \eqref{eq:simplify_Q_lin} and \eqref{eq:bound_w_norm_nn}, we conclude that 
 $\cF_{B, m}^{(0)}$ in  \eqref{eq:Fbm} is a subset of $\cF_{B, m}$ defined as 
 \# \label{eq:Fbm0}
 \cF_{B, m}  =  \biggl \{   \hat Q(s,a; W) =    \frac{1}{\sqrt{m} }   \sum_{j =1}^m  \phi_j(s,a) ^\top W_j  \colon W \in \{  W' \in \RR^{d\times m} \colon  \| W'  \|_{\fro}  \leq B \}  \biggr \} .
 \#
 Notice that each function in  $ \cF_{B, m}$ is a linear combination of $m$ i.i.d. random features. 
 In particular, let $\beta \in  \textrm{Unif}(\{ -1, 1\})$ and $\omega  \sim N(I_d / d)$ be two independent random variables and let $\mu$ denote their joint distribution. 
 Then the random feature $\phi(\cdot, \cdot; \beta, \omega)$ induces a 
 reproducing kernel
 Hilbert space $\cH$ \citep{rahimi2008random, rahimi2009weighted, bach2017equivalence} with  kernel $K\colon (\cS \times \cA)\times (\cS \times \cA) \rightarrow \RR$ given by 
 \#\label{eq:kernel}
 K\bigl [(s,a), (s',a') \bigr ] & = \EE_{\mu} \bigl  [ \phi(s,a; \beta, \omega) ^\top \phi(s,a; \beta, \omega)  \big ]  \notag \\
 & =  \EE_{\omega} \bigl [ \ind \{ \omega ^\top (s,a) > 0 \} \cdot \ind \{ \omega^\top (s',a') > 0 \} \cdot    (s,a) ^\top (s', a' )   \bigr ] .
 \#
Each function  in $\cH$ can be represented by a mapping $\alpha \colon \{-1, 1\} \times \RR^d \rightarrow \RR^d$ as  
\$
f_\alpha (\cdot, \cdot) = \int _{\{ -1, 1\} \times \RR^d} \phi(\cdot, \cdot ; b, W)^\top \alpha( b, W) ~\ud \mu(b, W) = \EE_{\mu} \big [ \phi (\cdot, \cdot; \beta, \omega) ^\top \alpha (\beta, \omega) \big ]  .
\$
For two functions $f_{\alpha_1}$ and $f_{\alpha_2}$ in $\cH$ represented by $\alpha_1$ and $\alpha_2$, respectively, 
their inner product is given by 
\$
\la f_{\alpha_1} , f_{\alpha_2} \ra _{\cH} = \int _{\{ -1, 1\} \times \RR^d} \alpha_1(b, W)^\top \alpha_2 (b, W) ~\ud \mu(b, W) = \EE_{\mu} \bigl [ \alpha_1 (  \beta, \omega) ^\top \alpha_2 (\beta, \omega)  \big ]. 
\$
We let $\| \cdot \|_{\cH}$ denote the RKHS norm of $\cH$. 
Then, when $m$ goes to infinity, $\cF_{B,m}$ in  \eqref{eq:Fbm0} converges to
the RKHS norm ball $\cH_B= \{ f\in \cH \colon \| f \|_{\cH} \leq B \}$.  

Therefore, from the perspective of neural tangent kernel, when the Q-network is represented by the class of  overparametrized neural networks given in \eqref{eq:two_layer_nn} with a sufficiently large number of neurons,   each population problem associated with each  FQI iteration in \eqref{eq:new_opt} becomes 
\$
\minimize_{Q \in \cH}  \EE  \bigl\{  [ Q(S,A) - Y ]^2  \bigr \} ,
\$
where the minimization is over a subset of $\cH_B$ as $\cF^{(0)}_{B,m}$ is a subset of  $\cF_{B, m}$.

Utilizing the connection between neural network training and RKHS, in the sequel, we provide a jointly statistical and computational analysis of Algorithm \ref{algo:fit_Q_overparam}. 
To this end, we define a function class $\cG_B$ as 
\#\label{eq:subset_rkhs}
\cG_B = \biggl \{  f_{\alpha} (s,a) =     \int _{\{ -1, 1\} \times \RR^d}  \phi (s,a; b,  W) ^\top \alpha (b, W) ~ \ud \mu(b, W)    \colon     \| \alpha(\cdot, \cdot )  \|_\infty   \leq B   /\sqrt{d}    \biggr \}.
\#
 That is, each function in $\cG_B$ is represented by a feature mapping $\alpha \colon \{ -1, 1\} \times \RR^d\rightarrow \RR^d $ which is almost surely bounded in the $\ell_{\infty}$-norm. Thus, $\cG_B$ is a strict subset of the RKHS-norm ball $\cH_B$. 
When $B$ is sufficiently large, $\cG_B$  is known to be a rich function class \citep{hofmann2008kernel}.  Similar to Assumption \ref{assume:closedness} in \S\ref{sec:theory}, 
we impose the following assumption on the Bellman optimality operator.

  \begin{assumption} [Completeness of $\cG_B$]\label{assume:Bellman_CGB}
   We assume that that $T Q(\cdot, \cdot ; W)  \in \cG_B$ for all for all $W \in \cB_{B}$, where $T$ is the Bellman optimality operator and $ Q(\cdot, \cdot ; W) $ is given in \eqref{eq:two_layer_nn}. 
        \end{assumption} 
 This assumption specifies that $T$ maps any neural network with weight matrix $W$ in $\cB_B$ to a subset $\cG_B$ of the RKHS $\cH$. When $m$ is sufficiently large, this assumption is similar to stating that $\cG_B$ is approximately closed under $T$.   
 
  We also impose the following regularity condition on the sampling distribution $\sigma$. 
\begin{assumption}[Regularity Condition on $\sigma$]
\label{assume:sampling}
We assume that there exists  an absolute constant $C > 0$ such that 
\$
 \EE_{\sigma} \Bigl[\ind\bigl\{|y ^\top (S, A)| \leq u\bigr\} \Bigr] \leq C \cdot u/\|y\|_2, \qquad \forall y \in \RR^d, ~ \forall u >0.
\$
\end{assumption}

Assumption  \ref{assume:sampling} states that the density of $\sigma$ is sufficiently regular, which holds when  the density is upper bounded. 

Now we are ready to present the main result of this section, which characterizes the performance  of   $\pi_K$ returned by Algorithm \ref{algo:fit_Q_overparam}. 

 \begin{theorem}\label{thm:2layer}
 In Algorithm \ref{algo:fit_Q_overparam}, we assume that each step of the fitted-Q iteration is solved by  $T$ steps of projected SGD updates with a constant stepsize $\eta > 0$. 
  We set $T = C_1 m$ and  $\eta = C_2 /\sqrt{T}$, where $C_1$, $C_2$ are absolute constants that are properly chosen. Then, under  
  Assumptions \ref{assume:concentrability}, \ref{assume:Bellman_CGB}, and  \ref{assume:sampling},  we have 
 \#\label{eq:overparam_final}
  \EE_{\textrm{init}} \bigl   [ \|  Q^* - Q^{\pi_K }\|_{1, \mu}  \bigr ] =  \cO \bigg (\frac{ \phi_{\mu, \sigma}\gamma} {(1 -\gamma)^2 } \cdot ( B^{3/2}\cdot m^{-1/4} + B^{5/4}\cdot m^{-1/8}  )+ \frac{    \gamma^{K+1}   }{(1- \gamma) ^2 }  \cdot R_{\max} \bigg ), 
 \#
where $ \EE_{\textrm{init}}$ denotes that the expectation is taken with respect to the randomness of the initialization.
 \end{theorem} 
 
 As shown in \eqref{eq:overparam_final}, the error $  \EE_{\textrm{init}}     [ \|  Q^* - Q^{\pi_K }\|_{1, \mu}    ]$ can be similarly written as the sum of a statistical error and an algorithmic error, where the algorithmic error converges to zero at a linear rate as $K$ goes to infinity. The statistical error corresponds to the error incurred in  solving each FQI step via $T$  projected SGD steps. As shown in \eqref{eq:overparam_final},  when $B$ is regarded as a constant, with $T \asymp m$ projected SGD  steps, we obtain an estimator  with error $\cO( m^{-1/8} )$. Hence, Algorithm \ref{algo:fit_Q_overparam}  finds the globally optimal policy when both $m$ and $K$ goes to infinity. 
Therefore, when using overparametrized neural networks, our fitted Q-iteration algorithm provably attains both statistical accuracy and computational efficiency. 
  
  Finally,   we remark that  focus  on the class of two-layer overparametrized ReLU neural networks only for the simplicity of presentation. The theory of neural tangent kernel can be extended to feedforward  neural networks with multiple layers and neural networks with more complicated architectures \citep{gao2019convergence, frei2019algorithm, yang2019fine, yang2019scaling, huang2020deep}.

 \subsection{Proof of Theorem \ref{thm:2layer}}
 \begin{proof} 
  Our proof is similar to that of Theorem \ref{thm:main}. For any $k \in [K]$, we define the 
  maximum one-step approximation error as  $\varepsilon_{\max} = \max_{ k \in [K] }\EE_{\textrm{init}} [ \| T \tilde Q_{k-1} - \tilde Q_{k} \|_{\sigma} ] $,
  where $ \EE_{\textrm{init}} $ denotes that the  expectation is taken with respect to the randomness in the initialization of network weights, namely $b^{(0)}$ and $W^{(0
  )}$. 
   By Theorem \ref{thm:err_prop}, we have
 \#\label{eq:err_prop_final000}
  \EE_{\textrm{init}} \bigl   [ \|  Q^* - Q^{\pi_K }\|_{1, \mu}  \bigr ]  \leq  \frac{2\phi_{\mu, \sigma}\gamma} {(1 -\gamma)^2 } \cdot \varepsilon_{\max} + \frac{4   \gamma^{K+1}   }{(1- \gamma) ^2 }  \cdot R_{\max},
   \#
   where $\phi_{\mu, \sigma}$, specified in   Assumption \ref{assume:concentrability},  is a constant that only depends on the concentration coefficients. Thus, it remains to characterize $\| T \tilde Q_{k-1} - \tilde Q_{k} \|_{\sigma}$ for each $k$, which corresponds to the prediction risk of the estimator constructed by $T$ projected SGD steps. 
   
  In the sequel, we characterize the prediction risk of the projected SGD method   via the framework of  neural tangent kernel. Our proof technique is motivated by recent work \citep{gao2019convergence, cai2019neural,  liu2019neural,  wang2019neural, xu2019finite} which analyze the training of overparametrized neural networks via  projected SGD  for    adversarial training and reinforcement learning. 
  We focus on the case where the target network is $\tilde Q_{k-1}$  and bound $\| T \tilde Q_{k-1} - \tilde Q_{k} \|_{\sigma}$.

  To begin with, recall that we define function classes $\cF_{B, m}^{(t)}$, $\cF_{B, m}$, and $\cG_{B}$ in  
  \eqref{eq:Fbmt}, \eqref{eq:Fbm0}, and \eqref{eq:subset_rkhs}, respectively. 
  Notice that $\tilde Q_{k-1} =  Q(\cdot, \cdot; W_{k} ) $ is a  neural network where $W_k \in \cB_B$ due to projection. 
  Then,  by Assumption \ref{assume:Bellman_CGB}, $ T\tilde Q_{k-1} $ belongs to $\cG_B$, which is a subset of the RKHS $\cH$. 
  Thus, it suffices to study the least-squares regression problem where  the target function is  in $\cG_B$ and the neural network is trained via projected SGD. 

In the following, we use function class $\cF_{B,m}$ to connect $\cG_B$ and $\tilde Q_k$. Specifically, we show that any function $g \in \cG_B$ as well as    $\tilde Q_k$ can   be well approximated by  functions in $\cF_{B, m}$ and  quantify the approximation errors. Finally, we focus on the projected SGD algorithm within  the linearized function class $\cF_{B, m}$ and establish the statistical and computational error. The proof is divided into three steps as follows.

 \vspace{5pt} 
 {\bf \noindent Step (i).} In the first step, we quantify the difference between $\tilde Q_{k} $ and functions in $\cF_{B, m}$. 
 For any $j \in [2m]$, let $[W_{k+1} ]_{j} \in \RR^d$ be the weights of $\tilde Q_{k}$  corresponding to the $j$-th neuron. 
 We define a function  $\hat Q _{k} \colon \cS \times \cA \rightarrow \RR$ as 
 $
 \hat Q _{k} (s,a) = 1 / \sqrt{2m} \cdot \sum_{j=1}^{2m} \phi_j (s,a) ^\top [ W_{k+1}]_j,
 $ 
 which belongs to $\cF_{B, m}^{(0)}$ defined  in \eqref{eq:Fbm}, a subset of $\cF_{B, m}$. 
The following 
 following lemma, obtained from \cite{wang2019neural}, prove that  $Q(\cdot, \cdot; W)$ is close to a linearized function 
 when $m$ is sufficiently large. 
\begin{lemma} [Linearization Error] \label{lemma:linearization_nn}
Let $\cB_B$ be defined in \eqref{eq:weight_ball_W}. 
Under Assumption \ref{assume:sampling}, for any $W^{(1)},W^{(2)} \in \cB_{B}$, we have
\$
\EE_{\textrm{init}} \Bigl [ \bigl  \| \la \nabla_{W} Q (\cdot, \cdot ; W^{(1)} ) -  \nabla_{W} Q (\cdot, \cdot ; W^{(0)}) , W^{(2)} \bigr \ra \bigr \|_{\sigma} ^2   \Bigr ]  = \cO( B^3 \cdot m^{- 1/2}   ) .
\$
\end{lemma} 
\begin{proof}
See Lemma A.2 in \cite{wang2019neural} for a detailed proof. 
\end{proof}
 
 Notice that we have $\tilde Q_k (\cdot, \cdot)  = \la  \nabla _W Q (\cdot, \cdot ; W_{k+1} ), W_{k+1}\ra $ and $\hat Q_k (\cdot, \cdot) =   \la  \nabla _W Q (\cdot, \cdot ; W^{(0)} ), W_{k+1}\ra $. 
 Applying Lemma \ref{lemma:linearization_nn} with  $W^{(1)} = W^{(2)}= W_{k+1}$, we obtain that 
\#\label{eq:step11}
\EE_{\textrm{init}} \bigl [ \bigl  \| \tilde  Q _{k} - \hat Q  _{k}  \bigr \|_{\sigma} ^2   \bigr ]   = \EE_{\textrm{init}} \Bigl [  \EE _{\sigma}\big\{   [ \tilde Q_{k} (S,A) - \hat Q  _{k}  (S,A) ] ^2 \bigr \}    \Bigr ]  = \cO( B^3  \cdot m^{- 1/2} ). 
\#
Thus, we have constructed a function in $\cF_{B, m}$ that is close to $\tilde Q_k$ when $m$ is sufficiently large, which completes the first step of the proof.

    \vspace{5pt} 
 {\bf \noindent Step (2).}  In the second step, we show that each function in $\cG_B$ can also be well approximated by functions in $\cF_{B, m}$. 
 To this end, similar to the definition of $\cG_B$ in \eqref{eq:subset_rkhs}, we define a function class $\overline\cF_{B, m}$ as 
   \#\label{eq:new_Fset}
\overline\cF_{B, m} = \biggl\{\hat Q(s,a; W)= \frac{1}{\sqrt{m}}\cdot\sum^m_{j = 1}\phi_j (s, a) ^\top  W_j  
 \colon \| W_j\|_{\infty} \leq B/\sqrt{md}\biggr\},&
\#
which is a subset of $\cF_{B, m}$ by definition. 
Intuitively, as $m$ goes to infinity, $\overline\cF_{B, m}$ becomes $\cG_B$. 
The following Lemma, obtained from   \citep{rahimi2009weighted},  provides a rigorous  characterization of this argument.

\begin{lemma}[Approximation  Error of $\overline\cF_{B, m}$ \citep{rahimi2009weighted}]
\label{lem::inf_dist}
Let $Q $ be any fixed function in $\cG_B$ and define 
$\overline \Pi_{ B, m}Q \in \overline\cF_{R, m}$ as the solution to 
\$
\minimize_{f \in \overline \cF_{B, m} } \| f - Q \|_{\sigma}.
\$
Then, there exists  a constant $C > 0$ such that,  for any $t > B/\sqrt{m} $, we have 
\#\label{eq::inf_dist}
 \PP_{\textrm{init} } \Bigl ( \bigl  \| \overline \Pi_{ B, m}Q   - Q \bigr \|_{\sigma} > t\Bigr ) \leq C\cdot \exp \bigl[ -1/2 \cdot (t\cdot\sqrt{m} / B - 1)^2\bigr] .
\#
\end{lemma}
\begin{proof}
See \cite{rahimi2009weighted} for a detailed proof.
\end{proof}
Now we integrate the tail probability in \eqref{eq::inf_dist} to  obtain a bound on $ \EE_{\textrm{init} } [   \| \overline \Pi_{ B, m}Q   - Q  \|_{\sigma } ] $.  
Specifically, by direct computation, we have 
\#\label{eq:step21}
& \EE_{\textrm {init}}  \Bigl[\big \| \overline \Pi_{ B, m}Q   - Q \big \|_{\sigma } ^2   \Bigr] = \int_0^\infty  2 t  \cdot  \PP_{\textrm{init} }\Bigl ( \bigl  \| \overline \Pi_{ B, m}Q   - Q \bigr \|_{\sigma} ^2 > t ^2 \Bigr ) ~ \ud t \notag\\
 &\qquad \leq  2 B ^2 /m  + 2 C \cdot \int_{B/\sqrt{m}} ^\infty  t \cdot   \exp\bigl [ -1/2 \cdot (t\cdot\sqrt{m }/ B - 1)^2\bigr] ~ \ud t  \notag \\
 & \qquad = B ^2 /\sqrt{m} + C\cdot B ^2/  m   \cdot \int_{0} ^\infty  ( 1 +u)\cdot   \exp(  -1/2 \cdot u^2 ) ~\ud u= \cO(B^2 \cdot m^{-1}),
\#
where in the second equality we let $u = t\cdot\sqrt{m }/ B - 1$. 
Thus, by \eqref{eq:step21} we have 
\#\label{eq:step22}
 \EE_{\textrm {init}}  \Bigl[\big \| \tilde T \tilde Q_{k-1} - T \tilde Q_{k-1} \big\|_{\sigma}^2  \Bigr  ] = \cO(B^2 \cdot m^{-1}),
\#
where $\tilde T \tilde Q_{k-1} =  \overline \Pi_{ B, m} T \tilde Q_{k-1}$. Thus, we conclude the second step.

    \vspace{5pt} 
 {\bf \noindent Step (3).} 
 Finally, in the last step, we utilize the existing analysis of projected SGD over function class $\cF_{B, m}$ to obtain an upper bound on $\| \tilde T \tilde Q_{k-1} - \hat Q_k\|$. 

\begin{theorem}[Convergence of Projected SGD \citep{liu2019neural}]
\label{thm::TD_converge}
In Algorithm \ref{algo:fit_Q_overparam}, let 
  $T$ be the number of iterations of projected SGD steps for solving each iteration of the FQI update and 
we  set $\eta  =  \cO( 1/\sqrt{T})$. Under Assumption \ref{assume:sampling}, it holds  that
\#\label{eq::SGDn}
\EE_{{\rm init}}\Bigl[\big\| \hat Q_k  - \tilde T \tilde Q_{k-1}  \big\|_{\sigma}^2 \Bigr] \leq   \cO  (B^2\cdot T ^{-1/2} + B^3\cdot m^{-1/2} + B^{5/2}\cdot m^{-1/4}  ). 
\#
\end{theorem}
\begin{proof}
See 
 Theorem 4.5 in  \cite{liu2019neural} for a detailed proof.
\end{proof}
 
   Finally, combining  \eqref{eq:step11}, \eqref{eq:step22},  and  \eqref{eq::SGDn}   we obtain that 
   \#\label{eq:final_error}
  &  \EE_{\textrm{init}} \Bigl [    \bigl  \| \tilde  Q _{k} - T\tilde Q  _{k-1}  \bigr \|_{\sigma} ^2    \Bigr ]  \notag \\
   & \qquad \leq 3  \EE_{\textrm{init}} \Bigl [  \bigl  \| \tilde  Q _{k} - \hat  Q  _{k}  \bigr \|_{\sigma} ^2    \Bigr ] +  3   \EE_{\textrm {init}}  \Bigl[\big \| \tilde T \tilde Q_{k-1} - T \tilde Q_{k-1} \big\|_{\sigma}^2  \Bigr  ]    + 3 \EE_{\textrm{init}} \Bigl [  \big\| \hat Q_k  - \tilde T \tilde Q_{k-1}\big  \|_{\sigma}^2    \Bigr ] \notag \\
   & \qquad = \cO \bigl ( B^2\cdot T ^{-1/2} + B^2 \cdot m^{-1} +   B^3\cdot m^{-1/2} + B^{5/2}\cdot m^{-1/4}  \bigr  ). 
   \#
   Setting $ T \asymp m $ in \eqref{eq:final_error}, we obtain that 
   \#\label{eq:final_stepp}
   \varepsilon_{\max} = \max_{ k \in [K] }\EE_{\textrm{init}} \Bigl  [ \bigl \| T \tilde Q_{k-1} - \tilde Q_{k} \bigr  \|_{\sigma}  \Bigr ] \leq \cO( B^{3/2}\cdot m^{-1/4} + B^{5/4}\cdot m^{-1/8}). 
   \#
 Combining \eqref{eq:err_prop_final000} and  \eqref{eq:final_stepp}, we conclude the proof of Theorem \ref{thm:2layer}.
 \end{proof}


\section{Proofs of Auxiliary Results} \label{sec:full_proofs}

 In this section, we  present the proofs for   Theorems \ref{thm:err_prop} and  \ref{thm:each_term_error}, which are  used in the \S\ref{proof:thm:main}  to establish our main theorem.

\subsection{Proof of Theorem \ref{thm:err_prop} } \label{proof:thm:err_prop}
  \begin{proof}
 Before we present the proof, we introduce some notation. For any $k \in \{0, \ldots, K-1\}$, we denote $T \tilde Q_{k}$ by $Q_{k+1}$ and define 
 \#\label{eq:rhok}
 \varrho_{k} = Q_k - \tilde{Q}_k.
 \#  Also, we denote by  $\pi_k$ the greedy policy with respect to $\tilde{Q}_k$.   In addition, throughout the proof, for two functions $Q_1, Q_2 \colon \cS \times \cA \rightarrow \RR$, we use the notation 
 $Q_1 \geq Q_2$ if $Q_1(s,a) \geq Q_2 (s,a) $ for any $s\in \cS$ and any $a\in \cA$, and define $Q_1 \leq Q_2$ similarly. Furthermore, for any policy $\pi$, recall that in \eqref{eq:operator_p} we define the operator $P^{\pi} $ by 
\# \label{eq:define_P_pi}
(P^\pi Q) (s, a) =  \EE \bigl [Q(S^\prime, A^\prime) \biggiven S^\prime \sim P(\cdot \given s,a), A^\prime \sim \pi(\cdot \given S^\prime) \bigr ]. 
\#
In addition, we define the operator $T^\pi$ by 
\$
	(T^{\pi}Q)(s, a) = r(s, a) + \gamma \cdot  (P^{\pi}Q)(s,a).
\$
Finally, we denote  $ R_{\max} / (1- \gamma)$ by $V_{\max} $.
Now we are ready to present the proof, which consists of three key steps. 
 \vskip4pt
 {\noindent \bf Step (i):}  In the first step, we establish a recursion that relates $Q^* - \tilde{Q}_{k+1}$ with $Q^* - \tilde{Q}_k$ to measure the sub-optimality of the value function $\tilde Q_k$. In the following, we first establish an upper bound for $Q^* - \tilde Q_{k+1}$ as follows.
For each $k\in \{0, \ldots, K-1\}$, by the definition of $\varrho_{k+1}$ in \eqref{eq:rhok}, we have  
	\#\label{eq:one_step1}
	  Q^* - \tilde{Q}_{k+1}  
	&=  Q^* - (Q_{k+1} - \varrho_{k+1})  
	=   Q^* - Q_{k+1} + \varrho_{k+1}  =  Q^* - T \tilde{Q}_{k} + \varrho_{k+1} \notag \\
	&   = Q^* - T^{\pi^*} \tilde{Q}_k +  ( T^{\pi^*} \tilde{Q}_k - T \tilde{Q}_{k}  ) + \varrho_{k+1} ,
	\#
	where $\pi^*$ is the greedy policy  with respect to $Q^*$.
Now we leverage the following lemma to show $ T^{\pi^*} \tilde{Q}_k  \leq T \tilde{Q}_{k}$.
	\begin{lemma}\label{lemma:aux1}
For any action-value function $Q: \cS \times \cA \to \R$ and any policy $\pi$,  it holds that 
\$
	T ^{\pi_Q} Q = T  Q \geq T^{\pi } Q.
\$
\end{lemma}
\begin{proof}
 Note that we have $\max_{a^\prime} Q(s^\prime, a^\prime) \geq Q(s^\prime, a^\prime)$ for any $s^\prime \in \cS$ and $a^\prime \in \cA$. Thus, it holds that 
\$
	(T  Q)(s,a) & = r(s,a) + \gamma \cdot \EE\bigl[\max_{a^\prime} Q(S^\prime, a^\prime) \biggiven S^\prime \sim P(\cdot \given s,a)\bigr]\\
	& \geq r(s, a) + \gamma \cdot  \EE\bigl[Q(S^\prime, A^\prime) \biggiven S^\prime \sim P(\cdot \given s,a), A^\prime \sim \pi(\cdot \given S^\prime) \bigr] = (T^\pi Q)(s,a).
\$
Recall that $\pi_Q$ is the greedy policy with respect to $Q$ such that 
\$
	\PP \bigl[A \in \argmax_a Q(s,a) \biggiven A \sim \pi_Q(\cdot \given s) \bigr]= 1,
\$
which implies 
\$
	\EE\bigl[Q(s^\prime, A^\prime)\biggiven A^\prime \sim \pi_Q(\cdot \given s^\prime) \bigr] = \max_{a^\prime} Q(s^\prime, a^\prime). 
\$
Consequently, we have 
\$
	(T ^{\pi_Q}Q) (s,a) & = r(s, a) + \gamma \cdot \EE\bigl[Q(S^\prime, A^\prime) \biggiven S^\prime \sim P(\cdot \given s,a), A^\prime \sim \pi_Q(\cdot \given S^\prime) \bigr] \\
	& = r(s, a) + \gamma \cdot \EE \bigl[\max_{a^\prime} Q(S^\prime, a^\prime) \biggiven S^\prime \sim P( \cdot \given s,a)\bigr] = (T Q)(s,a),
\$
which concludes the proof of Lemma \ref{lemma:aux1}. 
\end{proof}
By Lemma \ref{lemma:aux1}, we have  $T \tilde{Q}_{k} \geq T^{\pi^*} \tilde{Q}_k$.  
Also note that $Q^*$ is the unique fixed point of $T^{\pi^*}$. Thus, by \eqref{eq:one_step1} we have 
\# \label{eq:one_step2}
	  Q^* - \tilde{Q}_{k+1} & =   (T^{\pi^*} Q^* - T^{\pi^*} \tilde{Q}_k  ) +  ( T^{\pi^*} \tilde{Q}_k - T \tilde{Q}_{k} )+ \varrho_{k+1}  
	    \leq  ( T^{\pi^*} Q^* - T^{\pi^*} \tilde{Q}_k  )+ \varrho_{k+1} ,
\#
In the following,  we establish a lower bound for $Q^* - \tilde Q_{k+1}$ based on $\tilde Q^* - \tilde Q_{k}$.  Note that, by Lemma \ref{lemma:aux1}, we have $T^{\pi_k} \tilde Q_k = T\tilde Q_k$ and $T Q^*  \geq T^{\pi_k} Q^*$.
Similar to \eqref{eq:one_step1}, since $Q^*$ is the unique fixed point of $T$, it holds that 
\#\label{eq:one_step3}
  Q^* - \tilde{Q}_{k+1} & =   Q^* - T \tilde{Q}_{k} + \varrho_{k+1} 
	=  Q^* - T^{\pi_k} \tilde{Q}_{k} + \varrho_{k+1} 
	=   Q^*  -T^{\pi_k} Q^* + (  T^{\pi_k} Q^* - T^{\pi_k} \tilde{Q}_{k}   ) + \varrho_{k+1}  \notag \\
	&=   ( T Q^*  -T^{\pi_k} Q^* ) +  (  T^{\pi_k} Q^* - T^{\pi_k} \tilde{Q}_{k}   ) + \varrho_{k+1}  \geq  (  T^{\pi_k} Q^* - T^{\pi_k} \tilde{Q}_{k}   ) + \varrho_{k+1} .
	\#
	Thus, combining \eqref{eq:one_step2} and \eqref{eq:one_step3} we obtain that, for any $k \in \{0,\ldots, K-1\}$,
	  \#\label{eq:one_step_err}
 T^{\pi_k} Q^* - T^{\pi_k} \tilde{Q}_k + \varrho_{k+1} \leq 	Q^* - \tilde{Q}_{k+1} \leq T^{\pi^*} Q^* - T^{\pi^*} \tilde{Q}_k + \varrho_{k+1}. 
  \#
 The inequalities in  \eqref{eq:one_step_err} show that the error  $Q^* - \tilde{Q}_{k+1}$ can be sandwiched by  the  summation of a term involving $Q^* - \tilde{Q}_{k}$ and the error $\varrho_{k+1}$, which is defined in \eqref{eq:rhok} and induced by approximating the action-value function. Using $P^{\pi}$ defined in  \eqref{eq:define_P_pi}, we    can  write \eqref{eq:one_step_err}  in a more compact form, 
\# \label{eq:lemma_compact}
	\gamma  \cdot  P^{\pi^*} (Q^* - \tilde{Q}_k) + \varrho_{k+1} \geq Q^* - \tilde{Q}_{k+1} \geq \gamma \cdot   P^{\pi_k} (Q^* - \tilde{Q}_k) + \varrho_{k+1}.
\# 
Meanwhile, note that $P^{\pi}$ defined in \eqref{eq:define_P_pi} is a linear operator. In fact, $P^{\pi}$ is the Markov transition operator for the Markov chain on $\cS \times \cA$ with transition dynamics 
\$
	S_{t+1} & \sim P (\cdot \given S_t, A_t)  , \qquad 
	A_{t+1}  \sim \pi (\cdot \given S_{t+1}). 
\$
By the linearity of the operator $P^{\pi}$ and the one-step error bound in \eqref{eq:one_step_err}, we have the following characterization of the multi-step error.
\begin{lemma}
[Error Propagation]\label{lemma:multi_step_err}  
For any  $k, \ell  \in \{0, 1, \ldots, K-1 \}$ with $k < \ell$, we have
\# 
	& Q^* - \tilde{Q}_{\ell} \leq \sum_{i=k}^{\ell-1} \gamma^{\ell-1-i} \cdot (P^{\pi^*} )^{\ell-1-i} \varrho_{i+1} + \gamma^{\ell -k}\cdot  (P^{\pi^*})^{\ell-k} ( Q^* - \tilde{Q}_k) , \label{eq:multi_step_upper} \\
	& Q^* - \tilde{Q}_{\ell} \geq \sum_{i=k}^{\ell-1} \gamma^{\ell-1-i} \cdot (P^{\pi_{\ell-1}}P^{\pi_{\ell-2} }\cdots P^{\pi_{i+1}}) \varrho_{i+1} + \gamma^{\ell-k} \cdot  ( P^{\pi_{\ell-1}} P^{\pi_{\ell-2}}\cdots P^{\pi_{k}}) ( Q^* - \tilde{Q}_k). \label{eq:multi_step_lower}
\#
Here $\varrho_{i+1}$ is defined in \eqref{eq:rhok} and we use  $P^{\pi} P^{\pi'} $ and $(P^{\pi})^{k}$ to denote the composition of operators.
\end{lemma}
\begin{proof}
Note that $P^{\pi}$ is a linear operator for any policy $\pi$. We obtain \eqref{eq:multi_step_upper} and \eqref{eq:multi_step_lower} by iteratively applying the inequalities in \eqref{eq:lemma_compact}. 
\end{proof}
 
Lemma \ref{lemma:multi_step_err} gives the upper and  lower bounds for the propagation of error through multiple iterations of Algorithm \ref{algo:fit_Q}, which concludes the first step of our proof. 

\vskip4pt
{\noindent \bf Step (ii):} The results in the first step only concern the propagation of error $Q^* - \tilde{Q}_k$. In contrast, the output of Algorithm \ref{algo:fit_Q} is the greedy policy $\pi_k$ with respect to $\tilde{Q}_k$.  In the second step, our goal is to quantify the suboptimality  of $ Q^{\pi_k}$, which is the action-value function corresponding to $\pi_k$. In the following, we establish an upper bound for $Q^* -  Q^{\pi_k}$.

To begin with, we have $	Q^* \geq Q^{\pi_k}$ by the definition of $Q^*$ in  \eqref{eq:optimal_Q}. Note that 
we have $ Q^* = T^{\pi^*} Q^*$ and $
	Q^{\pi_k} = T^{\pi_k} Q^{\pi_k}.
$
Hence, it holds that 
\# \label{eq:greedy1}
	Q^* - Q^{\pi_k} &=   T^{\pi^*} Q^*  - T^{\pi_k} Q^{\pi_k} = T^{\pi^*} Q^*  +  ( - T^{\pi^*} \tilde{Q}_k + T^{\pi^*} \tilde{Q}_k   )  +  ( - T^{\pi_k} \tilde{Q}_k + T^{\pi_k} \tilde{Q}_k ) - T^{\pi_k} Q^{\pi_k} \notag \\
	&=        ( T^{\pi^*} \tilde{Q}_k - T^{\pi_k} \tilde{Q}_k ) +   ( T^{\pi^*} Q^* - T^{\pi^*} \tilde{Q}_k )+  ( T^{\pi_k} \tilde{Q}_k - T^{\pi_k} Q^{\pi_k}  ).
\#
Now we quantify the three terms on the right-hand side of \eqref{eq:greedy1} respectively. First, by Lemma \ref{lemma:aux1}, we have 
\#\label{eq:greedy2}
	T^{\pi^*} \tilde{Q}_k - T^{\pi_k} \tilde{Q}_k = T^{\pi^*} \tilde{Q}_k - T \tilde Q_k \leq 0.
\#
Meanwhile, by the definition of the operator $P^{\pi}$ in \eqref{eq:define_P_pi}, we have
\#\label{eq:greedy3}
T^{\pi^*} Q^* - T^{\pi^*} \tilde{Q}_k = \gamma \cdot P^{\pi^*}   ( Q^* - \tilde{Q}_k ), \qquad T^{\pi_k} \tilde{Q}_k - T^{\pi_k} Q^{\pi_k} = \gamma \cdot P^{\pi_k}  (\tilde Q_k - Q^{\pi_k} ).
\#
Plugging \eqref{eq:greedy2} and \eqref{eq:greedy3} into \eqref{eq:greedy1}, we obtain   
\$
	 Q^* - Q^{\pi_k} & \leq \gamma \cdot P^{\pi^*} (Q^* - \tilde{Q}_k ) + \gamma \cdot P^{\pi_k} ( \tilde{Q}_k - Q^{\pi_k} ) \\
	& = \gamma \cdot ( P^{\pi^*} - P^{\pi_k})  ( Q^* - \tilde{Q}_k )  + \gamma \cdot P^{\pi_k}  ( Q^* - Q^{\pi_k}),
\$
which further implies that   
\$
	( I - \gamma \cdot P^{\pi_k} ) ( Q^* - Q^{\pi_k}) \leq \gamma \cdot ( P^{\pi^*} - P^{\pi_k}) ( Q^* - \tilde{Q}_k).
\$
Here $I$ is the identity operator.
Since $T^{\pi}$ is a $\gamma$-contractive operator for any policy $\pi$, $I - \gamma \cdot P^{\pi}$ is invertible.  Thus, we obtain 
\#\label{eq:greedy_err}
0 \leq Q^* - Q^{\pi_k} \leq \gamma \cdot  ( I - \gamma\cdot  P^{\pi_k} )^{-1} \bigl  [ P^{\pi^*}  (Q^* - \tilde{Q}_k ) - P^{\pi_k}   (Q^* - \tilde{Q}_k )\bigr ],
\#
which relates $Q^* - Q^{\pi_k}$ with $Q^* - \tilde{Q}_k$. In the following,  we plug Lemma \ref{lemma:multi_step_err} into \eqref{eq:greedy_err} to obtain the multiple-step error bounds for $Q^{\pi_k}$. 
 First note that,  by the definition of $P^{\pi}$ in \eqref{eq:define_P_pi}, for any functions $f_1, f_2: \cS \times \cA \to \R $ satisfying $f_1 \geq f_2$,  we  have
$
P^{\pi} f_1 \geq P^{\pi} f_2.
$
Combining this inequality with the upper bound  in \eqref{eq:multi_step_upper} and  the lower bound in \eqref{eq:multi_step_lower}, we have that, for any $k < \ell$, 
\#
P^{\pi^*}(Q^* - \tilde{Q}_\ell) & \leq  \sum_{i=k}^{\ell-1} \gamma^{\ell-1-i} \cdot (P^{\pi^*} )^{\ell-i} \varrho_{i+1} + \gamma^{\ell -k}\cdot  (P^{\pi^*})^{\ell-k+1} ( Q^* - \tilde{Q}_k),  \label{eq:upper_seq_1}\\
P^{\pi_{\ell}}  (Q^* - \tilde{Q}_\ell ) & \geq \sum_{i=k}^{\ell-1} \gamma^{\ell-1-i} \cdot (P^{\pi_{\ell}}P^{\pi_{\ell-1} }\cdots P^{\pi_{i+1}}) \varrho_{i+1}\notag\\
&\qquad \qquad + \gamma^{\ell-k} \cdot  ( P^{\pi_{\ell}} P^{\pi_{\ell-1}}\cdots P^{\pi_{k}}) ( Q^* - \tilde{Q}_k). \label{eq:upper_seq_2}
\#
Then we plug \eqref{eq:upper_seq_1} and \eqref{eq:upper_seq_2} into \eqref{eq:greedy_err}  and obtain
 \#\label{eq:greedy_err_multiple}
	0 \leq Q^* - Q^{\pi_{\ell} }& \leq     ( I - \gamma\cdot  P^{\pi_{\ell} } )^{-1} \bigg \{ \sum_{i=k}^{\ell-1} \gamma^{\ell  - i} \cdot \bigl [  (P^{\pi^*} )^{\ell-i} -  (P^{\pi_{\ell}}P^{\pi_{\ell-1} }\cdots P^{\pi_{i+1}}) \bigr ] \varrho_{i+1} \notag\\
	&\qquad \qquad \qquad \qquad +  \gamma^{\ell + 1-k}\cdot \bigl [ (P^{\pi^*} )^{\ell-k+1} -  (P^{\pi_{\ell}}P^{\pi_{\ell-1} }\cdots P^{\pi_{k}}) \bigr ] ( Q^* - \tilde{Q}_k) \bigg\} 
\#
for any $ k < \ell$. To quantify the error of $Q^{\pi_K}$, we set $\ell = K$ and $k =0$ in \eqref{eq:greedy_err_multiple} to obtain 
\#\label{eq:greedy_err_K}
0 \leq Q^* - Q^{\pi_K}   & \leq ( I - \gamma P^{\pi_K} ) ^{-1} \bigg \{ \sum_{i=0}^{K -1} \gamma^{K   - i} \cdot \bigl [  (P^{\pi^*} )^{K -i} -  (P^{\pi_{K}}P^{\pi_{K-1} }\cdots P^{\pi_{i+1}}) \bigr ] \varrho_{i+1} \notag\\
	&\qquad \qquad \qquad \qquad +  \gamma^{K + 1}\cdot \bigl [ (P^{\pi^*} )^{K+1} -  (P^{\pi_{K}}P^{\pi_{K-1} }\cdots P^{\pi_{0}}) \bigr ] ( Q^* - \tilde{Q}_0) \bigg\} .
\#
For notational simplicity, we define 
\#\label{eq:define_alpha_param}
\alpha_{i} & = \frac{(1-\gamma) \gamma^{K-i-1}}{1-\gamma^{K+1}}, ~~\text{for}~~ 0 \leq i \leq  K-1, ~~\text{and}~~  
	\alpha_{K } = \frac{(1-\gamma) \gamma^{K}}{1-\gamma^{K+1}}.
\#
One can show that $\sum_{i=0}^K \alpha_ i = 1$. Meanwhile, we define $K+1$ linear operators $\{O_k \}_{ k=0}^K $ by 
\$
O_i &= {(1 - \gamma)}/{2} \cdot ( I - \gamma P^{\pi_K})^{-1}  \bigl [  (P^{\pi^*} )^{K -i} + (P^{\pi_{K}}P^{\pi_{K-1} }\cdots P^{\pi_{i+1}}) \bigr ], ~~\text{for}~~ 0 \leq i \leq K-1,\\
O_K & = {(1 - \gamma)}/{2} \cdot  ( I - \gamma P^{\pi_K})^{-1} \bigl [ (P^{\pi^*} )^{K+1} + (P^{\pi_{K}}P^{\pi_{K-1} }\cdots P^{\pi_{0}}) \bigr ].
\$
Using this notation, for any $(s, a) \in \cS \times \cA$, by \eqref{eq:greedy_err_K} we have 
\#\label{eq:absolute_value_bound}
&\bigl | Q^* (s,a) - Q^{\pi_K} (s,a) \bigr | \notag\\
&\qquad  \leq \frac{2 \gamma ( 1 - \gamma^{K+1} ) }{(1- \gamma)^2}  \cdot \biggl [ \sum_{i=0}^{K-1} \alpha_i \cdot \bigl  ( O_i | \varrho_{i+1} | \bigr ) (s,a) + \alpha_K  \cdot \bigl ( O_K | Q^* - \tilde Q_0 | \bigr ) (s,a) \biggr ],
\#
where both $O_i | \varrho_{i+1} | $ and $O_K | Q^* - \tilde Q_0 |$ are functions defined on $\cS \times \cA$.  Here \eqref{eq:absolute_value_bound} gives a uniform upper bound for $Q^* - Q^{\pi_K}$, which concludes the second step.

\vskip4pt
{\noindent \bf Step (iii):} In this step, we conclude the proof by establishing an upper bound for $\| Q^* - Q^{\pi_{K}} \|_{1, \mu}$ based on \eqref{eq:absolute_value_bound}. Here $\mu\in \cP(\cS \times \cA)$  is a fixed probability distribution. To simplify the notation, 
 for any measurable function $f: \cS \times \cA \to \R$, we denote $\mu (f) $ to be the expectation of $f$ under $\mu$, that is, 
$
	\mu(f)  = \int_{\cS \times \cA} f(s, a) \ud\mu(s,a).
$
Using this notation,  by \eqref{eq:absolute_value_bound} we  bound $\|   Q^* - Q^{\pi_{\ell}} \|_{1, \mu}$ by 
\#\label{eq:lp_norm_p}
 \| Q^* - Q^{\pi_{K}} \|_{1, \mu} & = \mu  \bigl( | Q^* -  Q^{\pi_{K} } | \bigr)  \notag \\
&  \leq  \frac{2 \gamma ( 1 - \gamma^{K+1} ) }{(1- \gamma)^2}    \cdot  \mu  \biggl [ \sum_{i=0}^{K-1} \alpha_i \cdot \bigl  ( O_i | \varrho_{i+1} | \bigr )   + \alpha_K  \cdot \bigl ( O_K | Q^* - \tilde Q_0 | \bigr )   \biggr ] .
\#
By the linearity of expectation, \eqref{eq:lp_norm_p} implies  
\#\label{eq:lp_norm1}
	\|  Q^* - Q^{\pi_K }\|_{1, \mu}   \leq   \frac{2 \gamma ( 1 - \gamma^{K+1} ) }{(1- \gamma)^2}  \cdot  \biggl [  \sum_{i=0}^{K-1}  \alpha_i \cdot  \mu \bigl ( O_{ i} |\varrho_{i+1}|   \bigr )+ \alpha_{K } \cdot \mu \bigl ( O_{K } | Q^* - \tilde Q_0 |   \bigr )\biggr ]  . 
\#
Furthermore, since both $Q^*$ and $\tilde Q_0$ are bounded by $V_{\max } = R_{\max} / ( 1-\gamma )$ in $\ell_{\infty}$-norm, we have 
\#\label{eq:last_term}
  \mu \bigl ( O_{K } | Q^* - \tilde Q_0 |   \bigr ) \leq    2 \cdot R_{\max} / ( 1- \gamma).
  \#
  Moreover, for any $i \in \{0, \ldots, K-1\}$, by expanding $( 1-\gamma P^{\pi_K})^{-1}$ into a infinite series, 
 we have
\#\label{eq:other_terms}
	\mu\bigl ( O_{ i} |\varrho_{i+1}|  \bigr)  & = \mu \biggl \{ \frac{1-\gamma}{2}\cdot ( 1-\gamma P^{\pi_K})^{-1} \bigl [ ( P^{\pi^*}) ^{K-i} +  (P^{\pi_{K}}P^{\pi_{K-1}}   \cdots P^{\pi_{i+1}})\bigr ] |\varrho_{i+1}| \biggr \} \notag \\ 
	& = \frac{1-\gamma}{2}\cdot  \mu \biggl \{ \sum_{j=0}^\infty \gamma^j \cdot \bigl [ (P^{\pi_K})^j(P^{\pi^*})^{K-i} + (P^{\pi_K})^{j+1} (P^{\pi_{K-1}}   \cdots P^{\pi_{i+1}})\bigr ] |\varrho_{i+1}|   \biggr \}.
\#
To upper bound the right-hand side of \eqref{eq:other_terms}, 
  we consider the following quantity 
\#\label{eq:one_term_mupi}
\mu \bigl[ (P^{\pi_K})^j(P^{\tau_m} P^{\tau_{m-1}}   \cdots P^{\tau_{1}} )  f \bigr ]= \int_{\cS \times \cA} \bigl[ (P^{\pi_K})^j(P^{\tau_m} P^{\tau_{m-1}}   \cdots P^{\tau_{1}} ) f\bigr ] ( s, a)  \ud\mu(s, a).
\#
Here $\tau_1, \ldots, \tau_m$ are   $m$ policies.
Recall that  $P^{\pi}$ is the transition operator of a Markov process defined on $\cS \times \cA$ for any policy $\pi$. Then the integral on the right-hand side of \eqref{eq:one_term_mupi} corresponds to the expectation of  the function $f(X_t) $, where $\{ X_t \}_{t\geq 0}$   is a Markov process defined on $\cS \times \cA$. Such a Markov process has initial distribution $X_0 \sim \mu$. The first $m$ transition operators are $\{ P^{\tau_j} \} _{j\in [m]}$, followed by $j$ identical transition operators $P^{\pi_K}$. Hence, $   (P^{\pi_K})^j(P^{\tau_m} P^{\tau_{m-1}}   \cdots P^{\tau_{1}} )\mu   $ is the marginal distribution of $X_{j+m}$, which we denote by $\tilde \mu_j$ for notational simplicity. 
Hence,  \eqref{eq:one_term_mupi} takes the form 
\#\label{eq:one_term_mupi2}
\mu \bigl[ (P^{\pi_K})^j(P^{\tau_m} P^{\tau_{m-1}}   \cdots P^{\tau_{1}} )  f \bigr ] =  \EE \bigl[ f(X_{j+m} ) \bigr]= \tilde \mu_j (f)  = \int_{\cS \times \cA} f(s,a) \ud  \tilde \mu_j (s, a)
\#
for any measurable function $f$ on $\cS \times \cA$.
By Cauchy-Schwarz  inequality, we have 
\#\label{eq:apply_radon}
\tilde \mu_j (f)    \leq  \biggl [ \int _{\cS \times \cA} \Bigl | \frac{\ud \tilde \mu_j} {\ud\sigma}(s,a)  \Bigr|^2 \ud\sigma(s,a)\biggr ]  ^{1/2 } \biggl [ \int_{\cS \times \cA} | f(s,a)|  ^{2 } \ud\sigma(s,a)\biggr ]^{1/2} ,	
\#
in which $ \ud \tilde \mu_j / \ud \sigma\colon \cS \times \cA \rightarrow \RR$ is the Radon-Nikodym derivative. Recall that the $(m+j)$-th order concentration coefficient $\kappa(m+j; \mu, \sigma)$ is defined in \eqref{eq:concentration_coef}.
Combining \eqref{eq:one_term_mupi2} and \eqref{eq:apply_radon}, we obtain  
\$
\tilde \mu_j (f)   \leq \kappa  (m+j; \mu, \sigma) \cdot \| f \| _{\sigma}.
\$
Thus, by \eqref{eq:other_terms} we have 
\#\label{eq:other_terms2}
\mu\bigl ( O_{ i} |\varrho_{i+1}| \bigr) & =   \frac{1-\gamma}{2}\cdot   \sum_{j=0}^\infty \gamma^j \cdot \Bigl \{ \mu \bigl[ (P^{\pi_K})^j(P^{\pi^*})^{K-i}|\varrho_{i+1}|  \bigr ] + \mu \bigl [  (P^{\pi_K})^{j+1} (P^{\pi_{K-1}}   \cdots P^{\pi_{i+1}})|\varrho_{i+1}| \bigr ]  \Bigr \}  \notag \\
&\leq (1-\gamma )  \cdot \sum_{j=0}^\infty \gamma^j  \cdot  \kappa ( K - i + j; {\mu, \sigma}) \cdot \| \varrho_{i+1}\|_{\sigma}.
 \#
 Now we combine \eqref{eq:lp_norm1}, \eqref{eq:last_term}, and \eqref{eq:other_terms2} to obtain
 \$
& \|  Q^* - Q^{\pi_K }\|_{1, \mu}   \leq \frac{2 \gamma ( 1 - \gamma^{K+1} ) }{(1- \gamma)^2}   \cdot  \biggl [  \sum_{i=0}^{K-1}  \alpha_i \cdot  \mu \bigl ( O_{ i} |\varrho_{i+1}|   \bigr )+ \alpha_{K } \cdot \mu \bigl ( O_{K } | Q^* - \tilde Q_0 |   \bigr )\biggr ]   \\
 &\qquad \leq \frac{2 \gamma ( 1 - \gamma^{K+1} ) }{(1- \gamma) }   \cdot  \biggl [  \sum_{i=0}^{K-1} \sum_{j=0}^{\infty}\alpha_i \cdot  \gamma^j \cdot \kappa( K- i + j; \mu, \sigma) \cdot \| \varrho_{i+1} \|_{\sigma} \biggr ] + \frac{4 \gamma ( 1 - \gamma^{K+1} ) }{(1- \gamma) ^3 } \cdot \alpha_K \cdot  R_{\max}. \notag
 \$
Recall that in Theorem \ref{thm:err_prop} and \eqref{eq:rhok} we define $\varepsilon_{\max} = \max_{i \in [K] } \| \varrho_{i} \|_{\sigma}$. We have that $\|  Q^* - Q^{\pi_K }\|_{1, \mu} $ is further upper bounded by 
\#\label{eq:p_norm_final2}
& \|  Q^* - Q^{\pi_K }\|_{1, \mu}\\
&\qquad  \leq  \frac{2 \gamma ( 1 - \gamma^{K+1} ) }{(1- \gamma) }   \cdot  \biggl [  \sum_{i=0}^{K-1} \sum_{j=0}^{\infty}\alpha_i \cdot  \gamma^j \cdot \kappa( K- i + j; \mu, \sigma) \biggr ] \cdot \varepsilon_{\max} + \frac{4 \gamma ( 1 - \gamma^{K+1} ) }{(1- \gamma) ^3 }   \cdot \alpha_K \cdot  R_{\max} \notag \\
&\qquad  =  \frac{2 \gamma ( 1 - \gamma^{K+1} ) }{(1- \gamma) }   \cdot  \biggl [  \sum_{i=0}^{K-1} \sum_{j=0}^{\infty} \frac{( 1 - \gamma) \gamma ^{K-i - 1} }{ 1 - \gamma ^{K+1} } \cdot \gamma^j \cdot \kappa( K- i + j; \mu, \sigma)   \biggr ] \cdot \varepsilon_{\max} + \frac{4   \gamma^{K+1}   }{(1- \gamma) ^2 }  \cdot R_{\max},\notag
\#
where the last equality follows from the definition of $\{ \alpha_i\}_{0\leq i\leq K}$  in \eqref{eq:define_alpha_param}. We simplify the summation on the right-hand side of \eqref{eq:p_norm_final2} and use Assumption \ref{assume:concentrability} to obtain
\#\label{eq:simplify_computation}
&\sum_{i=0}^{K-1} \sum_{j=0}^{\infty} \frac{( 1 - \gamma) \gamma ^{K-i - 1} }{ 1 - \gamma ^{K+1} } \cdot \gamma^j \cdot \kappa( K- i + j; \mu, \sigma) \notag \\
&\qquad    = \frac{1 - \gamma}{1 - \gamma ^{K+1}}  \sum_{j=0}^{\infty} \sum_{i=0}^{K-1}   \gamma^{K- i +j -1} \cdot  \kappa( K - i + j ; \mu, \sigma)  \notag \\
& \qquad    \leq  \frac{1 - \gamma}{1 - \gamma ^{K+1}}  \sum _{m = 0}^{\infty} \gamma ^{m - 1} \cdot m \cdot \kappa( m ; \mu, \sigma)  \leq       \frac{\phi_{\mu, \sigma}}{(1 - \gamma ^{K+1}) (1 -\gamma) }  ,
\#
where the last inequality follows from \eqref{eq:assume:concentrability} in Assumption \ref{assume:concentrability}. 
Finally,  combining \eqref{eq:p_norm_final2} and  \eqref{eq:simplify_computation}, we obtain 
 \$
  \|  Q^* - Q^{\pi_K }\|_{1, \mu}  \leq \frac{2\gamma\cdot \phi_{\mu, \sigma}} {(1 -\gamma)^2 }  \cdot \varepsilon_{\max} + \frac{4   \gamma^{K+1}   }{(1- \gamma) ^2 }\cdot   R_{\max},
 \$
 which concludes the third step and hence the proof of Theorem \ref{thm:err_prop}.
\end{proof}

%


\subsection{Proof of Theorem \ref{thm:each_term_error}} \label{proof:each_term_error}

\begin{proof} 
Recall that in Algorithm \ref{algo:fit_Q} we define $Y_i =   
 R_i + \gamma \cdot \max _{a\in \cA}   Q  (S_{i+1}, a)$, where $Q $ is any function in $\cF$. By definition, we have  $\EE (Y_i \given S_i = s, A_i=a) = (T   Q )(s,a) $ for any $(s, a) \in \cS \times \cA$. Thus, $T  Q $ can be viewed as the underlying truth of the regression problem defined in \eqref{eq:define_fit_Q}, where the  covariates and responses are $\{ (S_i, A_i)\}_{i \in [n]}$ and $\{Y_i \}_{i\in [n]}$, respectively.   Moreover,  note that $TQ$ is not necessarily in  function class $\cF$. We denote by  $ Q ^*$ the best approximation of $T   Q $ in $\cF$, which is the solution to  
\#\label{eq:population}
\minimize _{f \in \cF} \| f - T  Q  \|_{\sigma} ^2 = \EE \Bigl \{ \bigl[  f(S_i , A_i) -   Q  (S_i, A_i)  \bigr]^2 \Bigr \}. 
\#
 For notational simplicity, in the sequel we denote $ (S_i, A_i)$  by $X_i$ for all $i \in [n]$. 
For any $f \in \cF$, we define $\| f \|_{n} ^2= 1/n\cdot \sum_{i=1}^n [ f  (X_i)]^2$.  
 Since both $\hat Q$ and $TQ$ are bounded by $V_{\max}  = R_{\max} / ( 1 - \gamma)$, 
 we only need to consider the case where $\log  N_{\delta} \leq n$. Here $N_{\delta}$ is the cardinality of $\cN( \delta, \cF, \| \cdot \|_{\infty})$.
 Moreover, let $ f_{1}, \ldots, f_{N_{\delta}} $ be the centers of the minimal $\delta$-covering of $\cF$. 
Then by the   definition of $\delta$-covering, there exists $k^*\in [N_\delta]$ such that $\| \hat Q - f_{k^*} \|_{\infty} \leq \delta$. It is worth mentioning that $k^*$ is a  random variable since $\hat Q$ is obtained from data. 

In the following, we prove \eqref{eq:one_iter_bound} in two steps, which are bridged by $\EE [\| \hat Q - TQ \|_{n}^2]$.

\vskip4pt
{\noindent \bf Step (i):}
We relate  $\EE [ \| \hat Q - T Q \|_{n} ^2]  $ with its empirical counterpart $\| \hat Q - T Q \|_{n}^2$. Recall that we define $Y_i = R_i + \gamma \cdot \max_{a\in \cA} Q(S_{i+1}, a) $  for each $i \in [n]$.  
By the definition of  $\hat Q$,  for any $f \in \cF$ we have 
\#\label{eq:erm_hatQ}
\sum_{i=1}^n \bigl [ Y_i - \hat Q(X_i)  \bigr ] ^2 \leq \sum_{i=1}^n \bigl [ Y_i -  f(X_i) \bigr  ] ^2.
\#
For each $i \in [n]$, we define $\xi_i = Y_i - (TQ) (X_i) $. Then  \eqref{eq:erm_hatQ} can be written as
\#\label{eq:erm_hatQ1}
\| \hat Q - TQ \|_{n}^2 \leq \| f - TQ \|_n^2 + \frac{2}{n} \sum_{i=1}^n \xi_i \cdot \bigl [ \hat Q (X_i) - f (X_i) \bigr].
\#
Since both $f$ and $Q$ are deterministic, we have $\EE ( \| f - TQ \|_n^2 ) = \| f - TQ \|_{\sigma}^2$. Moreover, since  $\EE ( \xi_i \given X_i) = 0$ by definition, we  have $\EE [ \xi_i \cdot g(X_i) ] = 0$ for any bounded and measurable function $g  $. Thus, it holds that 
\#\label{eq:cross_term0}
\EE \biggl \{ \sum_{i=1}^n \xi_i \cdot \bigl [ \hat Q (X_i) - f (X_i) \bigr] \biggr \} = \EE \biggl \{ \sum_{i=1}^n \xi_i \cdot \bigl [ \hat Q (X_i) - (T Q) (X_i) \bigr] \biggr \} .
\#
In addition, by triangle inequality and \eqref{eq:cross_term0}, we have
\#\label{eq:cross_term1}
& \biggl | \EE \biggl \{ \sum_{i=1}^n \xi_i \cdot \bigl [ \hat Q (X_i) - (T Q)(X_i) \bigr] \biggr \}  \biggr | \notag \\
&\qquad \leq \biggl | \EE \biggl \{ \sum_{i=1}^n \xi_i \cdot \bigl [ \hat Q (X_i) - f_{k^*} (X_i) \bigr] \biggr \}  \biggr | + \biggl | \EE \biggl \{ \sum_{i=1}^n \xi_i \cdot \bigl [ f_{k^*} (X_i) - (T Q) (X_i) \bigr] \biggr \}  \biggr |,
\# 
where $f_{k^*}$ satisfies $\| f_{k^*} - \hat Q \|_{\infty} \leq \delta$. In the following, we upper bound the two terms on the right-hand side of \eqref{eq:cross_term1} respectively. 
For the first term, by applying Cauchy-Schwarz inequality twice, we have 
\#\label{eq:first_term}
& \biggl | \EE \biggl \{ \sum_{i=1}^n \xi_i \cdot \bigl [ \hat Q (X_i) - f_{k^*} (X_i) \bigr] \biggr \}  \biggr |  \leq \sqrt{n} \cdot \biggl | \EE \biggl [ \biggl ( \sum_{i=1}^n \xi_i^2 \biggr ) ^{1/2} \cdot \| \hat Q - f_{k^*}  \|_{n} \biggr ]  \biggr |  \notag \\
&\qquad \leq \sqrt{n}  \cdot \biggl [ \EE \biggl ( \sum_{i=1}^n \xi_i^2 \biggr ) \biggr ]^{1/2} \cdot  \Bigl [ \EE \bigl ( \| \hat Q - f_{k^*} \|_{n}^2   \bigr ) \Bigr ] ^{1/2}  \leq n \delta \cdot \bigl [  \EE(\xi_i^2 )\bigr ]^{1/2}  ,
\#
where we use the fact that $\{\xi_i \}_{i\in [n] } $
have the same marginal distributions and $\| \hat Q - f_{k^*} \|_{n} \leq \delta$. Since both $ Y_i $ and $TQ$ are bounded by $V_{\max}  $, $\xi_i$ is a bounded random variable by its definition. Thus, there exists a constant $C_{\xi} >0$ depending on $\xi$ such that $\EE (\xi_i^2) \leq C_{\xi}^2 \cdot V_{\max}^2.$ Then \eqref{eq:first_term} implies  
\#\label{eq:first_term_final}
\biggl | \EE \biggl \{ \sum_{i=1}^n \xi_i \cdot \bigl [ \hat Q (X_i) - f_{k^*} (X_i) \bigr] \biggr \}  \biggr | \leq C_{\xi} \cdot V_{\max}\cdot n \delta.
\#

It remains to upper bound the second term on the right-hand side of \eqref{eq:cross_term1}. We first define $N_{\delta}$ self-normalized random variables
\#\label{eq:define_Z_j}
Z_j = \frac{1}{\sqrt{n}} \sum_{i=1}^n   \xi_i \cdot \bigl [ f_j (X_i) - (TQ)(X_i) \bigr ]  \cdot    \| f_j - (TQ) \|_{n}^ {-1}    
\#
for all $j \in [N_{\delta}]$. Here recall that $\{ f_j \}_{j \in [N_{\delta} ]}$ are the centers of the minimal $\delta$-covering of $\cF$.
Then we have 
\#\label{eq:second_term1}
&\biggl | \EE \biggl \{ \sum_{i=1}^n \xi_i \cdot \bigl [ f_{k^*} (X_i) - (T Q) (X_i) \bigr] \biggr \}  \biggr | = \sqrt{n} \cdot \EE \bigl [ \| f_{k^*} - TQ \|_{n} \cdot |Z_{k^*}| \bigr] \notag \\
&\qquad \leq  \sqrt{n} \cdot \EE \Bigl \{ \bigl [ \| \hat Q - TQ \|_{n} + \| \hat Q - f_{k^*}  \|_{n}  \bigr ] \cdot |Z_{k^*} |\Bigr \}  \leq \sqrt{n} \cdot \EE  \Bigl \{  \bigl[ \| \hat Q - TQ \|_{n} + \delta \bigr] \cdot |Z_{k^*} |\Bigr \},
\#
where the first inequality follows from triangle inequality and the second inequality follows from the fact that $\| \hat Q - f_{k^*} \|_{\infty} \leq \delta$. Then applying Cauchy-Schwarz inequality to the last term on the right-hand side of \eqref{eq:second_term1}, we obtain  
\#\label{eq:second_term2}
\EE  \Bigl \{  \bigl[ \| \hat Q - TQ \|_{n} + \delta \bigr] \cdot |Z_{k^*} |\Bigr \} &  \leq \biggl (\EE \Bigl \{ \bigl[ \| \hat Q - TQ \|_{n} + \delta \bigr]^2 \Bigr \} \biggr )^{1/2} \cdot \bigl [ \EE ( Z_{k^*}^2 ) \bigr ]^{1/2} \notag \\
& \leq  \biggl ( \Bigl \{ \EE \bigl [ \| \hat Q - TQ \|_{n}^2 \bigr ] \Bigr \}^{1/2} + \delta \biggr ) \cdot  \Bigl [ \EE \bigl ( \max_{j\in [N]} Z_{j}^2 \bigr ) \Bigr ]^{1/2},
\#
where the last inequality follows from  \$\EE \bigl [ \|\hat Q - TQ \|_{n} \bigr ] \leq \Bigl \{ \EE \bigl [ \| \hat Q - TQ \|_{n}^2 \bigr ] \Bigr \}^{1/2}. \$ 
Moreover, since $\xi_i$ is centered conditioning on $\{ X_i\}_{i\in [n]} $ and is bounded by $2V_{\max}$, $\xi_i$ is a sub-Gaussian random variable. In specific, there exists an absolute constant $H_{\xi} > 0$ such that   $\| \xi_i  \|_{\psi_2} \leq H_{\xi} \cdot V_{\max}$ for each $i \in [n]$. Here the $\psi_2$-norm of a random variable $W \in \RR $ is defined as 
\$
\| W \|_{\psi_2} = \sup_{p \geq 1 } p^{-1/2} \bigl [ \EE (   |W |^p ) \bigr ]^{1/p}.
\$
By the definition of $Z_j$ in \eqref{eq:define_Z_j}, conditioning on $\{ X_i\}_{i\in [n]} $,  $\xi_i \cdot   [ f_j (X_i) - (T Q) (X_i)  ]$ is  a centered and sub-Gaussian random variable with  
\$ 
\bigl \| \xi_i \cdot  \bigl [ f_j (X_i) - (T Q) (X_i) \bigr  ] \bigr \|_{\psi_2} \leq  H_{\xi} \cdot  V_{\max} \cdot \bigl |   f_j (X_i) - (T Q) (X_i)   \bigr |.
\$
Moreover, since  $Z_j$ is a summation of independent sub-Gaussian random variables, by Lemma 5.9 of  \cite{vershynin2010introduction}, 
the $\psi_2$-norm of $Z_j$  satisfies 
\$
\| Z_j \|_{\psi_2}    \leq C \cdot H_{\xi} \cdot  V_{\max}\cdot \| f_j - TQ \|_{n}^{-1} \cdot \biggl [ \frac{1}{n}\sum_{i=1}^n  \bigl |  [ f_j (X_i) - (T Q) (X_i)  ] \bigr | ^2 \biggr ]^{1/2} \leq C \cdot     H_{\xi} \cdot V_{\max}, 
\$ 
where $C>0$ is an absolute constant. 
Furthermore, by  Lemmas  5.14 and    5.15 of  \cite{vershynin2010introduction}, 
 $Z_j^2$ is a sub-exponential random variable, and  its the moment-generating function is   bounded~by
\#\label{eq:mgf_2}
\EE \bigl [ \exp( t \cdot Z_j^2  ) \bigr ] \leq   \exp (  C\cdot  t^2 \cdot H_{\xi}^4 \cdot V_{\max}^4   )
\#
for any $t   $ satisfying $ C' \cdot |t| \cdot H_{\xi}^2 \cdot V_{\max}^2   \leq 1$, where $C$ and $C'$ are two positive absolute constants.
Moreover, by Jensen's inequality, we  bound the  moment-generating function of $  \max_{j \in [N_{\delta}]} Z_j^2 $ by  
\#\label{eq:jensen}
\EE \Bigl [ \exp \bigl ( t \cdot \max_{j \in [N_{\delta}] }Z_j^2  \bigr ) \Bigr ]  \leq \sum_{j \in [N_{\delta}]}  \EE \bigl [ \exp(t \cdot Z_j^2 ) \bigr ] .
\#
Combining \eqref{eq:mgf_2} and \eqref{eq:jensen}, we have 
\#\label{eq:conclude_maxima}
\EE \bigl ( \max_{j \in [N]} Z_j^2 \bigr ) \leq  C ^2 \cdot H_{\xi}^2 \cdot V_{\max} ^2 \cdot \log N_{\delta}   ,
\#
where $C >0$ is an absolute constant. 
Hence, plugging \eqref{eq:conclude_maxima} into \eqref{eq:second_term1} and \eqref{eq:second_term2}, we upper bound the second term of the right-hand side of \eqref{eq:cross_term0} by
\#\label{eq:second_term_final}
&\biggl | \EE \biggl \{ \sum_{i=1}^n \xi_i \cdot \bigl [ f_{k^*} (X_i) - (T Q) (X_i) \bigr] \biggr \}  \biggr | \notag\\
&\qquad \leq   \biggl ( \Bigl \{ \EE \bigl [ \| \hat Q - TQ \|_{n}^2 \bigr ] \Bigr \}^{1/2} + \delta \biggr ) \cdot C \cdot H_{\xi}\cdot V_{\max}\cdot  \sqrt{n \cdot \log N_{\delta}}.
\#
Finally, combining  \eqref{eq:erm_hatQ1}, \eqref{eq:first_term_final} and \eqref{eq:second_term_final}, we obtain the following   inequality 
\#\label{eq:final_inquality}
\EE \bigl [  \| \hat Q - TQ \|_{n}^2 \bigr ] &  \leq \inf_{f \in \cF} \EE \bigl [ \| f - TQ \|_{n}^2  \bigr ] + C_{\xi} \cdot V_{\max} \cdot  \delta \\
&\qquad + \Bigl ( \bigl \{ \EE \bigl [ \| \hat Q - (TQ) \|_{n}^2 \bigr ] \bigr \}^{1/2} + \delta \Bigr ) \cdot C \cdot H_{\xi} \cdot V_{\max} \cdot  \sqrt{  \log N_{\delta} / n} \notag \\
&\leq C \cdot V_{\max}  \sqrt{  \log N_{\delta} / n} \cdot  \Bigl \{ \EE \bigl [ \| \hat Q - (TQ) \|_{n}^2 \bigr ] \Bigr \}^{1/2} + \inf_{f \in \cF} \EE \bigl [ \| f - TQ \|_{n}^2  \bigr ] + C'  \cdot V_{\max} \delta, \notag
\#
where $C$ and $C'$ are two positive absolute constants. Here in the first inequality we take the infimum over $\cF$ because \eqref{eq:erm_hatQ} holds for any $f \in \cF$, and the second inequality holds because $\log N_{\delta} \leq n$.

Now we invoke a   simple fact to obtain the final bound for $\EE [ \| \hat Q - TQ \|_{n}^2  ] $ from \eqref{eq:final_inquality}. Let $a, b, $ and $c$ be positive numbers satisfying $a^2 \leq 2 ab + c$. For any $\epsilon \in (0,1],$ since $2ab \leq \epsilon \cdot  a^2 / (1+ \epsilon) + (1 +\epsilon) \cdot b^2 / \epsilon $,
we have 
\#\label{eq:relation_quadratic}
a^2 \leq ( 1+ \epsilon ) ^2\cdot  b^2 / \epsilon + ( 1+ \epsilon) \cdot c.
\#
Therefore, applying \eqref{eq:relation_quadratic} to \eqref{eq:final_inquality} with $a^2 =  \EE  [  \| \hat Q - TQ \|_{n}^2  ]$, $b = C \cdot V_{\max} \cdot\sqrt{  \log N_{\delta} / n}$, and $c = \inf_{f \in \cF} \EE [ \| f - TQ \|_{n}^2    ] + C' \cdot V_{\max} \cdot \delta$, 
 we obtain
  \#\label{eq:step1_final}
 \EE \bigl [  \| \hat  Q- TQ \|_{n}^2 \bigr ] \leq (1 + \epsilon) \cdot \inf_{f \in \cF}   \EE \bigl[ \| f - TQ \|_{n}^2 \bigr]  +     C \cdot V_{\max} ^2\cdot \log N_{\delta}/ (n \epsilon ) + C' \cdot V_{\max } \cdot \delta, 
 \#
 where $C$ and $C'$ are two positive absolute constants.  Now we conclude the first step.

\vskip4pt 
 {\noindent \bf Step (ii).} In this step, we relate the population risk  $\| \hat Q - TQ \|_{\sigma}^2 $ with $\EE [ \| \hat Q - TQ \|_{n}^2   ] $, which is characterized in the first step.
 To begin with, we generate $n$ i.i.d. random variables $\{ \tilde X_i = (\tilde S_i, \tilde A_i )\}_{i\in [n]}$ following $\sigma$, which are independent of $\{(S_i, A_i, R_i, S_i ')\}_{i\in[n]}$. 
 Since $\| \hat Q - f_{k^*} \|_{\infty} \leq \delta$, 
 for any $x \in \cS\times \cA$, we have 
 \#\label{eq:triangle_quad1}
 & \Bigl|   \bigl[ \hat Q(x) - (TQ)(x) \bigr]^2 - \bigl[ f_{k^*} (x) - (TQ)(x) \bigr]^2 \Bigr | \notag \\
 &\qquad = \bigl | \hat Q (x) - f_{k^*} (x) \bigr | \cdot \bigl|\hat Q(x) + f_{k^*} (x) -2 (TQ)(x) \bigr| \leq 4 V_{\max} \cdot \delta , 
 \#
 where the last inquality follows from the fact that $ \| T Q \|_{\infty} \leq V_{\max}   $ and $\| f \|_{\infty} \leq V_{\max} $  for any $f \in \cF$.  
 Then by the definition of $\| \hat Q - TQ \|_{\sigma}^2 $ and \eqref{eq:triangle_quad1}, we have 
 \#\label{eq:var_bound_first_upper}
 &\| \hat Q - TQ \|_{\sigma}^2 = \EE \biggl \{ \frac{1}{n} \sum_{i=1}^n \bigl[ \hat Q(\tilde X_i ) - (TQ) ( \tilde X_i) \bigr ] ^2 \biggr \} \notag \\
 &\qquad \leq  \EE \biggl \{ \| \hat Q - TQ \|_{n}^2 + \frac{1}{n} \sum_{i=1}^n  \bigl [ f_{k^*}(\tilde X_i  ) - (TQ) (\tilde X_i) \bigr ]^2 - \frac{1}{n}    \sum_{i=1}^n  \bigl [  f_{k^*} ( X_i  ) - (TQ) (\tilde X_i) \bigr ]^2  \biggr \} + 8 V_{\max} \cdot \delta \notag \\
 &\qquad  = \EE \bigl ( \| \hat Q - TQ \|_n^2 \bigr ) + \EE \biggl [ \frac{1}{n} \sum_{i=1}^n h_{k^*} (X_i, \tilde X_i) \biggr ] + 8 V_{\max} \cdot \delta,
 \#
 where we apply \eqref{eq:triangle_quad1} to obtain the first inequality, and in the last equality we define 
 \#\label{eq:define_hj}
 h_{j} (x,y) = \bigl[ f_{j} (y) - (TQ)(y) \bigr]^2 - \bigl[ f_{j} (x) - (TQ)(x) \bigr]^2, 
  \# 
 for any $(x, y) \in \cS \times \cA$ and any $j \in [N_\delta]$.
 Note that $h_{k^*}$ is a random function since $k^*$ is  random. 
By the definition of $h_j$ in \eqref{eq:define_hj}, we have $ | h_j (x, y)| \leq 4 V_{\max}^2 $ for any $(x, y)\in \cS \times \cA$  and $\EE[ h_j(X_i, \tilde X_i)] = 0$ for any $i \in [n]$. Moreover, the variance of $h_j(X_i, \tilde X_i)$ is upper bounded by 
\$ 
\Var\bigl[ h_j(X_i, \tilde X_i) \bigr] &= 2 \Var \Bigl \{ \bigl [ f_j (X_i) - (TQ) (X_i )\bigr ]^2 \Bigr \}\notag\\ &\leq 2 \EE \Bigl \{ \bigl [ f_j (X_i) - (TQ) (X_i )\bigr ]^4 \Bigr \} \leq 8 \Upsilon  ^2 \cdot V_{\max}^2,
\$
where we define $\Upsilon$ by letting 
\#\label{eq:define_ups}
\Upsilon^2  = \max \Bigl ( 4  V_{\max}^2  \cdot   \log N_{\delta} / n ,   \max _{j \in [N_{\delta}]} \EE \Bigl \{ \bigl [ f_j (X_i) - (TQ) (X_i )\bigr ]^2 \Bigr \}\Bigr ).
\# 
Furthermore, we define 
\#\label{eq:define_T}
T = \sup_{j\in [N_{\delta}] }     \biggl |  \sum_{i=1}^n h_j(X_i, \tilde X_i ) / \Upsilon  \biggr |  .
\#
Combining \eqref{eq:var_bound_first_upper} and \eqref{eq:define_T}, we obtain  \#\label{eq:upper_risk_T}
\| \hat Q - T Q \|_{\sigma}^2 \leq \EE \bigl [ \| \hat Q - TQ \|_n^2 \bigr ] +   \Upsilon / n   \cdot \EE (T)   + 8 V_{\max} \cdot \delta.
\#
In the sequel, we utilize Bernstein's inequality to establish an upper bound for 
$\EE (T)$, which is stated as follows for completeness.
\begin{lemma}[Bernstein's Inequality] \label{lemma:bernstein}
	Let $U_1, \ldots U_n$ be $n$ independent random variables satisfying $\EE (U_i) = 0$ and $|U_i| \leq M$ for all $i \in [n]$. 
	Then  for any $t >  0$, we have 
	\$
	\PP\biggl( \biggl | \sum_{i=1}^n U_i \biggr | \geq t  \biggr ) \leq 2 \exp \biggl ( \frac{-t^2 }{ 2M\cdot t /3 + 2\sigma^2} \biggr ),
	\$
	where $\sigma^2 = \sum_{i=1}^n \Var(U_i)$ is the variance of $\sum_{i=1}^n U_i$.
\end{lemma}

We first apply Bernstein's inequality by setting $U_i = h_j ( X_i, \tilde X_i) / \Upsilon $ for each $i \in [n]$.  Then we take a union bound for all $j \in [N_{\delta}]$ to obtain  
\#\label{eq:apply_bernstein}
 \PP(T  \geq t) = \PP \biggl [ \sup_{j \in [N_{\delta}] } \frac{1}{n} \biggl| \sum_{i=1}^n h_j(X_i, \tilde X_i) / \Upsilon  \biggr | \geq t \bigg] \leq 2 N_{\delta} \cdot   \exp \biggl \{ \frac{-t^2 }{ 8 V_{\max}^2 \cdot [  t / (3\Upsilon ) + n ] } \biggr \}.
\#
Since $T$ is nonnegative,  we have $ \EE (T) = \int_{0}^{\infty} \PP( T \geq t) \ud t$. Thus, for any   $u \in( 0, 3\Upsilon  \cdot n) $, by \eqref{eq:apply_bernstein} it holds that 
\#\label{eq:expected_T}
\EE (T) &\leq u + \int_u^{\infty } \PP( T \geq t) \ud t \leq u + 2 N_{\delta}  \int_{u }^{3 \Upsilon \cdot n} \exp\biggl( \frac{- t^2}{   16 V_{\max}^2\cdot n } \biggr ) ~\ud t + 2 N _{\delta } \int_{3 \Upsilon  \cdot n }^{\infty} \exp\biggl (  \frac{ -3 \Upsilon \cdot t  }{16 V_{\max}^2 } \biggr ) ~ \ud t \notag \\
& \leq u +   32 N_{\delta}  \cdot  V_{\max}^2 \cdot n /u    \cdot \exp\biggl (  \frac{-u^2}{ 
 16 V_{\max}^2 \cdot n} \biggr ) + 32 N_{\delta}  \cdot V_{\max}^2 / (3 \Upsilon ) \cdot  \exp \biggl( \frac{- 9 \Upsilon^2 \cdot n} { 16 V_{\max}^2 } \biggr ) ,
\#
where in the second inequality we use the fact that $ \int_{s}^{\infty} \exp( -t^2/ 2)\ud t \leq 1/s \cdot \exp(-s^2/2)$ for all $s > 0$. Now we set $u = 4 V_{\max} \sqrt{ n \cdot \log N_{\delta } }$ in \eqref{eq:expected_T} and plug in the definition of $\Upsilon$ in \eqref{eq:define_ups} to obtain  
\#\label{eq:expected_T_final}
\EE(T) \leq 4 V_{\max} \sqrt{ n \cdot \log  N_{\delta }} + 8 V_{\max} \sqrt{n / \log N_{\delta} }  + 6  V_{\max} \sqrt{n / \log  N_{\delta } } \leq 8 V_{\max} \sqrt{ n \cdot \log  N_{\delta }},
\#
where the last inequality holds when $\log N_{\delta } \geq 4.$ 
Moreover,  the definition of $\Upsilon$ in \eqref{eq:define_ups} implies that 
$
\Upsilon \leq \max   [ 2   V_{\max}   \sqrt{ \log N_{\delta } / n}, \| \hat Q - TQ \|_{\sigma}  + \delta  ].
$
In the following,  we only need to consider the case where $\Upsilon \leq \| \hat Q - TQ \|_{\sigma}  + \delta$, since we already have \eqref{eq:one_iter_bound} if $ \| \hat Q - TQ \|_{\sigma}  + \delta  \leq 2   V_{\max}   \sqrt{ \log N_{\delta } / n}$, which concludes the proof.

Then, when $\Upsilon \leq \| \hat Q - TQ \|_{\sigma}  + \delta$ holds, combining \eqref{eq:upper_risk_T} and \eqref{eq:expected_T_final} we obtain 
 \#\label{eq:combine_together_risk}
 \| \hat Q - T Q \|_{\sigma}^2 & \leq  \EE \bigl [ \| \hat Q - TQ \|_n^2 \bigr ] +  8  V_{\max} \sqrt{    \log (N)/ n} \cdot  \| \hat Q - T Q \|_{\sigma} +  8  V_{\max} \sqrt{    \log  N_{\delta }/ n} \cdot \delta        + 8 V_{\max} \cdot \delta \notag \\
 & \leq \EE \bigl [ \| \hat Q - TQ \|_n^2 \bigr ] + 8  V_{\max} \sqrt{    \log  N_{\delta } / n} \cdot  \| \hat Q - T Q \|_{\sigma} + 16 V_{\max} \cdot \delta.
 \#
 We apply the inequality in \eqref{eq:relation_quadratic} to \eqref{eq:combine_together_risk} with $a = \| \hat Q - T Q \|_{\sigma}$, $b = 8 V_{\max} \sqrt{    \log   N_{\delta } / n} $, and $c = \EE [ \| \hat Q - TQ \|_n^2 ] + 16 V_{\max} \cdot \delta$. Hence we finally obtain that 
 \#\label{eq:step2_final}
 \| \hat Q - T Q \|_{\sigma}^2 &\leq (1 + \epsilon) \cdot \EE \bigl [ \| \hat Q - TQ \|_n^2 \bigr ] \notag\\
 &\qquad+ ( 1+ \epsilon) ^2 \cdot 64 V_{\max} \cdot \log (N) / (n \cdot \epsilon)  +  ( 1+ \epsilon ) \cdot 18 V_{\max} \cdot \delta,
 \#
which concludes the second step of the proof.

Finally, combining these two steps together, namely, \eqref{eq:step1_final} and \eqref{eq:step2_final}, we conclude that 
\$
\| \hat Q - T Q \|_{\sigma}^2 \leq  ( 1 + \epsilon)^2   \cdot \inf_{f \in \cF}   \EE \bigl[ \| f - TQ \|_{n}^2 \bigr]  + C_1 \cdot  V_{\max }^2 \cdot \log    N_{\delta}  / ( n \cdot \epsilon)  + C_2 \cdot V_{\max} \cdot \delta,
\$
where $C_1$ and $C_2$ are two absolute constants. Moreover, since  $Q \in \cF$, we have 
\$
  \inf_{f \in \cF}   \EE \bigl[ \| f - TQ \|_{n}^2 \bigr]  \leq \sup _{Q \in \cF}  \Bigl \{ \inf_{f \in \cF}   \EE \bigl[ \| f - TQ \|_{n}^2 \bigr]  \Bigr \}, 
\$
which concludes the proof of Theorem \ref{thm:each_term_error}.
\end{proof}


\section{Proof of Theorem \ref{thm:game}}\label{proof:thm:game}

In this section, we present the proof of Theorem \ref{thm:game}. The proof is   
similar to that of 
 Theorem \ref{thm:main}, which is presented in \S\ref{proof:thm:main} in details.  
 In the following, we follow the proof in  \S\ref{proof:thm:main}  and only 
 highlight the differences for brevity.

 \begin{proof}
 	The proof requires two key ingredients, namely the error propagation and the statistical error incurred by a    single step of Minimax-FQI.  We note that \cite{perolat2015approximate} establish error propagation for the   state-value functions in the approximate modified policy iteration algorithm, which is more general than the FQI algorithm. 
 	\begin{theorem}[Error Propagation] \label{thm:err_prop2}
 		Recall that $\{\tilde Q_k\}_{ 0\leq k \leq K} $ are the iterates of Algorithm \ref{algo:fit_Q2} and $(\pi_K, \nu_K)$ is the equilibrium policy with respect to $\tilde Q_K$. Let   $Q_K^*$ be the action-value function corresponding to $(\pi_K, \nu_{\pi_K}^*)$, where $\nu_{\pi_K}^*$ is the best-response policy of the second player against $\pi_K$. 
 	Then under  Assumption \ref{assume:concentrability2}, we have 
 		\#\label{eq:err_prop_final2}
 		\|  Q^* - Q_K^*\|_{1, \mu}  \leq  \frac{2\phi_{\mu, \rho} \cdot \gamma} {(1 -\gamma)^2 } \cdot \varepsilon_{\max} + \frac{4   \gamma^{K+1}   }{(1- \gamma) ^2 }  \cdot R_{\max},
 		\#
 		where we define the maximum one-step approximation error  $\varepsilon_{\max} = \max_{ k \in [K] } \| T \tilde Q_{k-1} - \tilde Q_{k} \|_{\sigma}$, and constant $\phi_{\mu, \nu}$  is specified in Assumption  \ref{assume:concentrability2}. 
 	\end{theorem}
 	  
 	\begin{proof}
 		We note that the proof of Theorem \ref{thm:err_prop} cannot be directly applied to prove this theorem. The main reason is that here we also need to consider the role played by the opponent, namely player two. Different from the MDP setting, here $Q_K^*$ is a fixed point of a nonlinear operator due to the fact that player two adopts the optimal policy against $\pi_K$. Thus, we need to conduct a more refined analysis.
 		See \S \ref{proof:thm:err_prop2} for a detailed proof.
 	\end{proof}
 	
 	 By this theorem, we need to derive an upper bound of $\varepsilon_{\max}$. We achieve such a goal by studying the one-step approximation error $\| T \tilde Q_{k-1} - \tilde Q_{k} \|_{\sigma}$ for each $k \in [K]$.

 	\begin{theorem}  [One-step Approximation Error] \label{thm:each_term_error2}
 		Let  $\cF\subseteq \cB( \cS \times \cA\times \cB, V_{\max} ) $ be a family of measurable functions on $\cS \times \cA\times \cB$ that are bounded by  $ V_{\max} =  R_{\max} / (1- \gamma)$. 
 			Also, let   $\{(S_i, A_i, B_i)\}_{i\in [n]}$ be $n$ i.i.d. random variables  following distribution $\sigma\in \cP( \cS \times \cA \times \cB)$. 
 	. For each $i\in [n]$,   let $R_i$ and $S_i'
 		$ be the reward obtained by the first player and the next state following  $(S_i, A_i, B_i)$.   In addition,  for any fixed $Q\in \cF$, we define the response variable as 
 		\#\label{eq:some_new_tgt}
 		Y_i = R_i +   \gamma \cdot \max_{\pi'\in\cP(\cA)}   \min_{\nu'\in\cP(\cB)}\EE_{a  \sim \pi',b  \sim \nu' }   \bigl [Q  (S_{i}' ,  a ,b ) \bigr ]. 
 		\# 
 		 Based on $\{( X_i, A_i, Y_i)\}_{i\in [n] } $, we define $\hat Q$ as   the solution to the least-squares problem \#\label{eq:regress2}
 		 \min_{f\in \cF} \frac{1}{n}  \sum_{i=1}^n \bigl[ f(S_i, A_i) - Y_i \bigr]^2. 
 		 \#
 		Then for any $\epsilon \in (0, 1]$ and any $\delta > 0$,  we have  
 		\#\label{eq:one_iter_bound2}
 		\| \hat Q - T Q \|_{\sigma}^2 \leq  ( 1 + \epsilon)^2 \cdot \sup_{g \in \cF    }   \inf_{f \in \cF } \| f - T g   \|_{\sigma}^2+ C \cdot  V_{\max }^2  / ( n \cdot \epsilon) \cdot \log    N_{\delta}  + C' \cdot V_{\max} \cdot \delta,
 		\#
 		where $C$ and $C'$ are two positive absolute constants, $T$ is the Bellman operator defined in \eqref{eq:bellman_oper22}, $N_\delta$ is the cardinality of the minimal $\delta$-covering of $\cF$ with respect to $\ell_{\infty}$-norm. 
 	\end{theorem}
 	
 	\begin{proof} 
 		By the definition of $Y_i$ in \eqref{eq:some_new_tgt}, for any $(s,a,b) \in \cS\times \cA \times \min_{\nu'\in\cP(\cB)}$, we have 
 		\$
 		& \EE (Y_i \given S_i = s, A_i=a, B_i = b) \\
 		& \qquad  =  r(s,a) + \gamma \cdot 
 		\EE_{s' \sim P(\cdot \given s,a,b)}  \Bigl\{  \max_{\pi'\in\cP(\cA)}   \min_{\nu'\in\cP(\cB)}\EE_{a ' \sim \pi',b ' \sim \nu' }   \bigl [Q  (s'  ,  a ,b ) \bigr ] \Bigr \} =  (T   Q )(s,a, b).  
 		\$
 		Thus, $TQ$ can be viewed as the ground truth
 		of the nonlinear least-squares regression problem in \eqref{eq:regress2}. Therefore, following the same proof of Theorem 
 		\ref{thm:each_term_error}, we obtain the desired result.
 	\end{proof}

 	Now we let $\cF$ be the family of ReLU Q-networks  $\cF_1$  defined in \eqref{eq:define_cF2} and set $Q = \tilde Q_{k-1} $ in  Theorem \ref{thm:each_term_error2}.  
 	In addition, setting 
 	$\epsilon = 1$ and  $\delta = 1/n$ in \eqref{eq:one_iter_bound2}, 
 we obtain
 	\#\label{eq:apply_thm_each_term2}
 	\| \tilde Q_{k+1} -T\tilde  Q_k \|_{\sigma}^2 & \leq 4\cdot  \sup_{g \in \cF_1    }   \inf_{f \in \cF_1 } \| f - T g   \|_{\sigma}^2  + C \cdot V_{\max}^2  / n \cdot \log N_1 \notag \\
 	& \leq 4\cdot  \sup_{f' \in \cG_1   }   \inf_{f \in \cF_1 }    \| f -   f'   \|_{\infty}^2 + C \cdot V_{\max}^2  / n \cdot \log N_1, 
 	\#
 	where $C$ is a positive absolute constant, $N_1$ 
 	is the $1/n$-covering number of $\cF_1$, and function class $\cG_1$ is defined as 
 	\#
 	\cG _1&= \bigl \{ f\colon \cS \times \cA \rightarrow \RR  \colon f(\cdot, a, b) \in \cG( \{ p_j, t_j , \beta_j ,  H_j \}_{j \in[q]} ) ~\text{for any}~(a,b) \in \cA\times \cB \bigr \}.\label{eq:define_cG2} 
 	\#
 	Here the second inequality follows from Assumption \ref{assume:closedness2}. 
 	
 	Thus, it remains to bound the $\ell_{\infty}$-error of approximating functions in $\cG_1$ using ReLU Q-networks in $\cF_1$ and the $1/n$-covering number of $\cF_1$. In the sequel, obtain upper bounds for these two terms. 
 	
 	By the definition of $\cG_1$ in \eqref{eq:define_cG2}, for any $f\in \cG_1$ and any $(a,b) \in \cA \times \cB$, we have $f(\cdot, a, b) \in \cG(\{(p_j, t_j, \beta_j, H_j)\}_{j \in [q]})$. Following the same construction as in \S\ref{proof:each_term_error}, 
 	we can find a function $\tilde f $ in $ \cF(L^*, \{d_j^*\}_{j=1}^{L^*+1}, s^*  ) $  such that 
 	\$
 	\|   f (\cdot , a, b) - \tilde f \|_{\infty}  \lesssim  \max_{j \in [q] } n^{-2 \beta_j^* / ( 2\beta_j^*  + t_j ) } = n^{\alpha^* -1},
 	\$
 	which implies that 
 	\#\label{eq:trash23}
 	\sup_{f' \in \cG_1   }   \inf_{f \in \cF_1 }    \| f -   f'   \|_{\infty}^2   \lesssim  n^{\alpha^* -1}. 
 	\#
 	
 	Moreover, 
for any $f \in \cF_1$ and any $(a,b) \in \cA\times \cB$, we have $f (\cdot , a, b) \in \cF(L^*, \{d_j^*\}_{j=1}^{L^*+1}, s^*  ) $.  
Let  $\cN_{\delta}$ be  the $\delta$-covering of $ \cF(L^*, \{d_j^*\}_{j=1}^{L^*+1}, s^*  )$ in the $\ell_{\infty}$-norm. Then for any $f \in \cF_1$ and any $(a,b) \in \cA\times \cB$, 
 	 there exists $g_{ab} \in \cN _{\delta} $ such that  
 	$ \| f (\cdot , a, b)  - g_{a,b}  \|_{\infty} \leq \delta$.
 	Thus, the  
  cardinality of the $\cN( \delta , \cF_1, \| \cdot \| _{\infty})  $ satisfies  
 	\#\label{eq:cardinality_bound2}
 	\bigl |\cN( \delta , \cF_1, \| \cdot \| )  \bigr | \leq | \cN_{\delta} | ^{|\cA| \cdot | \cB| }. 
 	\#
 	Combining \eqref{eq:cardinality_bound2} with Lemma \ref{lemma:cover_number} and setting $\delta = 1/n$, we obtain  that 
 	\#\label{eq:covering_final2}
  	\log  N_1&   \leq |\cA| \cdot | \cB|  \cdot \log  | \cN_{\delta} |   \leq | \cA | \cdot  | \cB|  \cdot    (s^* +1) \cdot \log \bigl [ 2 n  \cdot (L^*+1) \cdot D^2 \bigr ]   \notag \\
 	&    \leq  |\cA| \cdot  | \cB|  \cdot s^* \cdot L^* \max_{j \in [L^*]} \log (d_j^*) \lesssim    |\cA| \cdot  | \cB|  \cdot n^{\alpha^*} \cdot (\log n)^ {1 + 2\xi^*}, 
 	\#
 	where $D = \prod_{\ell=1}^{L^*+1} ( d_{\ell}^* +1 )$ and the second inequality follows from 
    \eqref{eq:dnn_hyperparam}.

 	Finally, combining \eqref{eq:err_prop_final2}, \eqref{eq:apply_thm_each_term2},  \eqref{eq:trash23}, and \eqref{eq:covering_final2}, we conclude the proof of Theorem \ref{thm:game}. 
 \end{proof}


\subsection{Proof of Theorem \ref{thm:err_prop2} }    \label{proof:thm:err_prop2}
\begin{proof}
	The proof is similar to the that of  Theorem \ref{thm:err_prop2}. Before presenting the proof, we first introduce the following notation for simplicity.
	For any $k \in \{0, \ldots, K-1\}$, we denote $T \tilde Q_{k}$ by $Q_{k+1}$ and define 
	$
	\varrho_{k} = Q_k - \tilde{Q}_k.
	$   In addition, throughout the proof, for two action-value functions $Q_1$ and $Q_2$, we write $Q_1 \leq Q_2$ if $Q_1(s,a, b) \geq Q_2 (s,a, b) $ for any $(s,a,b) \in \cS\times \cA \times \cB$, and define $Q_1 \geq Q_2$ similarly. Furthermore,   we denote by $(\pi_k, \nu_k)$  and $(\pi^*, \nu^*)$ the equilibrium policies with respect to $\tilde Q_k$ by $Q^*$, respectively. 
Besides, in addition to the Bellman operators $T^{\pi, \nu}$ and $T$ defined in \eqref{eq:bellman_oper21} and \eqref{eq:bellman_oper22}, for any policy $\pi$ of the first  player, we define 
	\#\label{eq:Bellman3}
	T^\pi Q(s,a,b) = r(s,a,b) + \gamma   \cdot  \EE_{s' \sim P(\cdot \given s,a,b)} \Bigl\{ \min_{\nu'\in\cP(\cB)}\EE_{a' \sim \pi ,b' \sim \nu' }\bigl [Q(s', a',b') \bigr ]\Bigl\}, 
	\#
 corresponds to the case where the first player follows policy $\pi$ and player 2 adopts the best policy  in response to $\pi$.  By this definition, it holds that $Q^* = T^{\pi^*} Q^*$.
 Unlike the MDP setting, here $T^{\pi}$ is a nonlinear operator due to the minimization in \eqref{eq:Bellman3}. 
 Furthermore, for any fixed action-value function $Q$, we define the best-response policy against $\pi$ with respect to $Q$, denote by $\nu(\pi, Q)   $, as 
 \#\label{eq:best_response}
 \nu(\pi, Q)   (\cdot \given s) =    \argmin_{\nu ' \in \cP(\cB) }\EE_{a  \sim \pi ,b \sim \nu' }\bigl [Q(s , a,b) \bigr ] .
 \#
 Using this notation, we can write \eqref{eq:Bellman3} equivalently as 
 \$
 T^\pi Q(s,a,b) = r(s,a,b) + \gamma  \cdot \bigl ( P^{\pi , \nu(\pi, Q)} \bigr  ) (s,a,b).
 \$
 Notice that $ P^{\pi , \nu(\pi, Q)}$ is a linear operator and that $\nu_Q = \nu( \pi_Q, Q)$ by definition.

 Now we are ready to present the proof, which can be decomposed into 
	 three key steps. 
	
		\vskip4pt
	{\noindent \bf Step (i):}  In the first step, we establish   recursive upper and lower bounds for   $\{ Q^* - \tilde{Q}_{k} \}_{ 0 \leq k \leq K}$.    For each  $k\in \{0, \ldots, K-1\}$,  similar to  the decomposition in \eqref{eq:one_step1}, we have 
	\#\label{eq:one_step21}
	Q^* - \tilde{Q}_{k+1}  
   = Q^* - T^{\pi^*} \tilde{Q}_k +  ( T^{\pi^*} \tilde{Q}_k - T \tilde{Q}_{k}  ) + \varrho_{k+1} ,
	\#
	where $\pi^*$ is part of the equilibrium policy  with respect to $Q^*$ and   $T^{\pi^*}$ is defined in  \eqref{eq:Bellman3}.
	
	Similar to Lemma \ref{lemma:aux1}, we  utilize the following lemma to show $ T^{\pi^*} \tilde{Q}_k  \geq T \tilde{Q}_{k}$.

	\begin{lemma}\label{lemma:aux2}
		For any action-value function $Q: \cS \times \cA \times \cB \to \R$, let $(\pi_Q, \nu_Q)$ be the equilibrium policy with respect to $Q$. Then for and any policy $\pi $ of the first player,  it holds that 
		\$
		T ^{\pi_Q} Q = T  Q \geq T^{\pi } Q.
		\$
		Furthermore, for any policy $\pi \colon \cS  \rightarrow \cP(\cA) $ of   player one  and any action-value function $Q$, we~have 
		\#\label{eq:second_greedy}
		T^{\pi , \nu(\pi, Q) } Q  =  T^{\pi} Q  \leq T^{\pi, \nu} Q 
		\#
		for any  policy $\nu \colon   \cS \rightarrow \cP(\cB)$, where $\nu(\pi, Q)$ is the best-response policy defined in \eqref{eq:best_response}.
	\end{lemma}
	\begin{proof}
	Note that for any  $s' \in \cS$, by the definition of equilibrium policy, 
	we have 
	$$
	\max _{\pi'\in \cP(\cA )}   \min_{\nu'\in\cP(\cB)}\EE_{a' \sim \pi',b' \sim \nu' }\bigl [Q(s', a',b') \bigr ] = 	   \min_{\nu'\in\cP(\cB)}\EE_{a' \sim \pi _Q ,b' \sim \nu' }\bigl [Q(s', a',b') \bigr ] . 
	$$
		Thus, for any state-action tuple $(s,a, b)$, taking conditional expectations of $s$  with respect to  $   P(\cdot  \given s, a, b)$ on both ends of this equation, we have 
		\$
		&	( T ^{\pi_Q} Q) (s,a,b)   =  r(s,a,b) +  \gamma \cdot  \EE_{s' \sim P(\cdot \given s,a,b)} \Bigl\{ \min_{\nu'\in\cP(\cB)}\EE_{a' \sim \pi_Q  ,b' \sim \nu' }\bigl [Q(s', a',b') \bigr ]\Bigl\}  \\
			&\qquad  =  r(s,a,b) +  \gamma \cdot  \EE_{s' \sim P(\cdot \given s,a,b)} \Bigl\{  	\max _{\pi'\in \cP(\cA )} \min_{\nu'\in\cP(\cB)}\EE_{a' \sim \pi'   ,b' \sim \nu' }\bigl [Q(s', a',b') \bigr ]\Bigl\} = (TQ)(s,a,b),
		\$
		which proves $T^{\pi_Q} Q = TQ$.
	Moreover, for any policy $\pi$ of the first player, it holds that 
\$ 
 \max _{\pi'\in \cP(\cA )}   \min_{\nu'\in\cP(\cB)}\EE_{a' \sim \pi',b' \sim \nu' }\bigl [Q(s', a',b') \bigr ] \geq \min_{\nu'\in\cP(\cB)}\EE_{a' \sim \pi ,b' \sim \nu' }\bigl [Q(s', a',b') \bigr ] .
\$
Taking    expectations with respect to $s' \sim P(\cdot  \given s, a, b)$ 
 on both ends, 
 	we establish $ TQ \geq T^{\pi} Q$.
 	
 	It remains to show the second part of Lemma \ref{lemma:aux2}. 
 	By  the definition of $\nu(\pi, Q)$, we have 
 	$$
    \EE_{a' \sim \pi,b' \sim \nu(\pi, Q) }\bigl [Q(s', a',b') \bigr ] = 	   \min_{\nu'\in\cP(\cB)}\EE_{a' \sim \pi ,b' \sim \nu' }\bigl [Q(s', a',b') \bigr ] ,
 	$$
 	which, combined with the definition of $T^\pi$ in \eqref{eq:Bellman3},  implies that $	T^{\pi , \nu(\pi, Q) } Q  =  T^{\pi} Q $.
  Finally, for any policy $\nu$ of player two, we have 
  $$
 \min_{\nu'\in\cP(\cB)}\EE_{a' \sim \pi ,b' \sim \nu' }\bigl [Q(s', a',b') \bigr ] \geq \EE_{a' \sim \pi ,b' \sim \nu }\bigl [Q(s', a',b') \bigr ] ,
  $$
  which yields $T^{\pi} Q  \leq T^{\pi, \nu} Q$. Thus, we conclude the proof of this lemma. 
	\end{proof}

Hereafter, for notational simplicity,  for each $k $, let $(\pi_k, \nu_k)$ be the equilibrium joint policy with respect to $\tilde Q_k$, and  we denote $\nu( \pi^*, \tilde Q_k)$  and $\nu( \pi_k, Q^*)$ by $\tilde \nu_k$ and $\bar \nu_k$, respectively.
Applying 
  Lemma \ref{lemma:aux2} to \eqref{eq:one_step21} and utilizing the fact that  $Q^* = T^{\pi^*} Q^*$, 
  we have 
  	\# \label{eq:one_step22}
  Q^* - \tilde{Q}_{k+1} &  
  \leq  (  Q^* - T^{\pi^*} \tilde{Q}_k  )+ \varrho_{k+1} =  ( T^{\pi^*} Q^* - T^{\pi^*} \tilde{Q}_k  )+ \varrho_{k+1}  \notag \\
  & \leq \bigl  ( T^{\pi^*, \tilde \nu_k } Q^* - T^{\pi^*, \tilde \nu_k }  \tilde{Q}_k \bigr  ) +  \varrho_{k+1}  = \gamma \cdot P^{\pi^*, \tilde \nu_k } ( Q^* - \tilde Q_k) + \varrho_{k+1} ,
  \#
  where the last inequality follows from \eqref{eq:second_greedy}.
   Furthermore,   
   for a lower bound of $Q^* - \tilde Q_{k+1}$, similar to \eqref{eq:one_step3},  we have 
	\#\label{eq:one_step32}
	Q^* - \tilde{Q}_{k+1}  
	 & =   ( T Q^*  -T^{\pi_k} Q^* ) +  (  T^{\pi_k} Q^* - T^{\pi_k} \tilde{Q}_{k}   ) + \varrho_{k+1} \notag \\
	 & \geq  (  T^{\pi_k} Q^* - T^{\pi_k} \tilde{Q}_{k}   ) + \varrho_{k+1} \geq \gamma  \cdot P^{\pi_k, \bar \nu_k} ( Q^* - \tilde Q_k ) + \varrho _{k+1},
	\#
	where the both  inequalities  follow  from Lemma \ref{lemma:aux2}. 
	Thus, combining \eqref{eq:one_step22} and \eqref{eq:one_step32} we have 
	\# \label{eq:lemma_compact2}
	\gamma \cdot P^{\pi^*, \tilde \nu_k } (Q^* - \tilde{Q}_k) + \varrho_{k+1} \geq Q^* - \tilde{Q}_{k+1} \geq\gamma  \cdot P^{\pi_k, \bar \nu_k} (Q^* - \tilde{Q}_k) + \varrho_{k+1}.
	\# 
	for any $k \in \{0,\ldots, K-1\}$. Similar to the proof of Lemma  \ref{lemma:multi_step_err} , by applying recursion to \eqref{eq:lemma_compact2}, we obtain the following 
   upper and  lower bounds for the error  propagation of    Algorithm \ref{algo:fit_Q2}.

	 \begin{lemma}
	 	[Error Propagation]\label{lemma:multi_step_err2}  
	 	For any  $k, \ell  \in \{0, 1, \ldots, K-1 \}$ with $k < \ell$, we have
	 	\# 
	  Q^* - \tilde{Q}_{\ell}	&  \leq  \sum_{i = k}^{\ell-1} \gamma^{\ell-1-j} \cdot \bigl  (P^{\pi^*, \tilde \nu_{\ell-1}}  
	 	P^{\pi^*, \tilde \nu_{\ell-2}}\cdots P^{\pi^*, \tilde \nu_{i +1}} \bigr )   \varrho _{i +1}  \notag \\
	 	&\qquad\qquad\qquad + \gamma^{\ell-k} \cdot   (P^{\pi^*, \tilde \nu_{\ell-1}}  
	 	P^{\pi^*, \tilde \nu_{\ell-2}}\cdots P^{\pi^*, \tilde \nu_{k }} \bigr )   ( Q^* - \tilde{Q} _k),  \label{eq:multi_step_upper2} \\
	 Q^* - \tilde{Q}_{\ell} 	&  \geq \sum_{i=k}^{\ell-1} \gamma^{\ell-1-i} \cdot \bigl (P^{\pi_{\ell-1}, \bar \nu_{\ell-1}}P^{\pi_{\ell-2}, \bar \nu_{\ell-2} }\cdots P^{\pi_{i+1}, \bar \nu_{i+1}} \bigr ) \varrho_{i+1} \notag \\
	 & \qquad \qquad \qquad + \gamma^{\ell-k} \cdot\bigl (P^{\pi_{\ell-1}, \bar \nu_{\ell-1}}P^{\pi_{\ell-2}, \bar \nu_{\ell-2} }\cdots P^{\pi_{i+1}, \bar \nu_{k}} \bigr )   ( Q^* - \tilde{Q}_k). \label{eq:multi_step_lower2}
	 	\#
	 \end{lemma}

	 \begin{proof}
	 	  The desired results follows  from applying the inequalities in \eqref{eq:lemma_compact2} multiple times and the linearity  of the operator $P^{\pi,\nu}$ for any joint policy $(\pi, \nu)$.
	 \end{proof}
	 
	 The above lemma establishes recursive upper and lower bounds for the error terms  $\{ Q^*- \tilde Q_k \}_{ 0 \leq k \leq K-1}$, which completes the first step of the proof.

	\vskip4pt
	{\noindent \bf Step (ii):} 
	In the second step, we characterize the suboptimality of the equilibrium policies constructed by Algorithm \ref{algo:fit_Q2}. Specifically, for each $\pi_k$, we denote by $Q_k^*$ the action-value function obtained when agent one follows $\pi_k$ while agent two adopt the best-response policy against $\pi_k$. In other words, $Q^*_k$ is the fixed point of Bellman operator $T^{\pi_k}$ defined in \eqref{eq:Bellman3}. In the following, we obtain an upper bound of $Q^* - Q_k^*$, which establishes the  a notion of suboptimality of policy $(\pi_k, \nu_k)$  from the perspective of the first player.  
	
	To begin with, for any $k$, we first decompose $Q^* - Q_k^*$ by 
	\#
	\label{eq:decom222}
	 Q^* - Q_k^* = \bigl ( T ^{\pi^* } Q^* -  T ^{\pi^* } \tilde Q_k \bigr )  +  \bigl (  T ^{\pi^* } \tilde Q_k -  T ^{\pi_k } \tilde Q_k  \bigr )  + \bigl ( T ^{\pi_k } \tilde Q_k - T^{\pi_k} Q_k^*   \bigr )  .
	\#
	Since $\pi_k$ is the  equilibrium policy with respect to $\tilde Q_k$, 
	by Lemma \ref{lemma:aux2}, we have 
	$
	T ^{\pi^* } \tilde Q_k \leq  T ^{\pi_k } \tilde Q_k. 
	$
	Recall that $(\pi^*, \nu^*)$ is the joint equilibrium policy with respect to $Q^*$. 
The second argument of Lemma \ref{lemma:aux2} implies that 
	\#\label{eq:upper222}
	T ^{\pi^* } Q^* \leq T^{\pi^*, \tilde \nu_k} Q^* , \qquad T ^{\pi_k } \tilde Q_k \leq T ^{\pi_k, \hat \nu_k } \tilde Q_k, 
	\#
	where $\tilde \nu_k = \nu( \pi^*, \tilde Q_k)  $ and we define $\hat \nu_k = \nu(\pi_k, Q_k^*)$. Thus, combining \eqref{eq:decom222} and \eqref{eq:upper222} yields that 
	\#\label{eq:upper223}
	0 \leq Q^* - Q_k^*  &  \leq \gamma \cdot P^{\pi^*, \tilde \nu_k} \bigl ( Q^* - \tilde Q_k \bigr  ) + \gamma \cdot P^{\pi_k, \hat \nu_k} \bigl  ( \tilde Q_k - Q_k^* \bigr ) \notag \\
	& =  \gamma \cdot  \bigl ( P^{\pi^*, \tilde \nu_k} - P^{\pi_k, \hat \nu_k} \bigr )  \bigl ( Q^* - \tilde Q_k \bigr  ) +\gamma \cdot P^{\pi_k, \hat \nu_k} \bigl  (  Q^*  - Q_k^* \bigr )    .
	\#
	Furthermore, since $I - \gamma \cdot P^{\pi_k, \hat \nu_k} $ is invertible, by \eqref{eq:upper223} we have 
	\#\label{eq:upper224}
	0 \leq Q^* - Q_k^* \leq \gamma \cdot \bigl ( I - \gamma \cdot P^{\pi_k, \hat \nu_k}  \bigr )^{-1} \cdot  \bigl [ P^{\pi^*, \tilde \nu_k} \bigl ( Q^* - \tilde Q_k \bigr  )  - P^{\pi_k, \hat \nu_k} \bigl ( Q^* - \tilde Q_k \bigr  ) \bigr ].
		\#
Now we  apply Lemma \ref{lemma:multi_step_err2} to  the right-hand side of \eqref{eq:upper224}. Then for any $k \leq  \ell$,   we have 
\#
P^{\pi^*, \tilde \nu_\ell}  \bigl ( Q^* - \tilde{Q}_{\ell} \bigr ) 	&  \leq  \sum_{i = k}^{\ell-1} \gamma^{\ell-1-j} \cdot \bigl  (P^{\pi^*, \tilde \nu_{\ell}}  
P^{\pi^*, \tilde \nu_{\ell-1}}\cdots P^{\pi^*, \tilde \nu_{i +1}} \bigr )   \varrho _{i +1}  \notag \\
&\qquad\qquad\qquad + \gamma^{\ell-k} \cdot  \bigl  (P^{\pi^*, \tilde \nu_{\ell}}  
P^{\pi^*, \tilde \nu_{\ell-1}}\cdots P^{\pi^*, \tilde \nu_{k }} \bigr )   ( Q^* - \tilde{Q} _k),   \label{eq:upper225} \\
P^{\pi_\ell , \hat \nu_\ell}  \bigl ( Q^* - \tilde{Q}_{\ell} 	\bigr)&  \geq \sum_{i=k}^{\ell-1} \gamma^{\ell-1-i} \cdot \bigl (P^{\pi_{\ell }, \bar \nu_{\ell}}P^{\pi_{\ell}, \bar \nu_{\ell-2} }\cdots P^{\pi_{i+1}, \bar \nu_{i+1}} \bigr ) \varrho_{i+1} \notag \\
& \qquad \qquad \qquad + \gamma^{\ell-k} \cdot\bigl (P^{\pi_{\ell}, \bar \nu_{\ell}}P^{\pi_{\ell-1}, \bar \nu_{\ell-1} }\cdots P^{\pi_{k}, \bar \nu_{k}} \bigr )   ( Q^* - \tilde{Q}_k). \label{eq:upper226}
\#
Thus, setting $\ell = K$ and $k=0$ in \eqref{eq:upper225} and \eqref{eq:upper226}, we have
\#\label{eq:upper227}
& Q^* - Q_K^*   \leq \bigl ( I - \gamma \cdot P^{\pi_K, \hat \nu_K}  \bigr )^{-1} \cdot  \\
&  \qquad \qquad\qquad \qquad\biggl \{ 
\sum_{i=0}^{K-1} \gamma^{K-1} \cdot  \Bigl [ \bigl (P^{\pi^*, \tilde \nu_{K}}  
P^{\pi^*, \tilde \nu_{K-1}}\cdots P^{\pi^*, \tilde \nu_{i +1} } \bigr ) - \bigl ( P^{\pi_{K }, \bar \nu_{K}}P^{\pi_{K-1}, \bar \nu_{K-1} }\cdots P^{\pi_{i+1}, \bar \nu_{i+1}}
 \bigr ) \Bigr ]\varrho _{i+1} \notag \\
	& \qquad \qquad \qquad + \gamma^{K+1} \cdot \Bigl [   \bigl( P^{\pi^*, \tilde \nu_{K}}  
	P^{\pi^*, \tilde \nu_{K-1}}\cdots P^{\pi^*, \tilde \nu_{0 }} \bigr )  - \bigl( P^{\pi_{K}, \bar \nu_{K} }P^{\pi_{K-1}, \bar \nu_{K-1} }\cdots P^{\pi_{0}, \bar \nu_{0}} 
			\bigr)  \Bigr ] ( Q ^* - \tilde Q_0)  \biggr \}. \notag
			\#
	 To simplify the notation,  we define $\{ \alpha_{i} \}_{i = 0}^K$   as in \eqref{eq:define_alpha_param}. Note that we have $\sum_{i=0}^K \alpha_i = 1$ by definition. Moreover, we define $K+1$ 
linear operators $\{O_k \}_{ k=0}^K $ as follows. For any $i \leq K-1$, let 
	\$
	O_i &= \frac{1 - \gamma}{2} \cdot ( I - \gamma \cdot P^{\pi_K, \hat \nu_K}  \bigr )^{-1}    \Bigl [ \bigl (P^{\pi^*, \tilde \nu_{K}}  
	P^{\pi^*, \tilde \nu_{K-1}}\cdots P^{\pi^*, \tilde \nu_{i +1} } \bigr ) - \bigl ( P^{\pi_{K }, \bar \nu_{K}}P^{\pi_{K-1}, \bar \nu_{K-1} }\cdots P^{\pi_{i+1}, \bar \nu_{i+1}}
	\bigr ) \Bigr ] . 
	\$
	Moreover,  we define $O_K$ by
	 \$
	O_K & = \frac{1 - \gamma}{2} \cdot ( I - \gamma \cdot P^{\pi_K, \hat \nu_K}  \bigr )^{-1}  \Bigl [   \bigl( P^{\pi^*, \tilde \nu_{K}}  
	P^{\pi^*, \tilde \nu_{K-1}}\cdots P^{\pi^*, \tilde \nu_{0 }} \bigr )  - \bigl( P^{\pi_{K}, \bar \nu_{K} }P^{\pi_{K-1}, \bar \nu_{K-1} }\cdots P^{\pi_{0}, \bar \nu_{0}} 
	\bigr)  \Bigr ].
	\$
Therefore, taking 
	  absolute values on both sides of  \eqref{eq:upper227}, we obtain that   
	\#\label{eq:upper228}
&	\bigl | Q^* (s,a,b) - Q_{K}^*  (s,a,b) \bigr | \notag \\
	& \qquad 
	\leq   \frac{2 \gamma ( 1 - \gamma^{K+1} ) }{(1- \gamma)^2}  \cdot    	 
	 \biggl [ \sum_{i=0}^{K-1} \alpha_i \cdot \bigl  ( O_i | \varrho_{i+1} | \bigr ) (s,a, b) + \alpha_K  \cdot \bigl ( O_K | Q^* - \tilde Q_0 | \bigr ) (s,a,b) \biggr ], 
	\# 
	for any $(s, a, b) \in \cS \times \cA\times \cB$, which concludes  
  the second step of the proof.

	\vskip4pt
	{\noindent \bf Step (iii):} We note that \eqref{eq:upper228} is  nearly the same as  \eqref{eq:absolute_value_bound} for the MDP setting. Thus, in the last step, we follow the same proof strategy as in    {\bf{Step (iii)}}  in  \S\ref{proof:thm:err_prop}. For notational simplicity, for any function $f \colon \cS\times \cA \times \cB \rightarrow \RR$ and any probability distribution  $\mu \in \cP( \cS\times \cA \times \cB  )$, we denote the expectation of $f$ under $\mu$ by $\mu(f)$. By taking expectation with respect to $\mu$ in \eqref{eq:upper228}, we have 
	\#\label{eq:upper229}
	\|  Q^* - Q_{K}^*  \|_{1, \mu}   \leq   \frac{2 \gamma ( 1 - \gamma^{K+1} ) }{(1- \gamma)^2}  \cdot  \biggl [  \sum_{i=0}^{K-1}  \alpha_i \cdot  \mu \bigl ( O_{ i} |\varrho_{i+1}|   \bigr )+ \alpha_{K } \cdot \mu \bigl ( O_{K } | Q^* - \tilde Q_0 |   \bigr )\biggr ]  . 
	\#
	By the definition of $O_i$, we can write $\mu ( O_{ i} |\varrho_{i+1}|     )$ as 
	\#\label{eq:bound_oi2}
	\mu\bigl ( O_{ i} |\varrho_{i+1}|  \bigr)  &   = \frac{1-\gamma}{2}\cdot  \mu \biggl \{  \sum_{j=0}^\infty \gamma^j \cdot \Bigl [ \bigl (P^{\pi_K, \hat \nu_K} \bigr )^j  
	\bigl (P^{\pi^*, \tilde \nu_{K}}  
	P^{\pi^*, \tilde \nu_{K-1}}\cdots P^{\pi^*, \tilde \nu_{i +1} } \bigr ) \\
	& \qquad \qquad \qquad\qquad\qquad + \bigl (P^{\pi_K, \hat \nu_K} \bigr )^j    \bigl ( P^{\pi_{K }, \bar \nu_{K}}P^{\pi_{K-1}, \bar \nu_{K-1} }\cdots P^{\pi_{i+1}, \bar \nu_{i+1}}
	\bigr ) 
\Bigr ] |\varrho_{i+1}|   \biggr \}. \notag 
\#
	To upper bound the right-hand side of \eqref{eq:bound_oi2}, 
	we consider the following quantity 
	\$
	\mu \bigl[   P^{\tau_{m}}   \cdots P^{\tau_{1}} )  f \bigr ]= \int_{\cS \times \cA\times \cB} \bigl[(P^{\tau_m} P^{\tau_{m-1}}   \cdots P^{\tau_{1}} ) f\bigr ] ( s, a, b)  \ud\mu(s, a, b), 
	\$
	where $\{ \tau_t \colon \cS \rightarrow \cP(\cA \times \cB) \}_{t\in [m]}$ are $m$ joint policies of the two-players. 
	By Cauchy-Schwarz inequality, it holds that 
	\$
		\mu \bigl[   P^{\tau_{m}}   \cdots P^{\tau_{1}} )  f \bigr ] &\leq \Biggl [  \int _{\cS \times \cA \times \cB } \biggl | \frac{    \ud (  P^{\tau_m} P^{\tau_{m-1}}  \cdots P^{\tau_1} \mu )    } { \ud \sigma} (s,a,b) \Biggr | ^2   \ud\sigma(s,a,b)  \biggr ] ^{1/2} \notag \\
		& \qquad \cdot  \biggl [ \int_{\cS \times \cA\times \cB } | f(s,a, b)|  ^{2 } \ud\sigma(s,a,b)\biggr ]^{1/2}  \leq \kappa  (m ; \mu, \sigma) \cdot \| f \| _{\sigma},
	\$ 
	where  $ \kappa  (m ; \mu, \sigma) $ is the $m$-th  concentration parameter defined in \eqref{eq:concentration2}. 
	Thus, by \eqref{eq:bound_oi2}we have 
	\#\label{eq:other_terms22}
	\mu\bigl ( O_{ i} |\varrho_{i+1}| \bigr)  \leq (1-\gamma )  \cdot \sum_{j=0}^\infty \gamma^j  \cdot  \kappa ( K - i + j; {\mu, \nu}) \cdot \| \varrho_{i+1}\|_{\sigma}.
	\#
	Besides, 
	since both $Q^*$ and $\tilde Q_0$ are bounded by $  R_{\max} / ( 1-\gamma )$ in $\ell_{\infty}$-norm, we have 
	\#\label{eq:last_term2}
	\mu \bigl ( O_{K } | Q^* - \tilde Q_0 |   \bigr ) \leq    2 \cdot R_{\max} / ( 1- \gamma).
	\#
	Finally, combining 
  \eqref{eq:upper229}, \eqref{eq:other_terms22}, and \eqref{eq:last_term2}, we obtain that 
	\$
& 	\|  Q^* - Q^{\pi_K }\|_{1, \mu}  \\  
  &  \qquad  \leq \frac{2 \gamma ( 1 - \gamma^{K+1} ) }{(1- \gamma) }   \cdot  \biggl [  \sum_{i=0}^{K-1} \sum_{j=0}^{\infty}\alpha_i \cdot  \gamma^j \cdot \kappa( K- i + j; \mu, \nu) \cdot \| \varrho_{i+1} \|_{\sigma} \biggr ] \ + \frac{4 \gamma ( 1 - \gamma^{K+1} ) }{(1- \gamma) ^3 } \cdot \alpha_K \cdot  R_{\max}  \notag\\
   & \qquad  \leq   \frac{2 \gamma ( 1 - \gamma^{K+1} ) }{(1- \gamma) }   \cdot  \biggl [  \sum_{i=0}^{K-1} \sum_{j=0}^{\infty} \frac{( 1 - \gamma) \gamma ^{K-i - 1} }{ 1 - \gamma ^{K+1} } \cdot \gamma^j \cdot \kappa( K- i + j; \mu, \nu)   \biggr ] \cdot \varepsilon_{\max} + \frac{4   \gamma^{K+1}   }{(1- \gamma) ^2 }  \cdot R_{\max},\notag
	\$
	where  the last inequality follows from the fact that 
	$\varepsilon_{\max} = \max_{i \in [K] } \| \varrho_{i} \|_{\sigma}$.  Note that in  \eqref{eq:simplify_computation} we show that it holds under Assumption 
	\ref{assume:concentrability2} that
	 \$ 
	 \sum_{i=0}^{K-1} \sum_{j=0}^{\infty} \frac{( 1 - \gamma) \gamma ^{K-i - 1} }{ 1 - \gamma ^{K+1} } \cdot \gamma^j \cdot \kappa( K- i + j; \mu, \nu)      
	    \leq       \frac{\phi_{\mu, \nu}}{(1 - \gamma ^{K+1}) (1 -\gamma) }  .
	\$ Hence, we  obtain \eqref{eq:err_prop_final2} and thus
conclude the   proof of Theorem \ref{thm:err_prop2}.
\end{proof}

\newpage
\bibliographystyle{ims}
\bibliography{rl_ref}
\end{document}